\PassOptionsToPackage{table}{xcolor}

\documentclass{article}

\usepackage{xcolor}

\usepackage[preprint]{preprint_2027}

\usepackage{comment}
\usepackage{xspace}
\usepackage{float}
\usepackage{amsthm}
\usepackage{amsmath}
\usepackage{amssymb}
\usepackage{tikz}
\usetikzlibrary{arrows.meta,positioning,fit,calc}
\usepackage[ruled,vlined]{algorithm2e}
\usepackage{subcaption}
\usepackage{booktabs}
\usepackage{tabularx}
\usepackage{multirow}
\usepackage{makecell}
\usepackage{graphicx}
\usepackage{listings}

\usepackage{soul} 

\definecolor{oursblue}{RGB}{224,238,255}
\newtheorem{proposition}{Proposition}
\newtheorem{lemma}{Lemma}
\newcommand{\method}{FOCUS\xspace}

\newcommand{\cellstd}[2]{\shortstack[c]{$#1$\\[-2pt]{\tiny$\pm\,#2$}}}

\definecolor{codebg}{RGB}{245,245,245}
\definecolor{framegray}{RGB}{180,180,180}
\definecolor{simplegray}{gray}{0.93}

\title{\method{}: From Privileged States to RGB-D\\ with Controlled Modality Switching and Representation Alignment}

\author{
  Filip Grigorov\\
  University of Toronto \\
  Canada\\
  \texttt{filip.grigorov@mail.utoronto.ca} \\
  \And
  Kourosh Darvish \\
  University of Toronto \\
  Canada \\
  \texttt{kourosh.darvish@utoronto.ca} \\
  \AND
  Nandita Vijaykumar \\
  University of Toronto \\
  Canada \\
  \texttt{nandita.vijaykumar@utoronto.ca} \\
}

\begin{document}

\hypersetup{
  pdftitle={FOCUS: From Privileged States to RGB-D with Controlled Modality Switching and Representation Alignment},
  pdfauthor={Filip Grigorov, Kourosh Darvish, Nandita Vijaykumar},
  pdfsubject={},
  pdfkeywords={Vision-Based Reinforcement Learning, Learning Representations for Robotic Perception and Control, Robotic Manipulation}
}

\maketitle

\begin{abstract}
Vision-based reinforcement learning for robotic manipulation is
sample-inefficient because RGB-D observations are high-dimensional and noisy.
Privileged state information available in simulation can accelerate training,
but its absence at test time creates a train--test modality gap. We propose
FOCUS, a single-stage PPO framework that trains the critic on privileged state
while automatically regulating whether the actor collects rollouts from RGB-D
or privileged-state latents. Regulation is driven by the KL divergence between
the action distributions induced by the two modalities, while representation
alignment encourages consistent action selection across them. Together, these
mechanisms limit RGB-D rollouts when the actor's action distributions from
RGB-D and privileged-state latents disagree. As they align, RGB-D exposure
increases, shifting on-policy training toward the RGB-D inputs used at test
time. Across five manipulation tasks, FOCUS raises average test success from
0.71 to 0.93 relative to the strongest RGB-D-at-test baseline on each task.
When accounting for each method's complete training pipeline,
budget-normalized training-success AUC increases from 0.47 to 0.65. On
Pick-and-Place, test success rises from 0.47 to 0.86, while AUC increases from
0.12 to 0.61, a 5.0$\times$ improvement in learning efficiency over the fixed
interaction budget.
\end{abstract}


\section{Introduction}
\label{sec:introduction}

Reinforcement learning enables robots to learn closed-loop manipulation policies through trial-and-error interaction and optimize them directly for task success. However, reinforcement learning for robotics remains limited by poor sample efficiency, sensitivity to reward design, and the cost and safety risks of real-world trial-and-error data collection~\citep{DulacArnold2021Challenges}. Simulation-based RL mitigates some of these issues through resettable environments, parallel rollouts, controlled task variation, and access to training-only privileged information.

Learning manipulation policies directly from RGB-D observations is difficult because visual inputs are high-dimensional and noisy \citep{yarats2020improvingsampleefficiencymodelfree,yarats2021masteringvisualcontinuouscontrol,mnih2013playingatarideepreinforcement}. In some settings, visual observations may also provide an incomplete view of task-relevant state \citep{levine2015endtoendtrainingdeepvisuomotor,yarats2020improvingsampleefficiencymodelfree}. In simulation, however, training can access compact privileged state information, such as object and gripper poses, contact signals, and task-goal variables, that helps determine which action should be taken~\citep{cai2025provablepartiallyobservablereinforcement}. This creates an opportunity to use privileged state during RL training to improve sample efficiency while still producing an RGB-D policy that does not require privileged state at test time.

The difficulty is addressing the modality gap between privileged-state training and RGB-D-based action selection at test time. While privileged states can accelerate policy learning, reliance on them during training can produce representations and action decisions that are poorly matched to the RGB-D observations available at test time. Conversely, training only from RGB-D avoids this mismatch but can be sample-inefficient. The problem we address is therefore how to adapt the actor from privileged-state rollouts toward RGB-D control during RL training while avoiding premature RGB-D exposure.

Prior work improves sample efficiency with privileged information during training in several ways. Asymmetric actor–critic methods give privileged state to the critic while training the actor from visual inputs \citep{pinto2017asymmetricactorcriticimagebased}. This can improve value estimation and the resulting policy updates, but the actor still generates rollouts from high-dimensional visual observations throughout training. More broadly, asymmetric learning methods use full-state information during training to improve policies that operate from partial observations \citep{warrington2021robustasymmetriclearningpomdps}. However, these methods do not explicitly regulate visual rollout exposure according to whether privileged-state and visual inputs induce compatible action distributions. Teacher–student and distillation methods transfer knowledge from a trained teacher to a student model or policy \citep{hinton2015distillingknowledgeneuralnetwork,rusu2016policydistillation,parisotto2016actormimicdeepmultitasktransfer}. In robotics, privileged simulator state is often used to train a teacher policy or world model that is then transferred to a vision-based student \citep{yamada2023twistteacherstudentworldmodel,lum2024dextrahgpixelstoactiondexterousarmhand,singh2025dextrahrgbvisuomotorpoliciesgrasp,Srinivasan2025ACGD}. Such pipelines typically separate privileged teacher training from visual student transfer. In contrast, \method{} performs this transition online within a single PPO run by using cross-modal policy disagreement to regulate whether rollouts are generated from privileged-state or RGB-D latents. Other privileged-learning methods use training-only information to guide visual attention, provide auxiliary learning signals, or regularize reliance on privileged inputs~\citep{salter2021attentionprivilegedreinforcementlearning,hu2024privilegedsensingscaffoldsreinforcement,kamienny2020privilegedinformationdropoutreinforcement}. Separately, visual RL methods improve learning from pixels through representation learning and data augmentation~\citep{srinivas2020curlcontrastiveunsupervisedrepresentations,laskin2020reinforcementlearningaugmenteddata,yarats2021masteringvisualcontinuouscontrol}. These directions are complementary, but they do not use cross-modal action agreement to regulate which modality generates on-policy rollouts. We ask whether privileged information can be used online for this purpose.

In this work, our goal is to improve the sample efficiency and training reliability of RGB-D policies in simulated PPO-based manipulation without requiring RGB-D-based rollout collection from the start or separating privileged training and RGB-D transfer into distinct stages. We propose \textbf{\method}, a single-stage RL framework applied to PPO \citep{schulman2017proximalpolicyoptimizationalgorithms} for privileged-to-vision transfer in robotic manipulation. \method is based on two complementary ideas. First, we use cross-modal representation alignment to reduce mismatch between training-time and test-time inputs. Second, we introduce an online KL-guided switching mechanism that regulates when the actor uses RGB-D rather than privileged-state latents. 
An upper-quantile KL budget reduces RGB-D exposure when rare
high-disagreement states remain. This progressively shifts on-policy
training toward the RGB-D inputs available at test time while retaining
privileged-state value estimation.

\noindent\textbf{Contributions.}

(1) We formulate \method{} as a single-stage PPO-based method that makes privileged-state and RGB-D actor representations compatible while adaptively regulating whether on-policy rollouts are generated from privileged-state or RGB-D latents.

(2) We introduce KL-guided modality switching, which uses disagreement between the action distributions produced by the same policy from privileged-state- and RGB-D-derived actor inputs to regulate RGB-D-based rollout collection. An upper-quantile KL budget reduces the probability of RGB-D rollouts when rare high-disagreement states remain.

(3) Across five IsaacLab manipulation tasks, \method{} raises average test success
from 0.71 to 0.93 relative to the strongest RGB-D-at-test baseline on each task,
while end-to-end budget-normalized training-success AUC increases from 0.47 to
0.65. On Pick-and-Place, test success increases from 0.47 to 0.86 and AUC from
0.12 to 0.61 (5.0$\times$). Replacing adaptive KL-controlled exposure with a
fixed ramp reduces Pick-and-Place AUC to 0.36, while the lightweight
switch-plus-policy-KL variant (\method{}-simple) matches or exceeds \method{} on
two of three discriminative tasks for each metric, suggesting simplified future
variants.

\section{Related Work}
\label{sec:related_work}
We review prior work related to privileged-to-vision transfer, focusing on teacher–student distillation, asymmetric actor–critic methods, and privileged-information-guided learning.

\paragraph{Teacher--Student and Distillation.}
Knowledge distillation transfers behavior or representations from a teacher to a student
\citep{hinton2015distillingknowledgeneuralnetwork,
rusu2016policydistillation,
parisotto2016actormimicdeepmultitasktransfer}.
In robotic learning with privileged information, a policy or world model is often first trained using privileged state, and its knowledge is then transferred to an RGB- or depth-based student
\citep{yamada2023twistteacherstudentworldmodel,
lum2024dextrahgpixelstoactiondexterousarmhand,
singh2025dextrahrgbvisuomotorpoliciesgrasp,
Srinivasan2025ACGD}.
Such pipelines typically separate privileged training from visual transfer. Related asymmetric imitation methods use a privileged expert to guide a policy that operates from limited observations
\citep{warrington2021robustasymmetriclearningpomdps}. Recent methods reduce this separation through joint teacher-student training or adaptive data sharing
\citep{wu2025learnteachsampleefficientprivileged,
wang2024cts,
messikommer2025studentinformed},
but still organize learning around distinct teacher and student policies and do not directly use their current behavioral agreement to determine which observation pathway collects on-policy experience. \method{} instead uses privileged-state and RGB-D paths within a single actor and regulates their rollout use according to policy disagreement, limiting premature RGB-D rollout collection while allowing the RGB-D path to learn directly from task rewards. By reducing low-quality RGB-D rollouts early in training, this design aims to improve sample efficiency while preserving or improving RGB-D test-time success.

\paragraph{Asymmetric Actor--Critic and Privileged Critics.}
Asymmetric actor--critic methods provide privileged information to the critic while training the actor from observations available at test time
\citep{pinto2017asymmetricactorcriticimagebased,
yue2022aaccasymmetricactorcriticcontextual,
lambrechts2025theoreticaljustificationasymmetricactorcritic}.
The privileged critic can improve sample efficiency by providing more informative value estimates without making the actor depend on privileged information at test time.
Related dual-critic methods combine privileged-state and observation-based critics to improve learning under partial observability
\citep{li2024dual}.
However, the actor must still learn perception and control jointly from RGB-D from the start of training, so an immature visual pathway can generate low-quality on-policy experience and slow learning.
\method{} retains a privileged-state critic, but additionally provides privileged-state and RGB-D actor paths and regulates which path generates rollouts according to their policy disagreement. This aims to reduce inefficient early RGB-D experience while improving RGB-D-only success at test time.

\paragraph{Privileged-Information-Guided Learning and Visual RL.}
Privileged information has been used during training to guide visual attention, strengthen auxiliary learning components, or regularize reliance on privileged inputs
\citep{salter2021attentionprivilegedreinforcementlearning,
hu2024privilegedsensingscaffoldsreinforcement,
kamienny2020privilegedinformationdropoutreinforcement}.
These methods can improve sample efficiency and final visual-policy performance, but they do not explicitly align privileged-state and RGB-D actor paths at the level of action selection or use their agreement to control rollout generation.

Separately, visual RL methods improve learning from pixels through reconstruction, latent dynamics models, contrastive representation learning, and data augmentation
\citep{yarats2020improvingsampleefficiencymodelfree,
lee2020stochasticlatentactorcriticdeep,
srinivas2020curlcontrastiveunsupervisedrepresentations,
laskin2020reinforcementlearningaugmenteddata,
yarats2021masteringvisualcontinuouscontrol}.
These methods directly optimize the test-time visual policy and can improve both learning speed and final success, but still rely on RGB-D-generated experience from the start of training.
\method{} combines privileged-state and RGB-D actor paths, using action-level alignment and policy disagreement to limit inefficient early RGB-D rollouts while continuing to optimize RGB-D-only test-time performance.

\paragraph{Fusion and Modality Routing.}
Gated fusion, Transformer-based attention, multimodal transformers, FiLM conditioning, and representation switching provide useful mechanisms for combining or selecting representations
\citep{arevalo2017gatedmultimodalunitsinformation,
vaswani2017attentionneed,
tsai2019multimodaltransformerunalignedmultimodal,
perez2017filmvisualreasoninggeneral,
vanseijen2008switching}.
However, they do not address replacing training-only privileged state with RGB-D-based action selection at test time.

\section{Methodology}
\label{sec:method}

\subsection{Problem setup and method overview}
At environment step \(t\), the simulator provides
\(\mathbf{x}_t=[\mathbf{s}_t^{\mathrm{priv}},\mathbf{s}_t^{\mathrm{prop}},o_t,\mathbf{a}_{t-1}]\),
where \(\mathbf{s}_t^{\mathrm{priv}}\) is privileged simulator state, \(\mathbf{s}_t^{\mathrm{prop}}\) is proprioceptive state, \(o_t\) is an RGB-D observation, and \(\mathbf{a}_{t-1}\) is the previous action. At test time, privileged state is unavailable, so the policy acts from
\(\mathbf{y}_t=[\mathbf{s}_t^{\mathrm{prop}},o_t,\mathbf{a}_{t-1}]\).

\method uses a state encoder $E_s$, an RGB-D encoder $E_o$, an actor
$\pi_\theta$, and a critic $V_\phi$. The encoders produce
\begin{equation}
\label{eq:encoders}
\mathbf{z}_{t,o}=E_o(o_t),\qquad
\mathbf{z}_{t,s}=E_s(\mathbf{s}_t^{\mathrm{priv}}).
\end{equation}

A shared fusion module combines the modality latent with proprioception
$\mathbf{s}_t^{\mathrm{prop}}$ and the previous action $\mathbf{a}_{t-1}$
before it is passed to the actor or critic. We absorb this module into the
actor and critic notation below for compactness.

As shown in Figure~\ref{fig:focus-ideas}, the KL-guided switch
selects either the privileged-state latent $\mathbf{z}_{t,s}$ or the
RGB-D latent $\mathbf{z}_{t,o}$ for the current actor rollout.
We denote the selected latent by $\mathbf{z}_t^{\mathrm{selected}}$.
During training, the actor therefore generates each rollout according to
\begin{equation}
\pi_\theta\!\left(
\cdot\mid
\mathbf{z}_t^{\mathrm{selected}},
\mathbf{s}_t^{\mathrm{prop}},
\mathbf{a}_{t-1}
\right),
\end{equation}
while the critic remains anchored to privileged state,
\begin{equation}
V_\phi\!\left(
\mathbf{z}_{t,s},
\mathbf{s}_t^{\mathrm{prop}},
\mathbf{a}_{t-1}
\right).
\end{equation}
This retains privileged information for value estimation while regulating when the RGB-D pathway generates on-policy experience.

\method aligns privileged-state $\mathbf{z}_{t,s}$ and RGB-D $\mathbf{z}_{t,o}$
representations during training so that their fused actor inputs induce compatible
action distributions, enabling RGB-D-based action selection at test time. It
combines KL-guided RGB-D exposure, representation alignment that encourages
matched privileged-state and RGB-D actor inputs to induce similar actions, and
an additional control-agreement objective that regularizes agreement in policy
outputs and value estimates.

\begin{figure}[t]
\centering
\begin{subfigure}[t]{1.1\linewidth}
    \hspace*{-0.05\linewidth}
    \centering
    \includegraphics[width=\linewidth]{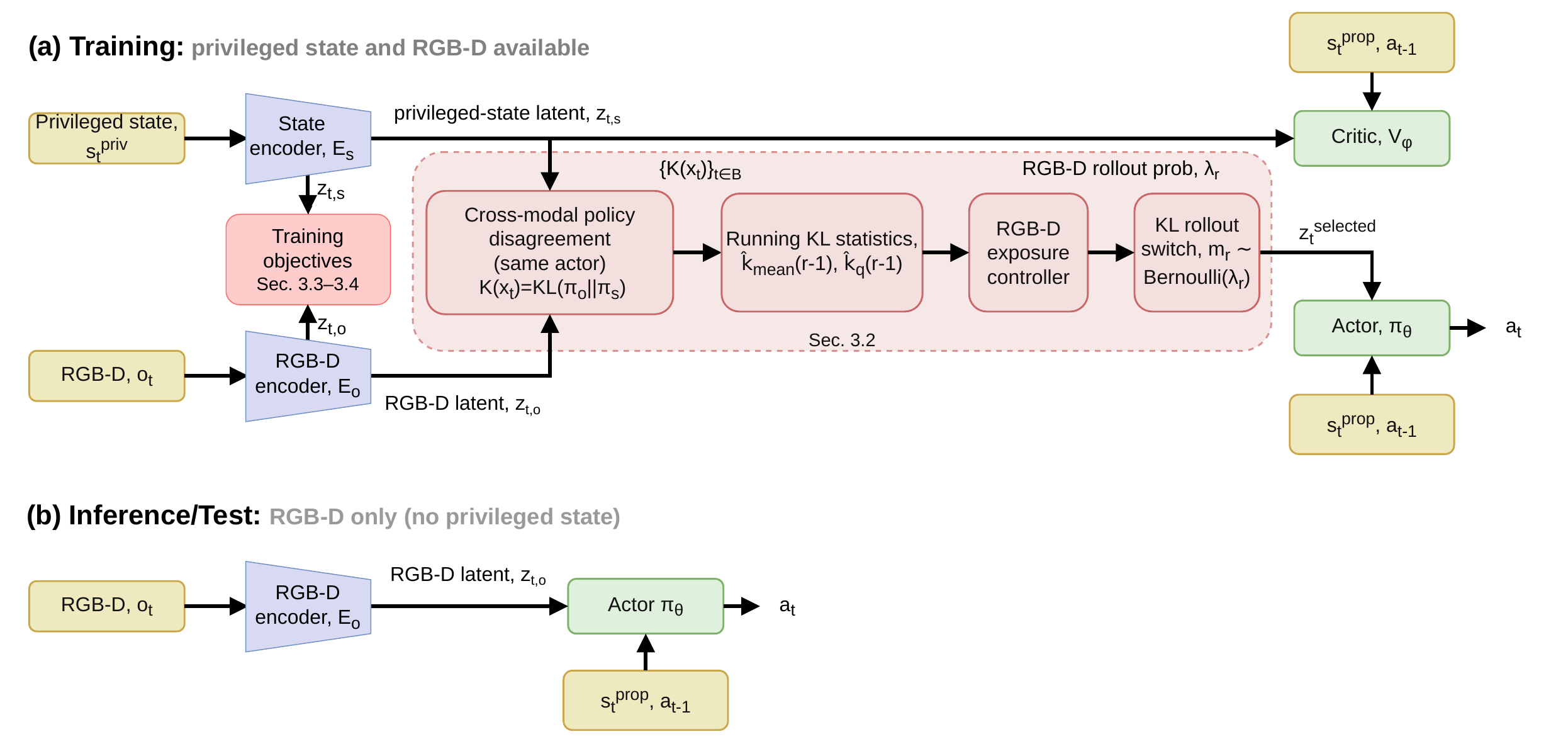}
\end{subfigure}

\vspace{0.6em}

\caption{\textbf{Overview of \method.} During training, cross-modal policy disagreement is aggregated into running statistics that determine the RGB-D exposure probability
$\lambda_r$ for subsequent actor rollouts. A KL-guided rollout switch samples either the privileged-state or RGB-D pathway and keeps the selected pathway fixed throughout the rollout, while the critic remains anchored to privileged state. At inference, privileged state and the
switching mechanism are removed, and the actor operates only from RGB-D, proprioception, and the previous action.}
\label{fig:focus-ideas}
\end{figure}

\subsection{KL-guided modality switching}
\label{sec:method:switch}

\method controls which observation pathway generates each actor rollout. For each rollout, it uses cross-modal policy disagreement to determine the probability of selecting the RGB-D pathway, and keeps the selected modality fixed throughout the rollout. It measures how much the actor's action distribution changes when the privileged-state input is replaced by the RGB-D input, and converts this disagreement into an RGB-D rollout probability \(\lambda_r\). A single Bernoulli variable is sampled per rollout and held fixed throughout that rollout, so the actor does not switch observation pathways between timesteps. We use this KL only as a measure of cross-modal behavioral compatibility, not as a certificate of task competence.

For each sampled transition, the same actor can be evaluated using either
the privileged-state or RGB-D latent while keeping proprioception and the
previous action fixed. We denote the resulting action distributions by
\begin{equation}
\pi_s(\cdot\mid t)
=
\pi_\theta\!\left(
\cdot\mid
\mathbf{z}_{t,s},
\mathbf{s}_t^{\mathrm{prop}},
\mathbf{a}_{t-1}
\right),
\end{equation}
and
\begin{equation}
\pi_o(\cdot\mid t)
=
\pi_\theta\!\left(
\cdot\mid
\mathbf{z}_{t,o},
\mathbf{s}_t^{\mathrm{prop}},
\mathbf{a}_{t-1}
\right).
\end{equation}
We use the privileged-state pathway as a training-time anchor and compute,
for each timestep in the collected rollout, the policy disagreement
induced by replacing it with the RGB-D pathway:
\begin{equation}
\label{eq:policy-kl}
K(\mathbf{x}_t)
=
\mathrm{KL}\!\left(
\pi_o(\cdot\mid t)
\,\Vert\,
\pi_s(\cdot\mid t)
\right).
\end{equation}
The resulting timestep-wise KL values are aggregated into running
statistics that determine the rollout-level RGB-D exposure probability.

At the beginning of rollout \(r\), the switch uses running estimates of
cross-modal disagreement obtained from the preceding rollout updates.
Let \(\hat{k}_{\mathrm{mean}}(r-1)\) and \(\hat{k}_{q}(r-1)\) denote
the running mean and upper-\(q\)-quantile estimates of
\(K(\mathbf{x}_t)\), with \(q=0.9\) in our experiments.
These statistics determine two RGB-D exposure factors for rollout \(r\):
\begin{equation}
\label{eq:kl-budget}
\begin{aligned}
\lambda_{\mathrm{mean}}(r)
&=
\min\!\left\{
\frac{\varepsilon_{\mathrm{mean}}}
{\max(\hat{k}_{\mathrm{mean}}(r-1),K_{\min})},\,1
\right\},\\
\lambda_q(r)
&=
\min\!\left\{
\frac{\varepsilon_q}
{\max(\hat{k}_{q}(r-1),K_{\min})},\,1
\right\},\\
\end{aligned}
\end{equation}
where \(\varepsilon_{\mathrm{mean}}\) and \(\varepsilon_q\) are fixed hyperparameters specifying the mean and upper-quantile disagreement levels at which RGB-D exposure begins to be reduced, and \(K_{\min}>0\) is a numerical floor. When the measured disagreement is at or below the corresponding \(\varepsilon\), the exposure factor is one and when above it, exposure decreases inversely with disagreement. For example, disagreement equal to \(2\varepsilon\) yields an exposure factor of \(0.5\). The upper-quantile term therefore limits RGB-D exposure when a subset of states still exhibits large cross-modal disagreement. We combine this KL-controlled exposure with a warmup schedule \(w(r)\) and a maximum exposure \(\lambda_{\max}\), and smooth the resulting target probability over training:
\begin{equation}
\lambda_{\text{target}}(r)
=
\min\{w(r)\lambda_{\max},\lambda_{\mathrm{mean}}(r),\lambda_q(r)\}.
\end{equation}
The controller state is smoothed at each controller update \(j\) as
\begin{equation}
\lambda^{(j+1)}
=
(1-\alpha)\lambda^{(j)}
+\alpha\,\lambda_{\text{target}}(r).
\end{equation}
We denote by \(\lambda_r\) the current controller value when rollout \(r\) is sampled.
RGB-D exposure therefore increases when the two modalities agree and decreases when mismatch becomes too large. The quantile term limits RGB-D exposure when a subset of samples still exhibits large policy disagreement.

At the start of each rollout, we sample a single Bernoulli variable
\(m_r\sim\mathrm{Bernoulli}(\lambda_r)\) and keep it fixed throughout
the rollout. The actor then receives
\begin{equation}
\label{eq:method:hard-switch}
\mathbf{z}^{\mathrm{selected}}_t
=
m_r\,\mathbf{z}_{t,o}
+
(1-m_r)\,\mathbf{z}_{t,s}.
\end{equation}
This keeps the actor input modality fixed within each rollout, providing
regulated RGB-D exposure without mixing privileged-state- and RGB-D-based
actions within the same trajectory.

During optimization of rollout \(r\), we evaluate the cross-modal policy
disagreement over its minibatches and update
$\hat{k}_{\mathrm{mean}}(r)$ and $\hat{k}_{q}(r)$.
Once the rollout update is complete, these running statistics inform the
RGB-D exposure probability for the next rollout.

\subsection{Representation alignment objective}
\label{sec:method:repr}

The representation objective trains the RGB-D pathway to induce control
behavior compatible with the privileged-state pathway. Our primary term
is a policy-output alignment loss. 
For corresponding privileged-state and RGB-D inputs, we evaluate the
same actor using $\mathbf{z}_{t,s}$ and $\mathbf{z}_{t,o}$ while keeping
proprioception and the previous action fixed, and denote the resulting
action distributions by $\pi_s$ and $\pi_o$. During this loss, the actor and fusion network are
frozen, so the RGB-D encoder is updated to reduce their policy disagreement:
\begin{equation}
\mathcal{L}_{\pi,\mathrm{repr}}
=
\tfrac12\mathrm{KL}(\pi_s\|\pi_o)
+
\tfrac12\mathrm{KL}(\pi_o\|\pi_s).
\end{equation}
This objective reduces cross-modal policy disagreement. Its
$\mathrm{KL}(\pi_o\|\pi_s)$ component corresponds to the directional
policy disagreement monitored by the KL-guided switch.

The full formulation additionally includes auxiliary representation
regularizers. When enabled, a cross-attention aligner \(f_\varphi\) produces
paired aligned targets
\begin{equation}
(\hat{\mathbf{z}}_{t,s},\hat{\mathbf{z}}_{t,o})
=
f_\varphi(\mathbf{z}_{t,s},\mathbf{z}_{t,o}),
\end{equation}
which support an aligner-guided auxiliary loss
\(\mathcal{L}_{\mathrm{align}}\). We additionally use raw-latent
regularization \(\mathcal{L}_{\mathrm{raw}}\) and a decoder objective
\(\mathcal{L}_{\mathrm{dec}}\).

The full objective is
\begin{equation}
\label{eq:align}
\mathcal{L}_{\mathrm{repr}}
=
\lambda_{\pi,\mathrm{repr}}\mathcal{L}_{\pi,\mathrm{repr}}
+
\mathcal{L}_{\mathrm{align}}
+
\mathcal{L}_{\mathrm{raw}}
+
\mathcal{L}_{\mathrm{dec}} .
\end{equation}

\subsection{Cross-modal control regularization}
\label{sec:method:ctrl}

Unlike the representation objective in Section~\ref{sec:method:repr},
which freezes the actor and fusion module and updates the RGB-D
representation pathway, these auxiliary terms regularize the control pathway
during RL optimization. They encourage the privileged-state and RGB-D
pathways to induce compatible policy outputs and value predictions, using the
privileged-state outputs as stop-gradient anchors.

We denote by $V_s$ and $V_o$ the value predictions obtained by evaluating the
critic from the privileged-state and RGB-D pathways, respectively. The
RGB-D-conditioned value $V_o$ is used only for this auxiliary agreement
objective. Rollout value estimation, return computation, and PPO advantages
remain anchored to $V_s$. We define
\begin{align}
\mathcal{L}_{\pi}
&=
\mathbb{E}\!\left[
\mathrm{KL}(\pi_s \Vert \pi_o)
\right], \\
\mathcal{L}_{V}
&=
\tfrac{1}{2}\,
\mathbb{E}\!\left[
(V_s - V_o)^2
\right].
\end{align}

When enabled, these agreement terms augment the PPO objective, with
$\lambda_{\pi}(t)$ and $\lambda_V(t)$ denoting their corresponding scheduled
weights:
\begin{equation}
\label{eq:method:rl-full}
\mathcal{L}_{\mathrm{RL}}
=
\mathcal{L}_{\mathrm{PPO}}
+\lambda_{\pi}(t)\mathcal{L}_{\pi}
+\lambda_{V}(t)\mathcal{L}_{V}.
\end{equation}

\subsection{Training procedure}
\label{sec:method:train}

Training uses a single on-policy PPO loop with separate RL and representation optimizers. The overall rollout flow is illustrated in Figure~\ref{fig:focus-ideas}, with the complete training procedure given in Appendix~\ref{app:algorithm_overview}.  
The RL optimizer steps the actor-critic and privileged-state pathway,
while the representation optimizer steps the RGB-D encoder and auxiliary
representation components. Implementation-level optimizer routing is
provided in Appendix~\ref{app:training_details}. For each rollout, the KL-guided switch selects either the privileged-state or RGB-D pathway for the actor and keeps that selection fixed throughout the rollout, while the critic remains anchored to the privileged-state pathway. Cross-modal policy disagreement measured during training updates the running statistics
used to regulate RGB-D exposure in subsequent rollouts.

\section{Experimental Setup}
\label{sec:expsetup}

\paragraph{Tasks and baselines.}
We evaluate on Reach-Cube, Lift-Cube, Push-Cube, Open-Drawer, and two-block Pick-and-Place (PnP) in IsaacLab (See Appendix \ref{app:protocol_details}). We compare \method against states-only as a privileged-observability reference, RGB-D-only, RGB-D-asymmetric, and a 3-stage teacher-student baseline that trains a privileged-state PPO teacher, distills an RGB-D student from simulator rollouts, and then fine-tunes the student with PPO. These baselines are matched IsaacLab instantiations of the corresponding paradigms rather than direct reproductions of every specialized prior system.

\paragraph{Protocol and metrics.}

All methods are trained with two training seeds and evaluated with five
evaluation seeds per checkpoint. Entries report mean \(\pm\) standard
deviation over training seeds after averaging evaluation seeds. Test
performance is measured by test episode success rate over a fixed
9000-step evaluation rollout. End-to-end training efficiency is measured
primarily by budget-normalized \textit{Training Success AUC},
computed as
\(\frac{1}{T}\int_0^T s(t)\,dt\),
where \(s(t)\) is training success occupancy and \(T\) is the
task-specific training horizon. This metric is computed over the complete
training procedure: for \method, both privileged-state and RGB-D rollouts
are included, while teacher--student includes teacher PPO, RGB-D
distillation, and student PPO fine-tuning. \textit{Train Step@50} is
reported as a supplementary diagnostic in Appendix~\ref{app:steps50}.
Appendix~\ref{app:diagnostics} lists additional diagnostics and metrics
implemented in IsaacLab.

\paragraph{Implementation details.}
All methods use a single \(128\times128\) RGB-D camera with matched preprocessing and shared core implementation choices unless otherwise stated (See Appendix \ref{app:implementation_details}). Tables~\ref{tab:arch-summary} and~\ref{tab:training-defaults-main} summarize architecture and hyperparameters.

\subsection{Main Results}
\label{subsec:main}

Table~\ref{tab:per-task-test-success} reports test success, Table~\ref{tab:success-occupancy-auc} reports budget-normalized \textit{Training Success AUC}, and Figure~\ref{fig:train-success-vs-steps} shows training curves on the more discriminative tasks. Supplementary \textit{Train Step@50} results and full five-task training curves are provided in Appendices~\ref{app:steps50} and~\ref{app:training-curves}, respectively.

On Lift-Cube, \method improves test success from 0.35 to 0.98 and normalized training success AUC from 0.16 to 0.45. On Push-Cube, it improves test success from 0.75 to 0.85 and normalized training success AUC from 0.29 to 0.34. On Pick-and-Place, it improves test success from 0.47 to 0.86 and normalized training success AUC from 0.12 to 0.61, a 5.0$\times$ increase in this end-to-end training-efficiency metric. Reach-Cube is solved by all methods, while Open-Drawer is less
diagnostic because dense reward shaping lets several methods approach saturation.

Across the evaluation tasks, \method{} improves both average RGB-D test success
and average end-to-end training efficiency. Relative to the strongest
RGB-D-at-test baseline on each task, average test success improves from
0.71 to 0.93. When comparing complete training pipelines,
budget-normalized Training Success AUC improves from 0.47 to 0.65.
These averages exclude the states-only reference, which uses privileged
state at test time.

The states-only policy is a privileged-observability reference rather than an oracle upper bound. On Lift-Cube, \method attains higher mean success than this reference, but the two states-only training runs show substantial variation. Given the limited number of training seeds, we
do not interpret this gap as evidence that RGB-D observations are intrinsically superior to privileged state. The asymmetric actor-critic baseline also exhibits unstable behavior on some tasks, including low performance on Push-Cube and substantial variation on Pick-and-Place,
and the present evaluation does not isolate the source of this behavior. Teacher-student reaches high success during its final student PPO stage, but its end-to-end training-success AUC is lower when the preceding teacher and distillation stages are included.

\paragraph{Qualitative test-time behavior}
Figure~\ref{fig:test-rollouts} shows representative successful test-time
rollouts of \method{} on Lift-Cube, Push-Cube, and Pick-and-Place.
The sequences illustrate the manipulation behavior underlying the
quantitative success rates in Table~\ref{tab:per-task-test-success}.

\begin{figure*}[t]
    \centering

    \includegraphics[
        width=0.98\textwidth
    ]{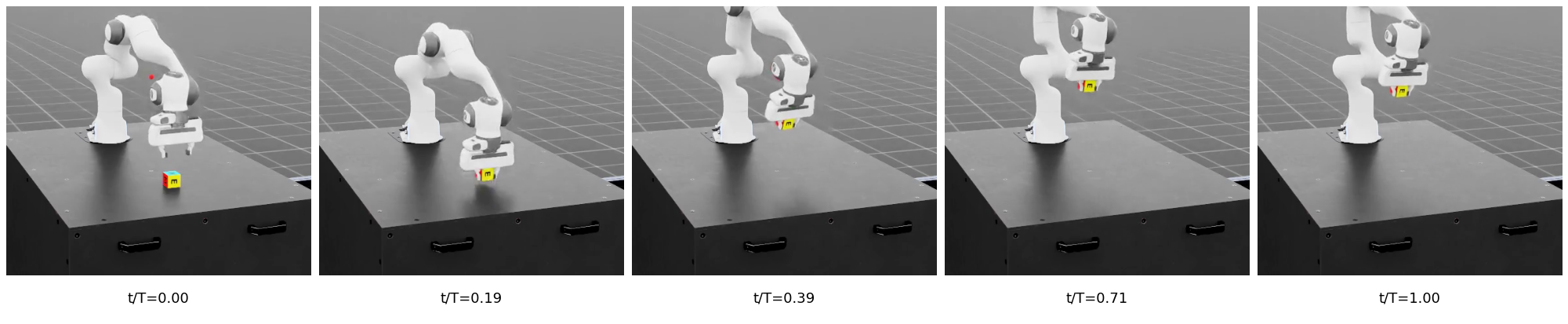}

    \vspace{0.3em}

    \includegraphics[
        width=0.98\textwidth
    ]{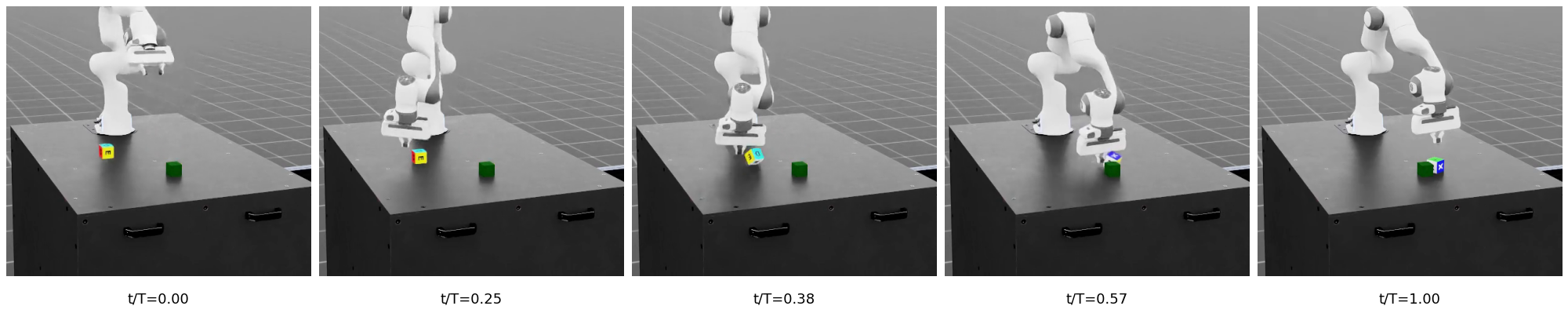}

    \vspace{0.3em}

    \includegraphics[
        width=0.98\textwidth
    ]{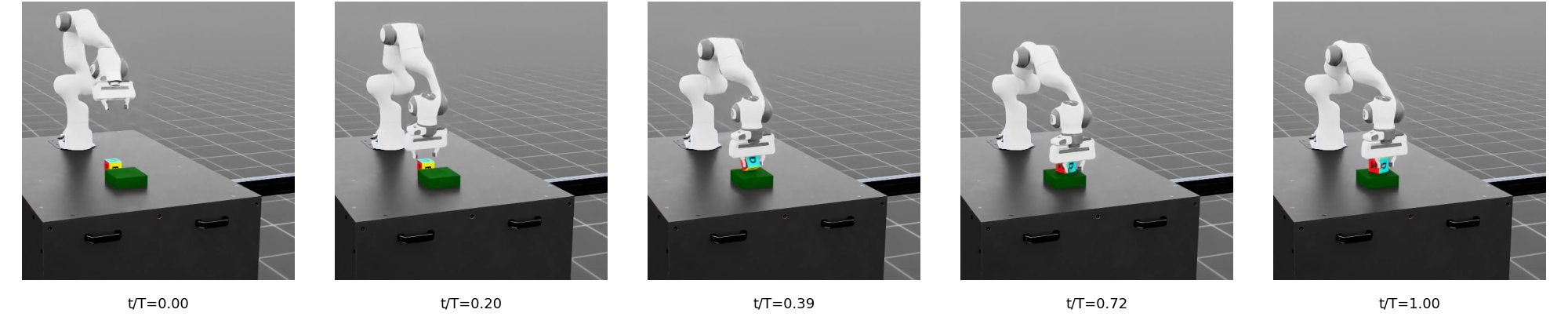}

    \caption{\textbf{Representative successful test-time rollouts of \method{}.}
    Sequential frames show \method{} executing Lift-Cube (top),
    Push-Cube (middle), and Pick-and-Place (bottom) using the deployment
    observation pathway. At test time, privileged state and modality
    switching are removed, and the policy acts from RGB-D observations,
    proprioception, and the previous action. Frames are labeled by
    normalized rollout time \(t/T\).}
    \label{fig:test-rollouts}
\end{figure*}

\begin{table}[t]
\caption{\textbf{Test episode success rate} over a 9k-step evaluation rollout. An episode is counted as successful if the task succeeds at least once before termination. Higher is better.}
\label{tab:per-task-test-success}
\centering
\scriptsize
\setlength{\tabcolsep}{2.5pt}
\renewcommand{\arraystretch}{1.08}
\begin{tabularx}{\linewidth}{@{}l *{5}{>{\centering\arraybackslash}X}@{}}
\toprule
\textbf{Method} & \textbf{Reach} & \textbf{Lift} & \textbf{Push} & \textbf{Open\,-\,Drawer} & \textbf{Pick\,-\,and\,-\,Place}$^{\dagger}$ \\
\midrule
states-only & 1.000 $\pm$ 0.000 & 0.555 $\pm$ 0.628 & 0.972 $\pm$ 0.040 & 1.000 $\pm$ 0.000 & 1.000 $\pm$ 0.000 \\
RGB-D-only & 1.000 $\pm$ 0.000 & 0.233 $\pm$ 0.024 & 0.753 $\pm$ 0.098 & 0.983 $\pm$ 0.024 & 0.453 $\pm$ 0.122 \\
RGB-D-asymmetric & 1.000 $\pm$ 0.000 & 0.350 $\pm$ 0.338 & 0.192 $\pm$ 0.059 & 0.811 $\pm$ 0.157 & 0.467 $\pm$ 0.660 \\
teacher-student (3-stage, final student) & 1.000 $\pm$ 0.000 & 0.181 $\pm$ 0.122 & 0.595 $\pm$ 0.032 & 0.978 $\pm$ 0.031 & 0.289 $\pm$ 0.110 \\
\midrule
\rowcolor{oursblue}
\textbf{\method} (ours) & 1.000 $\pm$ 0.000 & 0.975 $\pm$ 0.012 & 0.850 $\pm$ 0.071 & 0.978 $\pm$ 0.018 & 0.858 $\pm$ 0.069 \\
\bottomrule
\end{tabularx}
\vspace{2pt}
\raggedright\footnotesize
$^{\dagger}$ For Pick-and-Place, evaluation requires 4 consecutive on-target steps; training uses a one-step success signal.
\end{table}

\begin{table}[t]
\caption{\textbf{Budget-normalized Training Success AUC summarizes sample efficiency} over the full training trajectory. AUC is divided by each method's task-specific simulator-interaction horizon. For teacher-student, we report both the final student phase alone as a diagnostic and the comparable end-to-end 3-stage cost used in our conclusions. Higher is better.}
\label{tab:success-occupancy-auc}
\centering
\scriptsize
\setlength{\tabcolsep}{2.5pt}
\renewcommand{\arraystretch}{1.08}
\begin{tabularx}{\linewidth}{@{}l *{5}{>{\centering\arraybackslash}X}@{}}
\toprule
\textbf{Method} & \textbf{Reach} & \textbf{Lift} & \textbf{Push} & \textbf{Open\,-\,Drawer} & \textbf{Pick\,-\,and\,-\,Place} \\
\midrule
states-only & 0.966 $\pm$ 0.002 & 0.211 $\pm$ 0.268 & 0.447 $\pm$ 0.007 & 0.912 $\pm$ 0.016 & 0.656 $\pm$ 0.009 \\
RGB-D-only & 0.889 $\pm$ 0.003 & 0.023 $\pm$ 0.002 & 0.288 $\pm$ 0.036 & 0.859 $\pm$ 0.021 & 0.094 $\pm$ 0.056 \\
RGB-D-asymmetric & 0.891 $\pm$ 0.020 & 0.158 $\pm$ 0.190 & 0.094 $\pm$ 0.003 & 0.886 $\pm$ 0.015 & 0.122 $\pm$ 0.171 \\
teacher-student (final student) & 0.958 $\pm$ 0.010 & 0.258 $\pm$ 0.321 & 0.233 $\pm$ 0.012 & 0.930 $\pm$ 0.013 & 0.155 $\pm$ 0.003 \\
teacher-student (3-stage, e2e) & 0.411 $\pm$ 0.004 & 0.111 $\pm$ 0.137 & 0.100 $\pm$ 0.005 & 0.399 $\pm$ 0.006 & 0.071 $\pm$ 0.001 \\
\midrule
\rowcolor{oursblue}
\textbf{\method} (ours) & 0.969 $\pm$ 0.001 & 0.451 $\pm$ 0.040 & 0.337 $\pm$ 0.167 & 0.864 $\pm$ 0.025 & 0.612 $\pm$ 0.037 \\
\bottomrule
\end{tabularx}
\end{table}

\begin{figure}[t]
    \centering

    \begin{subfigure}[t]{0.32\linewidth}
        \centering
        \includegraphics[width=\linewidth]{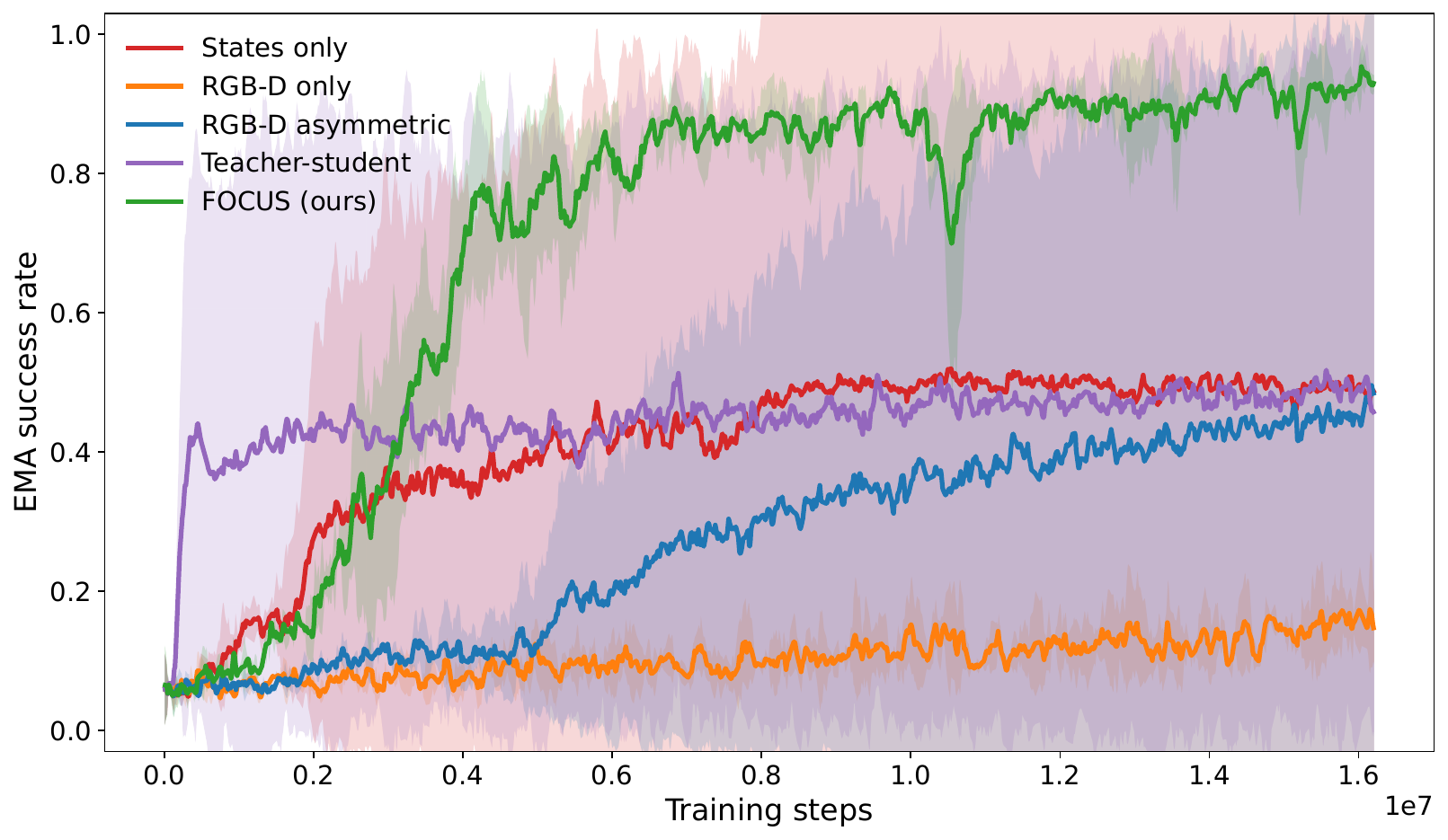}
        \caption{\textbf{Lift-Cube}}
        \label{fig:train-lift-success}
    \end{subfigure}\hfill
    \begin{subfigure}[t]{0.32\linewidth}
        \centering
        \includegraphics[width=\linewidth]{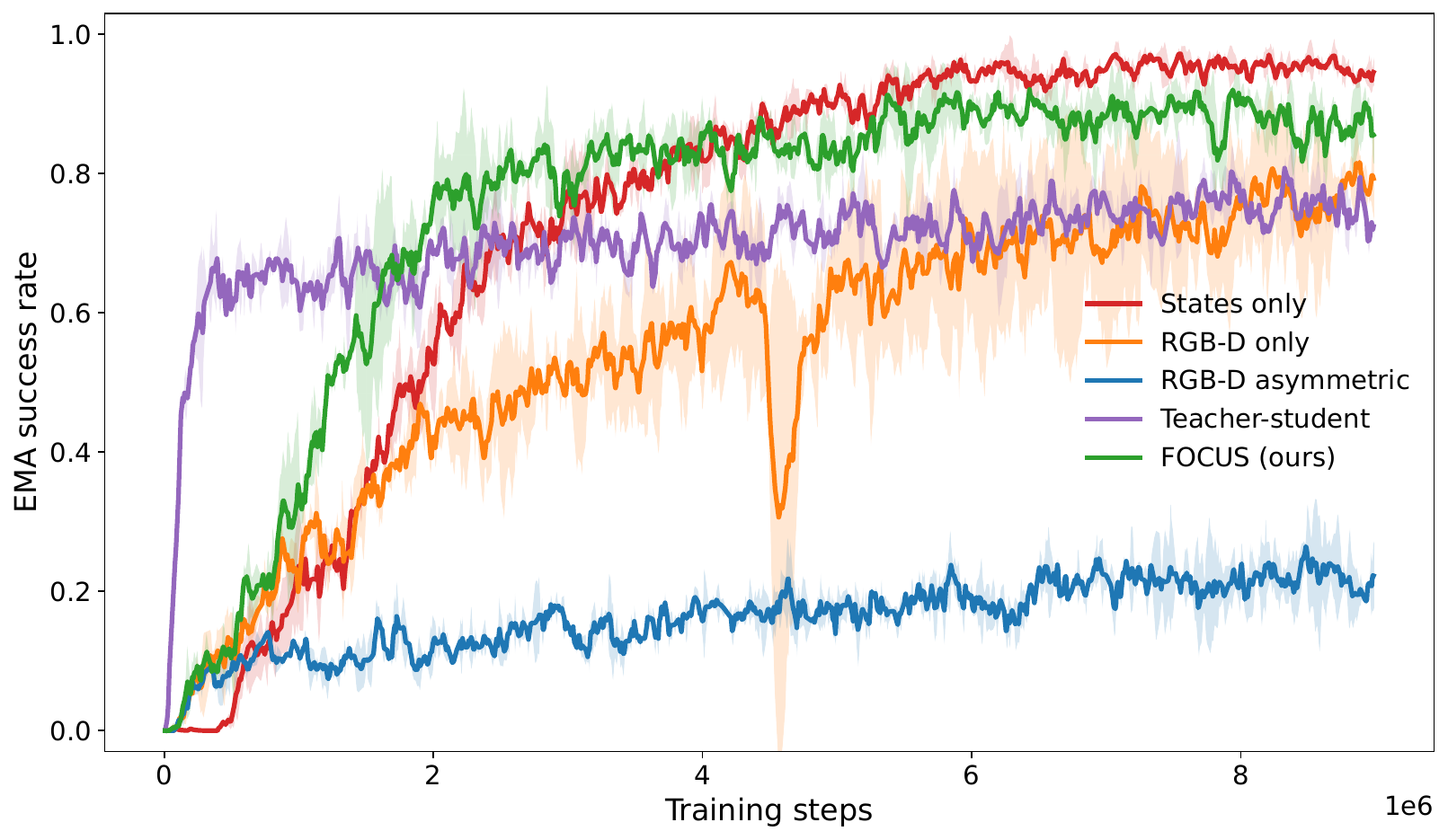}
        \caption{\textbf{Push-Cube}}
        \label{fig:train-push-success}
    \end{subfigure}\hfill
    \begin{subfigure}[t]{0.32\linewidth}
        \centering
        \includegraphics[width=\linewidth]{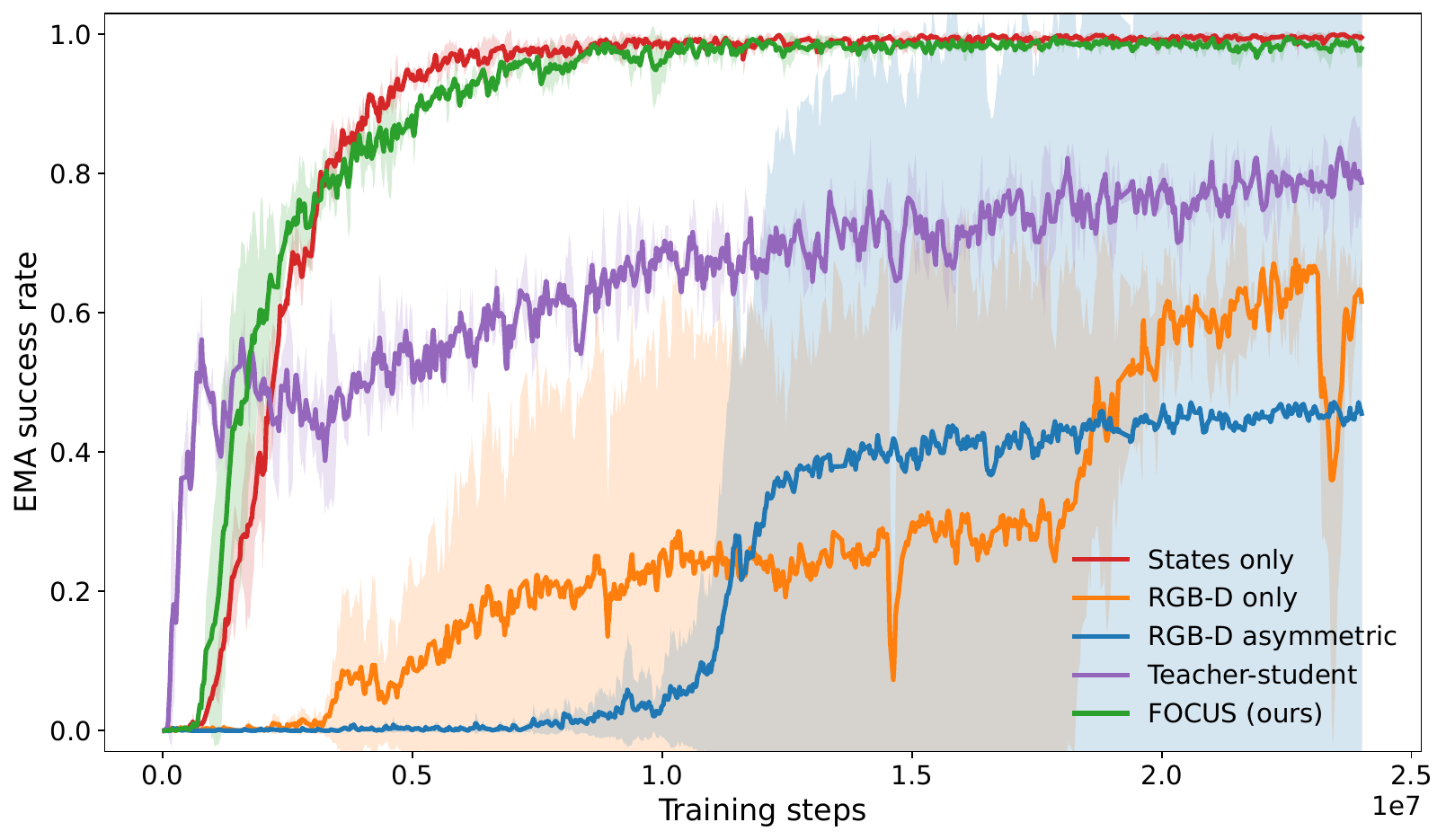}
        \caption{\textbf{Pick-and-Place}}
        \label{fig:train-pick-place-success}
    \end{subfigure}
    \caption{\textbf{Training success EMA on the more discriminative tasks.}
        Teacher--student curves show the final student PPO phase, while
        Table~\ref{tab:success-occupancy-auc} reports end-to-end normalized AUC.
        Full curves and Step@50 are provided in
        Appendices~\ref{app:training-curves} and~\ref{app:steps50}, respectively.}
    \label{fig:train-success-vs-steps}
\end{figure}

\subsection{Ablation Study}
\label{subsec:ablations}

We ablate the modality-switching and alignment components of \method. \textbf{Fixed schedule ramp} removes KL feedback from the modality switch and increases the RGB-D rollout probability according only to the predefined warmup schedule \(w(r)\), while retaining rollout-level
modality sampling. \textbf{w/o KL upper-quantile budget} retains KL-guided switching but
removes the upper-tail constraint, so RGB-D exposure is regulated only by the running mean policy disagreement, \(\lambda_{\mathrm{budget}}(r)=\lambda_{\mathrm{mean}}(r)\). \textbf{\method-simple} retains the KL-guided rollout switch and policy-output representation loss \(\mathcal{L}_{\pi,\mathrm{repr}}\), while removing the auxiliary representation and control-agreement terms.

Table~\ref{tab:ablations} summarizes ablations on the three more discriminative tasks. The main trend is that controlled RGB-D rollout exposure and policy-level image-state agreement are the most reliable ingredients in our evaluation. A fixed exposure ramp can reach high test success on some tasks, but gives lower normalized AUC on Pick-and-Place than \method{} (0.361 vs. 0.612) and lower test success on Push-Cube (0.822 vs. 0.850), indicating that adaptive exposure provides a better sample-efficiency and success trade-off than a predetermined schedule on these tasks.

Figure~\ref{fig:push_kl_switch_diagnostics} provides a representative
Push-Cube diagnostic. The KL-guided switch keeps RGB-D exposure from
saturating too early, which lowers exposure-weighted upper-tail KL mismatch
relative to a fixed ramp and a variant without the upper-quantile budget.
This controlled exposure coincides with higher training success, supporting
tail-aware control as a way to avoid premature reliance on RGB-D rollouts.

\method{}-simple keeps only the KL-guided switch and image-encoder policy-output KL term, disabling all other representation and RL agreement losses. It remains competitive, reaching 1.000 test success on Lift-Cube and 0.994 on Pick-and-Place. This ablation suggests that controlled visual exposure and image-encoder alignment through policy-output KL are the most reliable contributors in our evaluation. We therefore treat the remaining auxiliary losses as optional task-dependent terms within the \method{} formulation, rather than as universally required components.

\begin{table}[t]
\caption{
\textbf{Ablations on the three more discriminative tasks.}
All variants use the same training budget for Push/Lift/PnP (1500/1500/2500 training iterations), are trained with two seeds, and are evaluated with five test seeds. AUC is budget-normalized by each task horizon. Higher is better for test success and normalized AUC.
}
\label{tab:ablations}
\centering
\scriptsize
\setlength{\tabcolsep}{2.6pt}
\renewcommand{\arraystretch}{1.04}
\begin{tabularx}{\linewidth}{@{}p{0.31\linewidth} *{6}{>{\centering\arraybackslash}X}@{}}
\toprule
\multirow{2}{*}{\textbf{Variant}} &
\multicolumn{3}{c}{\textbf{Test success}} &
\multicolumn{3}{c}{\textbf{Norm. Training Success AUC}} \\
\cmidrule(lr){2-4}\cmidrule(lr){5-7}
& \textbf{Push} & \textbf{Lift} & \textbf{Pick\,-\,and\,-\,Place}$^{\dagger}$ 
& \textbf{Push} & \textbf{Lift} & \textbf{Pick\,-\,and\,-\,Place} \\
\midrule
\method{} 
& \cellstd{0.850}{0.071} & \cellstd{0.975}{0.012} & \cellstd{0.858}{0.069}
& \cellstd{0.337}{0.167} & \cellstd{0.451}{0.040} & \cellstd{0.612}{0.037} \\

\midrule
\multicolumn{7}{@{}l}{\textit{Training-objective ablations}} \\
w/o decoder 
& \cellstd{0.669}{0.090} & \cellstd{0.986}{0.020} & \cellstd{0.997}{0.003}
& \cellstd{0.406}{0.104} & \cellstd{0.472}{0.003} & \cellstd{0.666}{0.080} \\

w/o aligner ($f_\varphi$ bypass) 
& \cellstd{0.854}{0.038} & \cellstd{0.964}{0.011} & \cellstd{0.458}{0.458}
& \cellstd{0.445}{0.013} & \cellstd{0.464}{0.030} & \cellstd{0.283}{0.283} \\

w/o RL agreement 
& \cellstd{0.856}{0.032} & \cellstd{0.950}{0.063} & \cellstd{0.822}{0.050}
& \cellstd{0.345}{0.063} & \cellstd{0.415}{0.102} & \cellstd{0.631}{0.064} \\

w/o image policy KL 
& \cellstd{0.097}{0.043} & \cellstd{0.975}{0.027} & \cellstd{0.586}{0.242}
& \cellstd{0.422}{0.126} & \cellstd{0.266}{0.196} & \cellstd{0.580}{0.029} \\

\midrule
\multicolumn{7}{@{}l}{\textit{Switching and visual-pathway ablations}} \\
fixed schedule ramp 
& \cellstd{0.822}{0.031} & \cellstd{0.997}{0.004} & \cellstd{0.953}{0.042}
& \cellstd{0.281}{0.063} & \cellstd{0.184}{0.218} & \cellstd{0.361}{0.175} \\

w/o KL upper-quantile budget 
& \cellstd{0.653}{0.020} & \cellstd{0.970}{0.043} & \cellstd{0.961}{0.028}
& \cellstd{0.371}{0.096} & \cellstd{0.317}{0.216} & \cellstd{0.480}{0.155} \\

top-view only 
& \cellstd{0.670}{0.067} & \cellstd{0.756}{0.039} & \cellstd{0.922}{0.006}
& \cellstd{0.471}{0.039} & \cellstd{0.499}{0.003} & \cellstd{0.587}{0.069} \\

RGB-D encoder from scratch 
& \cellstd{0.716}{0.001} & \cellstd{0.950}{0.016} & \cellstd{0.983}{0.017}
& \cellstd{0.406}{0.014} & \cellstd{0.381}{0.020} & \cellstd{0.664}{0.020} \\

\midrule
\rowcolor{simplegray}
\method{}-simple: KL-guided switch + policy-output KL
& \cellstd{0.831}{0.138} & \cellstd{1.000}{0.000} & \cellstd{0.994}{0.008}
& \cellstd{0.411}{0.074} & \cellstd{0.529}{0.034} & \cellstd{0.568}{0.021} \\

\bottomrule
\end{tabularx}

\vspace{2pt}
\raggedright\footnotesize
$^{\dagger}$ For Pick-and-Place, evaluation requires 4 consecutive on-target steps; training uses a one-step success signal.
\end{table}

\begin{figure}[t]
    \centering
    \includegraphics[
        width=0.95\linewidth,
    ]{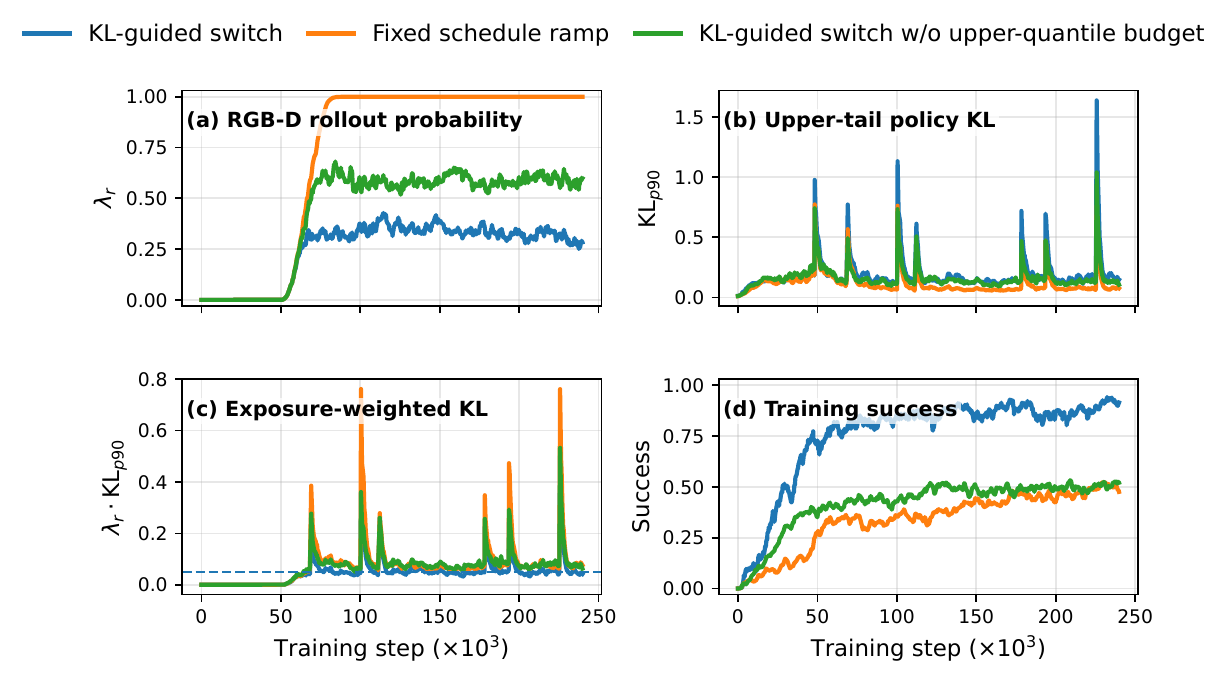}
    \caption{\textbf{Representative Push-Cube switch diagnostic.}
        The KL-guided switch reduces exposure-weighted upper-tail mismatch,
        \(\lambda_r \cdot \mathrm{KL}_{p90}\), relative to a fixed ramp and a variant
        without the upper-quantile budget, while improving training success.
        Here, \(\mathrm{KL}_{p90}\) denotes the 90th percentile of cross-modal policy KL,
        and \(\lambda_r\) is the Bernoulli probability that the actor uses RGB-D latents
        for an entire rollout.}
    \label{fig:push_kl_switch_diagnostics}
\end{figure}

\section{Conclusion}
\label{sec:conclusion}
We introduced FOCUS, a single-stage PPO-based framework that regulates
whether privileged-state or RGB-D latents generate on-policy actor rollouts
during training, while using RGB-D for action selection at test time.
Across five IsaacLab manipulation tasks, FOCUS improves average test success
relative to the strongest RGB-D-at-test baseline on each task from 0.71 to
0.93. When comparing complete training pipelines, budget-normalized training
success AUC improves from 0.47 to 0.65. Ablations suggest that controlled
RGB-D rollout exposure and image-encoder alignment through policy-output KL
are the most reliable contributors, while the remaining auxiliary losses
have task-dependent effects and motivate simpler future variants.

\section{Limitations}
\label{sec:limitations}

This study is limited to simulation, relatively short-horizon manipulation, and two training seeds, reflecting the engineering cost and long runtimes of full RGB-D PPO training. The method is also sensitive to task-specific reward engineering, which may affect the learned behavior and comparability across tasks. Some runs converge to locally successful but suboptimal behaviors. Given the limited evaluation suite, ablations should be interpreted as component-level trends rather than universal claims about which auxiliary losses are necessary. \method addresses the privileged-state-to-RGB-D observation gap and is orthogonal to the sim-to-real gap. Appearance shift, sensing noise, calibration error, latency, and dynamics mismatch would still require dedicated sim-to-real methods or real-world fine-tuning. Future work should evaluate \method{} on physical robots, broader manipulation suites, and less engineered training signals. More extensive task and seed coverage is also needed to isolate when the auxiliary alignment components are necessary versus when the lightweight \method{}-simple core is sufficient.

\clearpage

\acknowledgments{This work was carried out in the \href{https://www.embarclab.com/}{embARC Research Group} at the University of Toronto. This research is part of the University of Toronto's Acceleration Consortium, which receives funding from the CFREF-2022-00042 Canada First Research Excellence Fund. We thank Callum Sharrock for his contributions to this work during his internship in the lab.}

\bibliography{references}  

\clearpage

\appendix

\section*{Appendix Overview}
This appendix provides implementation and analysis details supporting the main paper. Appendix~\ref{app:algorithm_overview} gives the full training algorithm, Appendix~\ref{app:theory} analyzes the KL-guided modality switch, Appendix~\ref{app:repr-details} defines the representation-alignment objective, and Appendix~\ref{app:utils} describes latent post-processing. Appendix~\ref{app:protocol_details} gives task-suite and protocol details, and Appendix~\ref{app:implementation_details} summarizes implementation details including RGB-D preprocessing. We also include \textit{Train Step@50} results, full training curves, qualitative activation and latent visualizations, architecture and hyperparameter details, logged diagnostics, and task reward definitions for reproducibility.

\section{Algorithm overview}
\label{app:algorithm_overview}

For implementation-level detail, we explicitly denote by $g_\psi$ the
shared fusion module that combines a modality latent with proprioception
and the previous action before it is passed to the actor or critic:
\[
g_\psi:
\left(
\mathbf{z}_t,
\mathbf{s}_t^{\mathrm{prop}},
\mathbf{a}_{t-1}
\right)
\mapsto
\tilde{\mathbf{z}}_t.
\]
The same fusion module $g_\psi$, with shared parameters $\psi$, is used
for both actor and critic inputs and across the privileged-state and RGB-D
pathways. The actor receives the selected modality latent, whereas the
critic always receives the privileged-state latent. In the main text, we
absorb this implementation detail into the actor and critic notation.

\begingroup
\SetAlFnt{\footnotesize}
\setlength{\algomargin}{1em}
\SetAlgoInsideSkip{smallskip}
\SetAlgoSkip{smallskip}
\SetNlSty{}{}{}
\SetInd{0.4em}{0.9em}
\DontPrintSemicolon

\begin{algorithm}[H]
\caption{\textbf{Training Recipe for FOCUS}}
\label{alg:focus_overview}

\BlankLine

\textbf{Act (rollout collection).}\;
$\lambda_r \leftarrow \textsc{CurrentModalityProbability}()$\;
$m_r \sim \mathrm{Bernoulli}(\lambda_r)$ \tcp*[f]{snapshot current smoothed controller state}\;
\For{$t=1,\dots,T$}{
  observe $(\mathbf{s}^{priv}_t,\mathbf{s}^{prop}_t,o_t,\mathbf{a}_{t-1})$\;
  $z_{s,t}\leftarrow E_s(\mathbf{s}^{priv}_t)$;\quad $z_{o,t}\leftarrow E_o(o_t)$\;
  $z_t^{\text{pre}}\leftarrow (1-m_r)z_{s,t}+m_r z_{o,t}$\;
$\tilde z_t^{act}\leftarrow g_\psi(z_t^{\text{pre}},\mathbf{s}^{prop}_t,\mathbf{a}_{t-1})$\;
$\tilde z_t^{crt}\leftarrow g_\psi(z_{s,t},\mathbf{s}^{prop}_t,\mathbf{a}_{t-1})$\;
$a_t\sim \pi_\theta(\cdot\mid \tilde z_t^{act})$\;
  step env; store $(m_r,\lambda_r,\mathbf{s}^{priv}_t,o_t,\mathbf{s}^{prop}_t,\mathbf{a}_{t-1},a_t,r_t,d_t)$\;
}
compute returns/advantages using state-anchored critic\;

\BlankLine
\textbf{Update (on-policy optimization).}\;
\For{PPO epochs/minibatches}{
  load minibatch (incl.\ stored $m_r$ and $\lambda_r$)\;
  recompute $(z_{s,t},z_{o,t})$ and form $(\tilde z_t^{act},\tilde z_t^{crt})$ using stored $m_r$\;
  update \(\{E_o,f_\varphi,D_s\}\) with \(\mathcal{L}_{\pi,\mathrm{repr}}\) 
        (+ optional \(\mathcal{L}_{\mathrm{align}},\mathcal{L}_{\mathrm{raw}},\mathcal{L}_{\mathrm{dec}}\))\;
  update \(\{\pi_\theta,g_\psi,E_s\}\) with \(\mathcal{L}_{\mathrm{PPO}}\)
  (+ optional \(\mathcal{L}_{\pi},\mathcal{L}_{V}\))\;
  update $(k_{\text{mean}},k_{q90})$ from $\mathrm{KL}(\pi_o\|\pi_s)$\;
}
\end{algorithm}
\endgroup

\section{KL-Controlled Modality Switching}
\label{app:theory}

This appendix characterizes policy divergence induced by stochastic
modality switching and motivates KL-budgeted exposure control.

\subsection{Setup}
\label{app:setup}

For the theoretical analysis, we make the implementation-level shared
fusion module explicit so that we can distinguish raw encoder latents
from the effective inputs presented to the actor.

We analyze the policy as a function of its \emph{effective actor input}. At timestep \(t\), let
\[
\tilde{\mathbf z}^{\,s}_t
:= g_\psi(\mathbf z_{t,s}, \mathbf s_t^{\mathrm{prop}}, \mathbf a_{t-1}),
\qquad
\tilde{\mathbf z}^{\,o}_t
:= g_\psi(\mathbf z_{t,o}, \mathbf s_t^{\mathrm{prop}}, \mathbf a_{t-1}),
\]
denote the fused actor inputs obtained from the state and RGB-D pathways.

We consider a shared diagonal-Gaussian policy
\[
\pi(\cdot\mid \tilde{\mathbf z})=\mathcal{N}(\mu(\tilde{\mathbf z}),\Sigma),
\]
where the mean \(\mu(\tilde{\mathbf z})\) depends on the effective actor input \(\tilde{\mathbf z}\), while the covariance
\(\Sigma\succ 0\) is shared across modalities. In our implementation, \(\Sigma\) is diagonal and is parameterized by a learned global standard-deviation parameter, so it does not depend on the input. Let \(\sigma_{\min}^2\) denote the smallest eigenvalue of \(\Sigma\).

Define the fused modality gap
\[
\Delta \tilde{\mathbf z}_t
:=
\tilde{\mathbf z}^{\,o}_t-\tilde{\mathbf z}^{\,s}_t.
\]
These inputs induce the modality-conditioned policies
\[
\pi_s(\cdot\mid t):=\pi(\cdot\mid \tilde{\mathbf z}^{\,s}_t),
\qquad
\pi_o(\cdot\mid t):=\pi(\cdot\mid \tilde{\mathbf z}^{\,o}_t).
\]

\subsection{Stochastic switching}
\label{app:bernoulli}

For analysis, we consider Bernoulli switching at a generic timestep \(t\). In implementation, the sampled mask is held fixed over an entire rollout \(r\).

Using the rollout-wise mask
\[
m_r\sim\mathrm{Bernoulli}(\lambda_r),
\]
define the switched effective actor input
\[
\tilde{\mathbf z}^{\,\mathrm{pre}}_t
=
(1-m_r)\tilde{\mathbf z}^{\,s}_t + m_r \tilde{\mathbf z}^{\,o}_t
=
\tilde{\mathbf z}^{\,s}_t + m_r\,\Delta \tilde{\mathbf z}_t.
\]

Defining \(K_t \equiv K(\mathbf{x}_t)\), with
\[
K_t
:=
\mathrm{KL}\!\left(\pi_o(\cdot\mid t)\Vert\pi_s(\cdot\mid t)\right),
\]
the expected divergence satisfies
\[
\mathbb E_{m_r}\!\left[
\mathrm{KL}\!\left(
\pi(\cdot\mid \tilde{\mathbf z}^{\,\mathrm{pre}}_t)
\Vert
\pi_s(\cdot\mid t)
\right)
\right]
=
\lambda_r\,K_t,
\]
and is therefore linear in the switching probability. The auxiliary result below relates \(K_t\) to fused-input mismatch under a local Lipschitz assumption.

\subsection{KL-budgeted exposure control}
\label{app:budget}

At the beginning of rollout \(r\), let
\(\hat{k}_{\mathrm{mean}}(r-1)\) and \(\hat{k}_{q}(r-1)\) denote the
running estimates of the mean and an upper quantile (e.g., \(q=0.9\))
of \(K_t\) obtained from preceding rollout updates.
Given budgets \(\varepsilon_{\mathrm{mean}},\varepsilon_q>0\), define
\[
\lambda_{\mathrm{budget}}(r)
=
\min\!\left\{
\frac{\varepsilon_{\mathrm{mean}}}{\max(\hat{k}_{\mathrm{mean}}(r - 1),K_{\min})},
\frac{\varepsilon_q}{\max(\hat{k}_{q}(r - 1),K_{\min})},
1
\right\}.
\]

With warmup \(w(r)\in[0,1]\) and cap \(\lambda_{\max}\le1\), the target exposure is
\[
\lambda_{\mathrm{target}}(r)
=
\min\{w(r)\lambda_{\max},\,\lambda_{\mathrm{budget}}(r)\},
\]
and the controller state is smoothed at each controller update \(j\) via
\[
\lambda^{(j+1)}
=
(1-\alpha)\lambda^{(j)}
+\alpha\,\lambda_{\mathrm{target}}(r).
\]
For rollout \(r\), let \(j_r\) denote the controller update immediately
preceding rollout sampling, and define
\[
\lambda_r := \lambda^{(j_r)}.
\]

\begin{proposition}[Exposure-weighted KL control]
\label{prop:tail-budget}
Consider a rollout batch at rollout \(r\), and let
\[
K_t =
\mathrm{KL}\!\left(\pi_o(\cdot\mid t)\Vert \pi_s(\cdot\mid t)\right)
\]
denote the cross-modal policy mismatch for the matched state and RGB-D
observations at timestep \(t\). For the Bernoulli switch
\(m_r\sim\mathrm{Bernoulli}(\lambda_r)\),
\[
\mathbb E_{m_r}
\left[
\mathrm{KL}\!\left(
\pi(\cdot\mid \tilde{\mathbf z}^{\,\mathrm{pre}}_t)
\Vert
\pi_s(\cdot\mid t)
\right)
\middle| t
\right]
=
\lambda_r K_t .
\]
If the applied exposure satisfies
\[
\lambda_r \le
\min\!\left\{
\frac{\varepsilon_{\mathrm{mean}}}
{\max(\hat{k}_{\mathrm{mean}}(r),K_{\min})},
\frac{\varepsilon_q}
{\max(\hat{k}_q(r),K_{\min})},
1
\right\},
\]
then the exposure-weighted mean mismatch satisfies
\[
\mathbb E_t
\left[
\mathbb E_{m_r}
\left[
\mathrm{KL}\!\left(
\pi(\cdot\mid \tilde{\mathbf z}^{\,\mathrm{pre}}_t)
\Vert
\pi_s(\cdot\mid t)
\right)
\middle| t
\right]
\right]
=
\lambda_r \hat{k}_{\mathrm{mean}}(r)
\le
\varepsilon_{\mathrm{mean}},
\]
up to estimation error. Moreover, the \(q\)-quantile of the
exposure-weighted conditional mismatch satisfies
\[
Q_q\!\left(
\mathbb E_{m_r}
\left[
\mathrm{KL}\!\left(
\pi(\cdot\mid \tilde{\mathbf z}^{\,\mathrm{pre}}_t)
\Vert
\pi_s(\cdot\mid t)
\right)
\middle| t
\right]
\right)
=
\lambda_r \hat{k}_q(r)
\le
\varepsilon_q .
\]
Thus, the mean constraint controls average modality-induced policy drift,
while the quantile constraint suppresses exposure to rare high-mismatch
states.
\end{proposition}

\begin{proof}
For fixed timestep \(t\), the switched input equals
\(\tilde{\mathbf z}^{\,s}_t\) when \(m_r=0\) and
\(\tilde{\mathbf z}^{\,o}_t\) when \(m_r=1\). Therefore, the KL
divergence to the state-conditioned action distribution is \(0\) with
probability \(1-\lambda_r\) and \(K_t\) with probability \(\lambda_r\).
Hence,
\[
\mathbb E_{m_r}
\left[
\mathrm{KL}\!\left(
\pi(\cdot\mid \tilde{\mathbf z}^{\,\mathrm{pre}}_t)
\Vert
\pi_s(\cdot\mid t)
\right)
\middle| t
\right]
=
(1-\lambda_r)\cdot 0+\lambda_r K_t
=
\lambda_r K_t .
\]
Averaging this conditional expectation over timesteps in the rollout
batch gives
\[
\mathbb E_t[\lambda_r K_t]
=
\lambda_r \hat{k}_{\mathrm{mean}}(r),
\]
where \(\hat{k}_{\mathrm{mean}}(r)\) estimates the batch mean of \(K_t\).

Under the stated exposure condition,
\[
\lambda_r \hat{k}_{\mathrm{mean}}(r)
\le
\varepsilon_{\mathrm{mean}},
\]
up to estimation error and the use of \(K_{\min}\) for numerical
stability.

Similarly, because \(\lambda_r\) is fixed for the rollout, the
\(q\)-quantile over timesteps of the exposure-weighted conditional
mismatch is
\[
Q_q(\lambda_r K_t)
=
\lambda_r Q_q(K_t)
\approx
\lambda_r \hat{k}_q(r).
\]

The quantile-budget constraint then gives
\[
\lambda_r \hat{k}_q(r)
\le
\varepsilon_q .
\]
This proves the stated mean and tail exposure bounds.
\end{proof}

In implementation, the controller uses the lagged running statistics
\(\hat{k}_{\mathrm{mean}}(r-1)\) and \(\hat{k}_q(r-1)\) when determining
the exposure for rollout \(r\). Thus, the proposition gives the exact
conditional bound for a fixed rollout batch, while the implemented
controller realizes this bound approximately, subject to changes in the
running KL statistics between rollouts and EMA tracking error.

\subsection{Why representation alignment can reduce policy mismatch}
\label{app:kl-bound}

The switching rule itself uses the measured policy KL \(K_t\), so it does not require the following bound. The bound is included only to motivate why aligning the fused state and RGB-D representations can reduce the measured cross-modal policy mismatch.

For diagonal Gaussians \(\mathcal{N}(\mu_a,\mathrm{diag}(\sigma_a^2))\) and
\(\mathcal{N}(\mu_b,\mathrm{diag}(\sigma_b^2))\),
\begin{equation}
\mathrm{KL}(a\Vert b)
=
\sum_{j=1}^m
\left[
\log\frac{\sigma_{b,j}}{\sigma_{a,j}}
+
\frac{\sigma_{a,j}^2+(\mu_{a,j}-\mu_{b,j})^2}{2\sigma_{b,j}^2}
-\frac12
\right].
\end{equation}

In the shared-covariance case used here, this expression simplifies to a quadratic form in the mean difference.

\begin{lemma}[Local KL sensitivity to fused-input mismatch]
\label{lem:latent-to-kl}
Assume that the actor uses a Gaussian policy with shared covariance,
\[
\pi(\cdot \mid \tilde{\mathbf z}) =
\mathcal N(\mu(\tilde{\mathbf z}), \Sigma),
\]
where \(\Sigma \succ 0\) does not depend on the input. Let
\(\sigma_{\min}^2\) denote the smallest eigenvalue of \(\Sigma\). Assume
that \(\mu\) is Lipschitz along the line segment
\[
\left\{
\tilde{\mathbf z}^{\,s}_t
+
\lambda\bigl(\tilde{\mathbf z}^{\,o}_t-\tilde{\mathbf z}^{\,s}_t\bigr)
:\lambda\in[0,1]
\right\}
\]
with constant \(L_\mu\). Then
\[
\mathrm{KL}\!\left(\pi_o(\cdot\mid t)\Vert\pi_s(\cdot\mid t)\right)
\le
\frac{L_\mu^2}{2\sigma_{\min}^2}
\left\|
\tilde{\mathbf z}^{\,o}_t-\tilde{\mathbf z}^{\,s}_t
\right\|^2.
\]
\end{lemma}

\begin{proof}
Let
\[
\mu_o = \mu(\tilde{\mathbf z}^{\,o}_t),
\qquad
\mu_s = \mu(\tilde{\mathbf z}^{\,s}_t),
\]
so that
\[
\pi_o(\cdot\mid t)=\mathcal N(\mu_o,\Sigma),
\qquad
\pi_s(\cdot\mid t)=\mathcal N(\mu_s,\Sigma).
\]
For two \(m\)-dimensional Gaussians
\(\mathcal N(\mu_o,\Sigma)\) and \(\mathcal N(\mu_s,\Sigma)\) with the
same positive-definite covariance, the KL divergence is
\[
\mathrm{KL}\!\left(
\mathcal N(\mu_o,\Sigma)
\Vert
\mathcal N(\mu_s,\Sigma)
\right)
=
\frac{1}{2}
\left[
\mathrm{tr}(\Sigma^{-1}\Sigma)
+
(\mu_s-\mu_o)^\top \Sigma^{-1}(\mu_s-\mu_o)
-
m
+
\log\frac{\det\Sigma}{\det\Sigma}
\right].
\]
The trace term satisfies \(\mathrm{tr}(\Sigma^{-1}\Sigma)=m\), and the
log-determinant term is zero because the two covariances are identical.
Therefore the KL reduces to
\[
\mathrm{KL}\!\left(\pi_o\Vert\pi_s\right)
=
\frac{1}{2}
(\mu_o-\mu_s)^\top
\Sigma^{-1}
(\mu_o-\mu_s).
\]

Since \(\Sigma\succ 0\) and its smallest eigenvalue is
\(\sigma_{\min}^2\), every eigenvalue of \(\Sigma^{-1}\) is at most
\(\sigma_{\min}^{-2}\). Equivalently,
\[
\Sigma^{-1} \preceq \sigma_{\min}^{-2} I .
\]
Thus, for any vector \(v\),
\[
v^\top \Sigma^{-1} v
\le
\sigma_{\min}^{-2}\|v\|^2.
\]
Applying this inequality with \(v=\mu_o-\mu_s\) gives
\[
\mathrm{KL}\!\left(\pi_o\Vert\pi_s\right)
\le
\frac{1}{2\sigma_{\min}^2}
\|\mu_o-\mu_s\|^2.
\]

It remains to relate the difference in policy means to the difference
between the fused actor inputs. By the assumed Lipschitz continuity of
\(\mu\) along the segment connecting
\(\tilde{\mathbf z}^{\,s}_t\) and \(\tilde{\mathbf z}^{\,o}_t\),
\[
\|\mu_o-\mu_s\|
=
\left\|
\mu(\tilde{\mathbf z}^{\,o}_t)
-
\mu(\tilde{\mathbf z}^{\,s}_t)
\right\|
\le
L_\mu
\left\|
\tilde{\mathbf z}^{\,o}_t
-
\tilde{\mathbf z}^{\,s}_t
\right\|.
\]
Substituting this bound into the KL inequality yields
\[
\mathrm{KL}\!\left(\pi_o(\cdot\mid t)\Vert\pi_s(\cdot\mid t)\right)
\le
\frac{L_\mu^2}{2\sigma_{\min}^2}
\left\|
\tilde{\mathbf z}^{\,o}_t
-
\tilde{\mathbf z}^{\,s}_t
\right\|^2,
\]
which proves the claim.
\end{proof}

\section{Detailed representation alignment objective}
\label{app:repr-details}

Let
\[
\mathbf{z}_{t,s}=E_s(\mathbf{s}_t^{\mathrm{priv}}),\qquad
\mathbf{z}_{t,o}=E_o(o_t)
\]
denote the raw latent representations produced by the state and RGB-D encoders. The representation alignment objective updates the RGB-D encoder and, when enabled, the cross-attention aligner and state decoder. Its purpose is to make matched privileged-state and RGB-D observations produce compatible latent representations for control.

\paragraph{Optional aligner and aligned targets.}
When enabled, a lightweight cross-attention aligner $f_\varphi$ produces
\[
(\hat{\mathbf{z}}_{t,s},\hat{\mathbf{z}}_{t,o}) = f_\varphi(\mathbf{z}_{t,s},\mathbf{z}_{t,o}).
\]
These aligned latents are used as aligned targets for representation learning. In practice, we post-process [see Appendix \ref{app:utils}] both raw and aligned latents before applying several of the losses:
\[
z_s'=\texttt{post\_process}(\mathbf{z}_{t,s}),\qquad
z_o'=\texttt{post\_process}(\mathbf{z}_{t,o}),
\]
\[
z_s^T=\texttt{post\_process}(\hat{\mathbf{z}}_{t,s}),\qquad
z_o^T=\texttt{post\_process}(\hat{\mathbf{z}}_{t,o}).
\]

\paragraph{Loss decomposition.}
The full representation objective is
\begin{equation}
\label{eq:repr-full}
\mathcal{L}_{\mathrm{repr}}
=
\mathcal{L}_{\mathrm{align}}
+\mathcal{L}_{\mathrm{raw}}
+\lambda_{\pi,\mathrm{repr}}\mathcal{L}_{\pi,\mathrm{repr}}
+\mathcal{L}_{\mathrm{dec}}.
\end{equation}
The four groups are defined below.

\paragraph{Aligned-target regularization.}
This group uses the aligner outputs as a supervisory signal. It aligns the post-alignment latents and regularizes the raw latents toward that aligned structure. 
\begin{equation}
\mathcal{L}_{\mathrm{align}}
=
\mathcal{L}_{\mathrm{sim}}
+\mathcal{L}_{\mathrm{reg}}
+\mathcal{L}_{\mathrm{mse}}
+\lambda_{\mathrm{byol}}\mathcal{L}_{\mathrm{byol}}
+\lambda_{\mathrm{distill}}\mathcal{L}_{\mathrm{distill}}.
\end{equation}

Here, $\mathcal{L}_{\mathrm{sim}}$ is a similarity-based alignment loss on the aligned latents, $\mathcal{L}_{\mathrm{reg}}$ is an alignment regularizer, and $\mathcal{L}_{\mathrm{mse}}$ is a direct feature-matching term. We also use a BYOL-style term on the aligner outputs,
\[
\mathcal{L}_{\mathrm{byol}}=\mathrm{BYOL}\!\left(\hat{\mathbf{z}}_{t,s},\hat{\mathbf{z}}_{t,o}\right),
\]
and a cosine target-matching term that pushes the raw latents toward the aligned targets:
\[
\mathcal{L}_{\mathrm{distill}}
=
\tfrac12\,\mathbb{E}\!\left[1-\langle z_s',z_s^T\rangle\right]
+
\tfrac12\,\mathbb{E}\!\left[1-\langle z_o',z_o^T\rangle\right].
\]

\paragraph{Raw-latent regularization.}
This group acts directly on the encoder outputs before control. It keeps the raw latent space comparable across modalities and discourages collapse:
\begin{equation}
\mathcal{L}_{\mathrm{raw}}
=
\lambda_{\mathrm{raw}}\mathcal{L}_{\mathrm{raw\_sim}}
+\lambda_{\mathrm{NCE}}\mathcal{L}_{\mathrm{NCE}}
+\lambda_{\mathrm{VICReg}}\mathcal{L}_{\mathrm{VICReg}}.
\end{equation}

$\mathcal{L}_{\mathrm{raw\_sim}}$ is a raw-latent similarity term. For InfoNCE, we use matched state-image pairs from the same timestep:
\[
\mathcal{L}_{\mathrm{NCE}}
=
-\frac{1}{B}\sum_{b=1}^{B}
\log
\frac{\exp(\mathrm{sim}(\mathbf{z}_{s}^{(b)},\mathbf{z}_{o}^{(b)})/\tau)}
{\sum_{k=1}^{B}\exp(\mathrm{sim}(\mathbf{z}_{s}^{(b)},\mathbf{z}_{o}^{(k)})/\tau)}.
\]
We also apply VICReg-style regularization on the post-processed raw latents:
\[
\mathcal{L}_{\mathrm{VICReg}}
=
\|z_s'-z_o'\|_2^2
+\tfrac12\!\left(\mathcal{L}_{\mathrm{var}}(z_s')+\mathcal{L}_{\mathrm{var}}(z_o')\right)
+\lambda_{\mathrm{cov}}\!\left(\mathcal{L}_{\mathrm{cov}}(z_s')+\mathcal{L}_{\mathrm{cov}}(z_o')\right).
\]

\paragraph{Policy-conditioned alignment.}
To tie representation learning to control, we add a symmetric KL term between the state-conditioned and image-conditioned policies:
\begin{equation}
\label{eq:pi-distill-app}
\mathcal{L}_{\pi,\mathrm{repr}}
=
\tfrac12\,\mathrm{KL}\!\left(\pi(\cdot\mid \tilde{\mathbf{z}}^{\,s})\Vert\pi(\cdot\mid \tilde{\mathbf{z}}^{\,o})\right)
+
\tfrac12\,\mathrm{KL}\!\left(\pi(\cdot\mid \tilde{\mathbf{z}}^{\,o})\Vert\pi(\cdot\mid \tilde{\mathbf{z}}^{\,s})\right).
\end{equation}
During this step, the actor and fusion module are frozen. Gradients from $\mathcal{L}_{\pi,\mathrm{repr}}$ therefore update the RGB-D encoder rather than the control policy. This makes the RGB-D encoder learn representations that are not only close to state representations, but also useful for action selection.

\paragraph{Optional decoding.}
The optional decoder reconstructs privileged-state targets from both modalities:
\[
\hat{\mathbf{s}}_{t,s}=D_s(\mathbf{z}_{t,s}),\qquad
\hat{\mathbf{s}}_{t,o}=D_s(\mathbf{z}_{t,o}).
\]
This yields
\begin{equation}
\mathcal{L}_{\mathrm{dec}}
=
\mathcal{L}_{\mathrm{state\_dec}}
+\lambda_{\mathrm{conf}}\mathcal{L}_{\mathrm{conf}},
\end{equation}
where
\[
\mathcal{L}_{\mathrm{state\_dec}}
=
\tfrac12\|\hat{\mathbf{s}}_{t,s}-\mathbf{s}_t^{\mathrm{priv}}\|_2^2
+
\tfrac12\|\hat{\mathbf{s}}_{t,o}-\mathbf{s}_t^{\mathrm{priv}}\|_2^2,
\]
and
\[
\mathcal{L}_{\mathrm{conf}}
=
\|\hat{\mathbf{s}}_{t,s}-\hat{\mathbf{s}}_{t,o}\|_2^2.
\]
This term preserves control-relevant state information in both modalities. It can also serve as a simple interpretability probe of what the latent space retains.

\paragraph{Implementation weights.}
In the implementation used here, the full loss is
\begin{align}
\mathcal{L}_{\mathrm{repr}}
&=
\mathcal{L}_{\mathrm{sim}}
+\mathcal{L}_{\mathrm{reg}}
+\mathcal{L}_{\mathrm{mse}}
+0.02\,\mathcal{L}_{\mathrm{byol}}
+0.02\,\mathcal{L}_{\mathrm{distill}} \nonumber\\
&\quad
+0.01\,\mathcal{L}_{\mathrm{raw\_sim}}
+0.05\,\mathcal{L}_{\mathrm{NCE}}
+0.3\,\mathcal{L}_{\mathrm{VICReg}}
+0.01\,\mathcal{L}_{\pi,\mathrm{repr}} \nonumber\\
&\quad
+\mathcal{L}_{\mathrm{state\_dec}}
+0.2\,\mathcal{L}_{\mathrm{conf}}.
\end{align}
These weights match the current implementation. If different weights are used in a variant or ablation, we report them in the corresponding experimental settings table.

\section{Latent post-processing}
\label{app:utils}

For a minibatch \(Z\in\mathbb{R}^{M\times d}\) (rows are latent vectors), we apply a
two-stage normalization consisting of global standardization followed by
row-wise \(\ell_2\) normalization:
\begin{equation}
\label{eq:postprocess}
\begin{aligned}
\widehat{Z} \;&=\; \frac{Z - \operatorname{mean}(Z)}{\operatorname{std}(Z) + \delta},\\
\texttt{post\_process}(Z)_i \;&=\; 
\frac{\widehat Z_{i\cdot}}{\|\widehat Z_{i\cdot}\|_2 + \varepsilon},
\qquad i=1,\dots,M,
\end{aligned}
\end{equation}
where \(\operatorname{mean}(Z)\) and \(\operatorname{std}(Z)\) are the scalar (global)
mean and standard deviation over all \(Md\) entries of \(Z\), and
\(\delta,\varepsilon>0\) are small stability constants (we use
\(\delta=10^{-12}\), \(\varepsilon=10^{-6}\)).
This operation is applied whenever alignment or contrastive losses require
scale-invariant latent comparisons.

\clearpage

\section{Train Step@50 Results}
\label{app:steps50}

\begin{table}[H]
\caption{\textbf{Train Step@50\% success occupancy} reports the first training point at which success occupancy reaches 50\%. Lower is better. Parentheses indicate the number of training seeds that reached the threshold, e.g., 693~(1/2). ``Not reached'' means no seed reached the threshold. For teacher-student, ``final student'' reports final fine-tuning, while ``3-stage, e2e'' counts the full pipeline.}
\label{tab:steps-at-50}
\centering
\scriptsize
\setlength{\tabcolsep}{2.5pt}
\renewcommand{\arraystretch}{1.08}
\begin{tabularx}{\linewidth}{@{}l *{5}{>{\centering\arraybackslash}X}@{}}
\toprule
\textbf{Method} & \textbf{Reach} & \textbf{Lift} & \textbf{Push} & \textbf{Open\,-\,Drawer} & \textbf{Pick\,-\,and\,-\,Place} \\
\midrule
states-only & 18.5 $\pm$ 6.4~(2/2) & 693~(1/2) & 556.5 $\pm$ 27.6~(2/2) & 31.5 $\pm$ 3.5~(2/2) & 522.0 $\pm$ 46.7~(2/2) \\
RGB-D-only & 15.5 $\pm$ 9.2~(2/2) & Not reached~(0/2) & 1447.5 $\pm$ 20.5~(2/2) & 51.5 $\pm$ 4.9~(2/2) & Not reached~(0/2) \\
RGB-D-asymmetric & 12.5 $\pm$ 5.0~(2/2) & 1098~(1/2) & Not reached~(0/2) & 50.0 $\pm$ 4.2~(2/2) & 1611~(1/2) \\
teacher-student (final student) & 1.0 $\pm$ 1.4 (2/2) & 262 (1/2) & Not reached (0/2) & 5.0 $\pm$ 5.7 (2/2) & Not reached (0/2) \\
teacher-student (3-stage, e2e) & 2001.0 $\pm$ 1.4 (2/2) & 2262 (1/2) & Not reached (0/2) & 2005.0 $\pm$ 5.7 (2/2) & Not reached (0/2) \\
\midrule
\rowcolor{oursblue}
\textbf{\method} (ours) & 11.5 $\pm$ 3.5~(2/2) & 461.5 $\pm$ 112.4~(2/2) & 461.5 $\pm$ 19.1~(2/2) & 91.5 $\pm$ 20.5~(2/2) & 470.0 $\pm$ 107.5~(2/2) \\
\bottomrule
\end{tabularx}
\end{table}

\section{Additional Training Curves}
\label{app:training-curves}
\begin{figure}[H]
    \centering

    \begin{subfigure}[t]{0.49\linewidth}
        \centering
        \includegraphics[width=\linewidth]{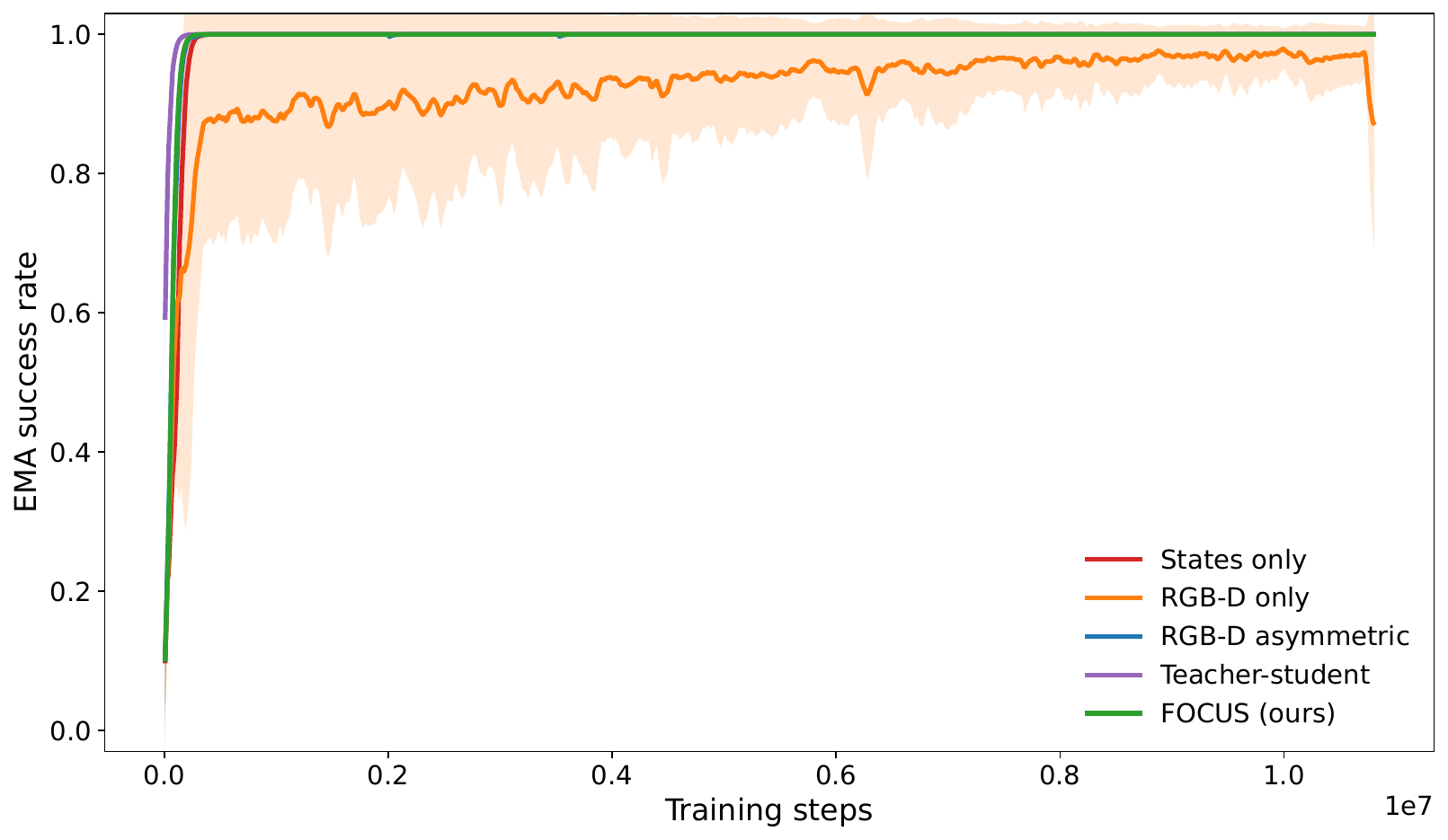}
        \caption{\textbf{Reach-Cube}}
    \end{subfigure}\hfill
    \begin{subfigure}[t]{0.49\linewidth}
        \centering
        \includegraphics[width=\linewidth]{images/sample_efficiency/train_lift_success_envsteps.pdf}
        \caption{\textbf{Lift-Cube}}
    \end{subfigure}

    \smallskip

    \begin{subfigure}[t]{0.49\linewidth}
        \centering
        \includegraphics[width=\linewidth]{images/sample_efficiency/train_push_success_envsteps.pdf}
        \caption{\textbf{Push-Cube}}
    \end{subfigure}\hfill
    \begin{subfigure}[t]{0.49\linewidth}
        \centering
        \includegraphics[width=\linewidth]{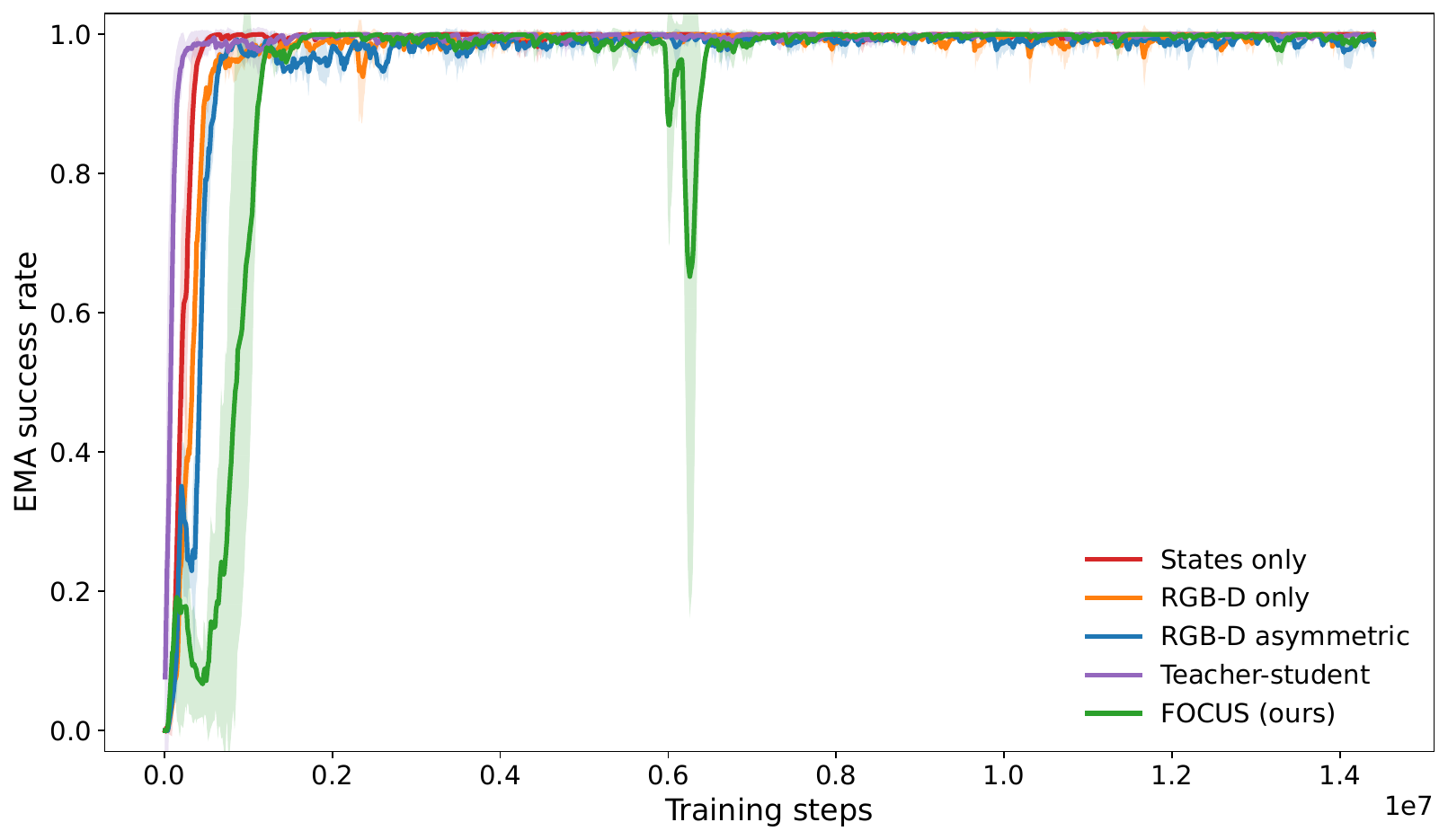}
        \caption{\textbf{Open-Drawer}}
    \end{subfigure}

    \smallskip

    \begin{subfigure}[t]{0.49\linewidth}
        \centering
        \includegraphics[width=\linewidth]{images/sample_efficiency/train_pick_and_place_success_envsteps.pdf}
        \caption{\textbf{Pick-and-Place}}
    \end{subfigure}

    \caption{\textbf{Full training success EMA curves for all tasks.}
    Curves show training-time success. For teacher-student, curves show the final student PPO phase; Table~\ref{tab:success-occupancy-auc} reports end-to-end normalized AUC in the main paper.}
    \label{fig:train-success-vs-steps-full}
\end{figure}

\section{Per episode success rate curves}
This appendix reports per-episode training success curves for all five tasks. Unlike the EMA curves in the main text, these plots show the raw training success signal and therefore expose short-term variance, instability, and plateau behavior. Earlier rises and higher final plateaus indicate better sample efficiency and more reliable training.

\begin{figure}[H]
    \centering

    \begin{subfigure}[t]{0.49\columnwidth}\vspace{0pt}
        \centering
        \includegraphics[width=\columnwidth]{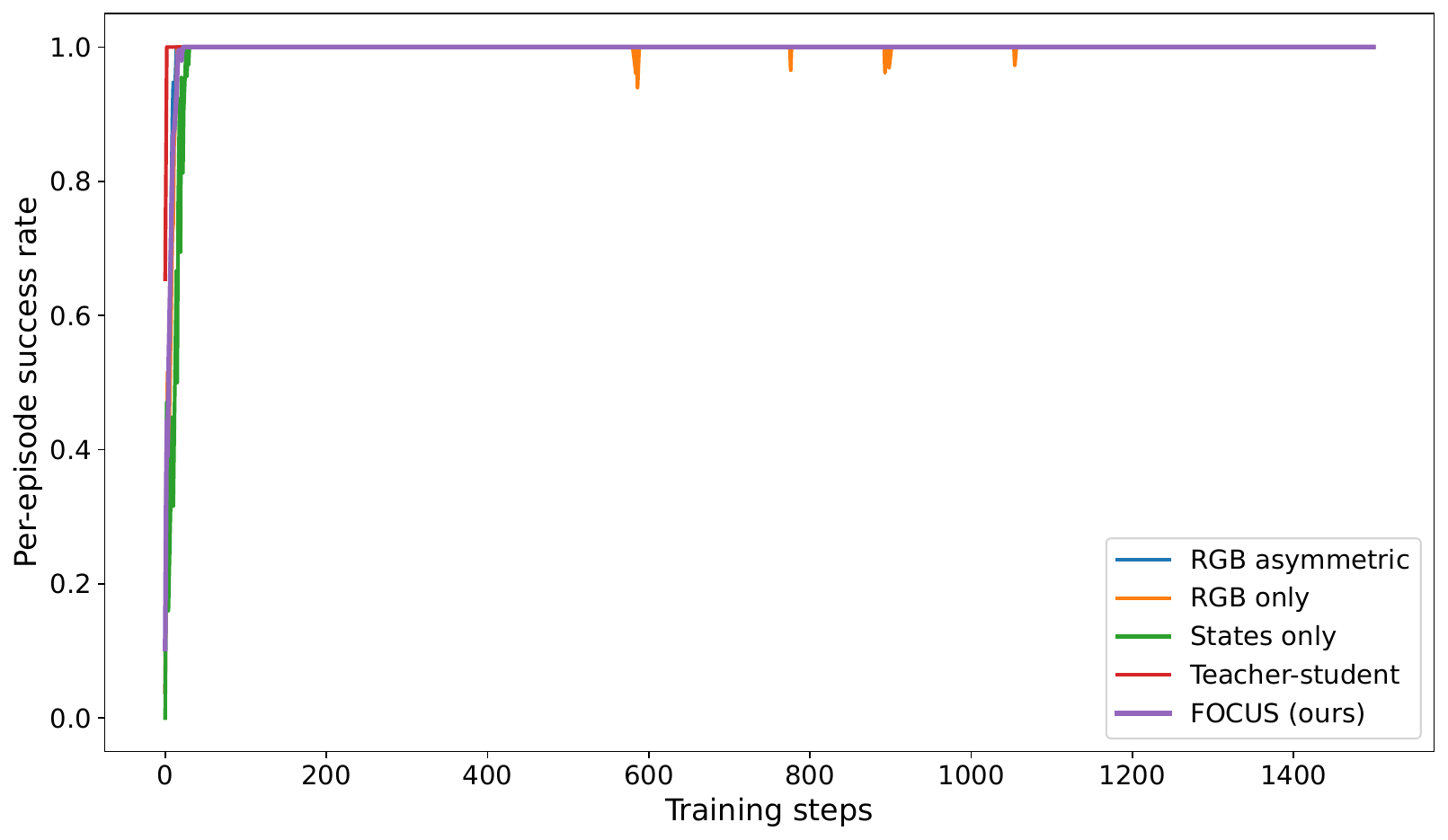}
        \caption{\textbf{Reach-Cube.}}
        \label{fig:per-episode-reach-success}
    \end{subfigure}\hfill
    \begin{subfigure}[t]{0.49\columnwidth}\vspace{0pt}
        \centering
        \includegraphics[width=\columnwidth]{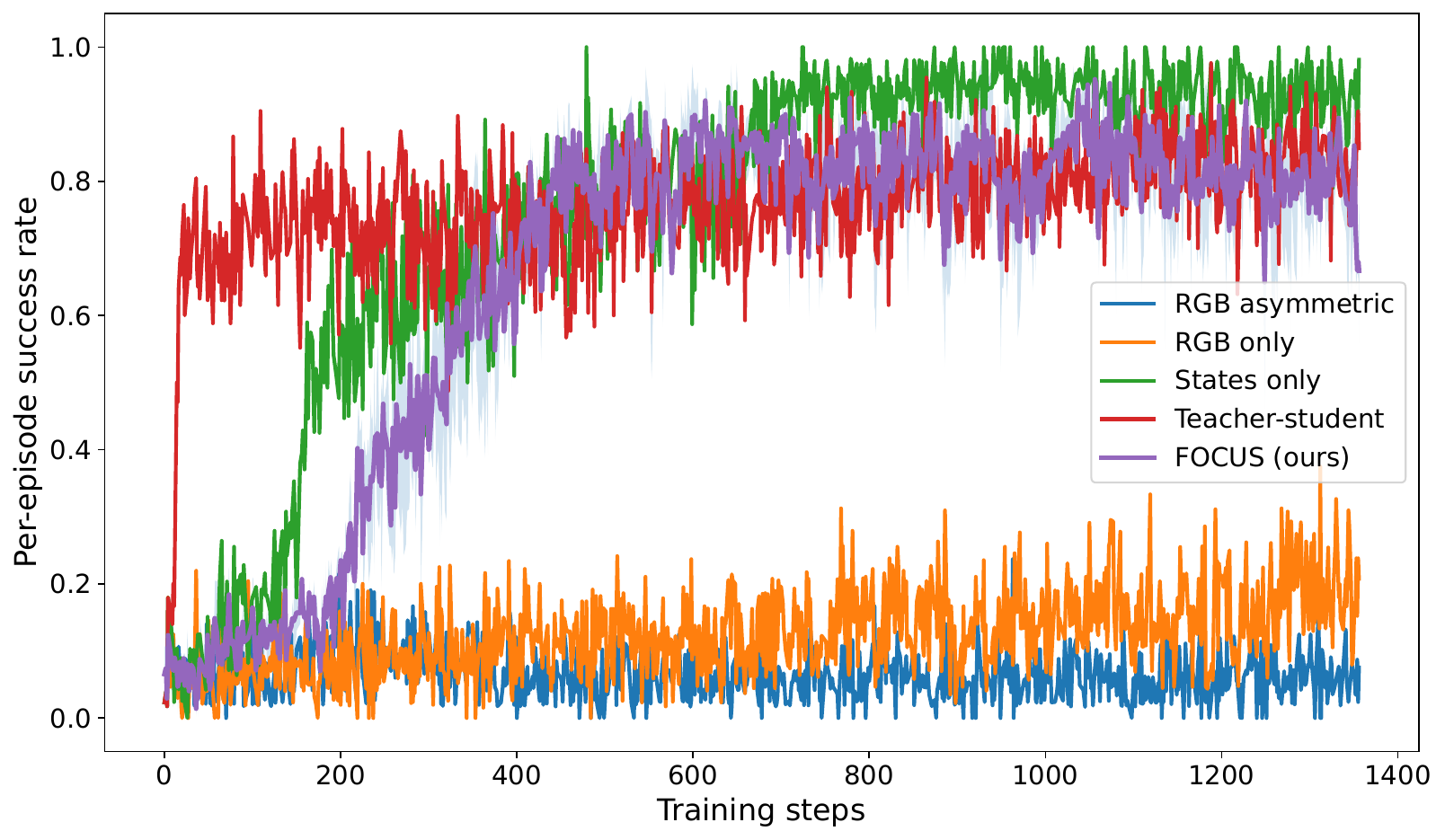}
        \caption{\textbf{Lift-Cube.}}
        \label{fig:per-episode-lift-success}
    \end{subfigure}

    \smallskip

    \begin{subfigure}[t]{0.49\columnwidth}\vspace{0pt}
        \centering
        \includegraphics[width=\columnwidth]{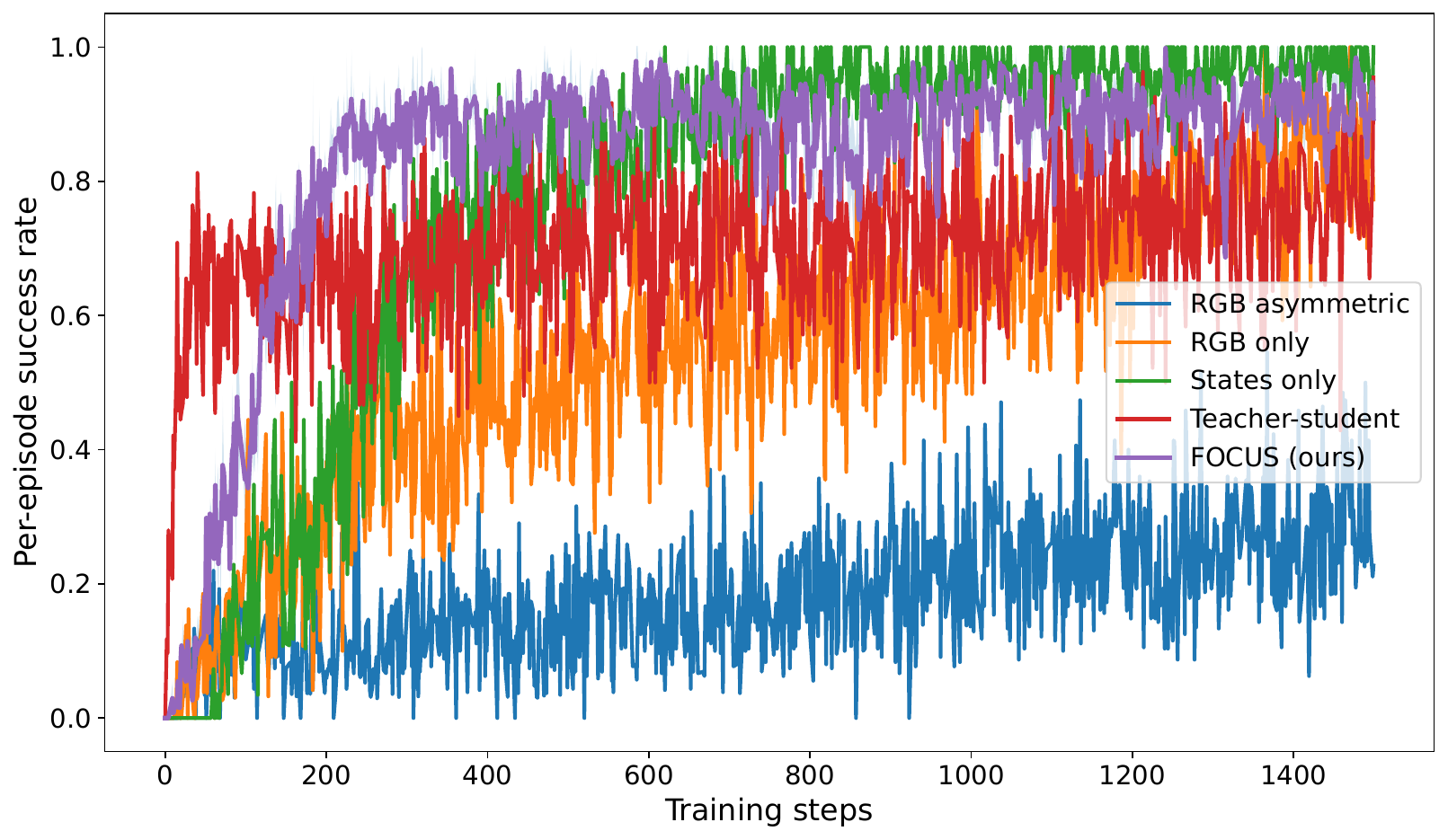}
        \caption{\textbf{Push-Cube.}}
        \label{fig:per-episode-push-success}
    \end{subfigure}\hfill
    \begin{subfigure}[t]{0.49\columnwidth}\vspace{0pt}
        \centering
        \includegraphics[width=\columnwidth]{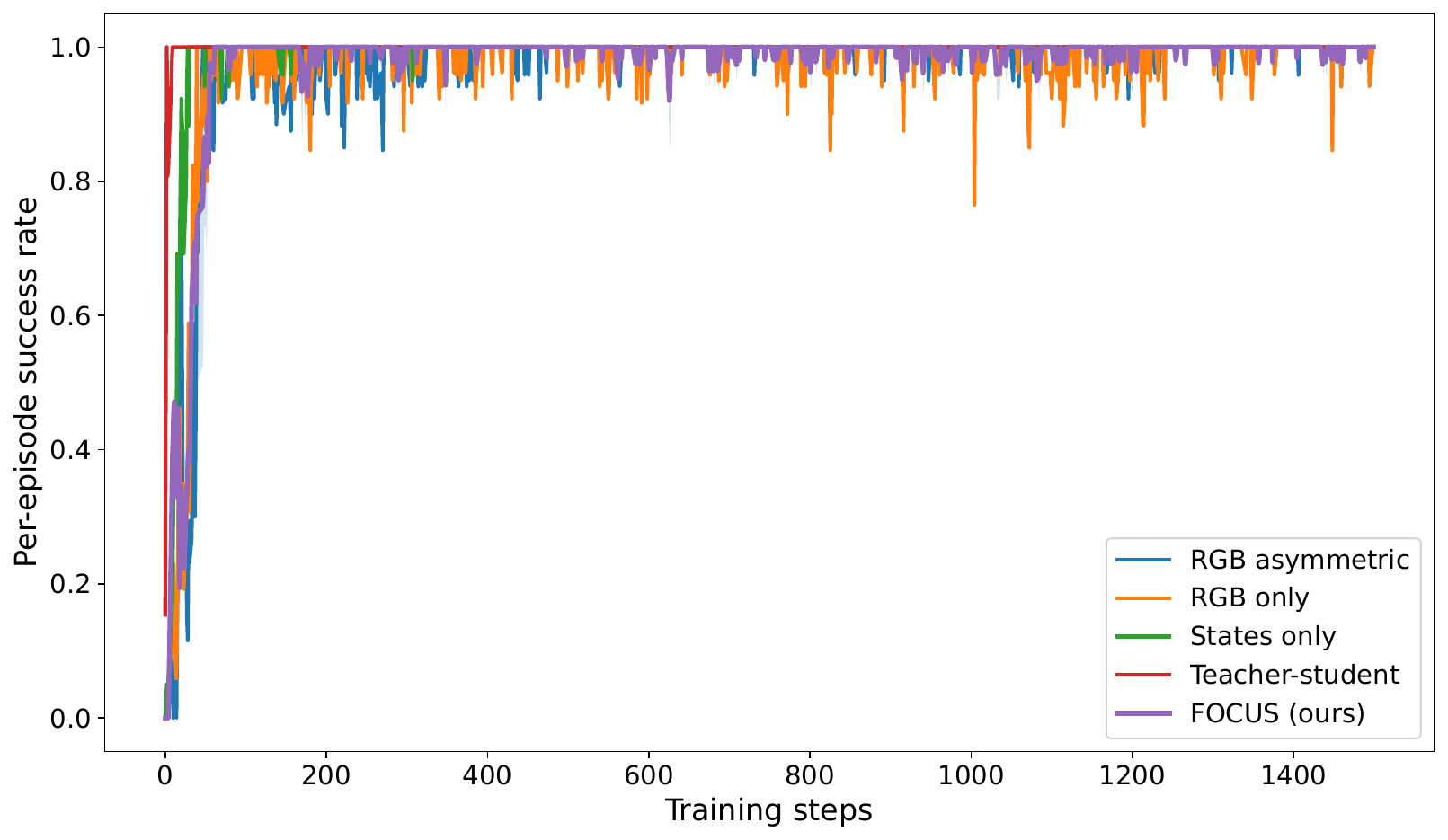}
        \caption{\textbf{Open-Drawer.}}
        \label{fig:per-episode-open-success}
    \end{subfigure}

    \smallskip

    \begin{subfigure}[t]{0.49\columnwidth}\vspace{0pt}
        \centering
        \includegraphics[width=\columnwidth]{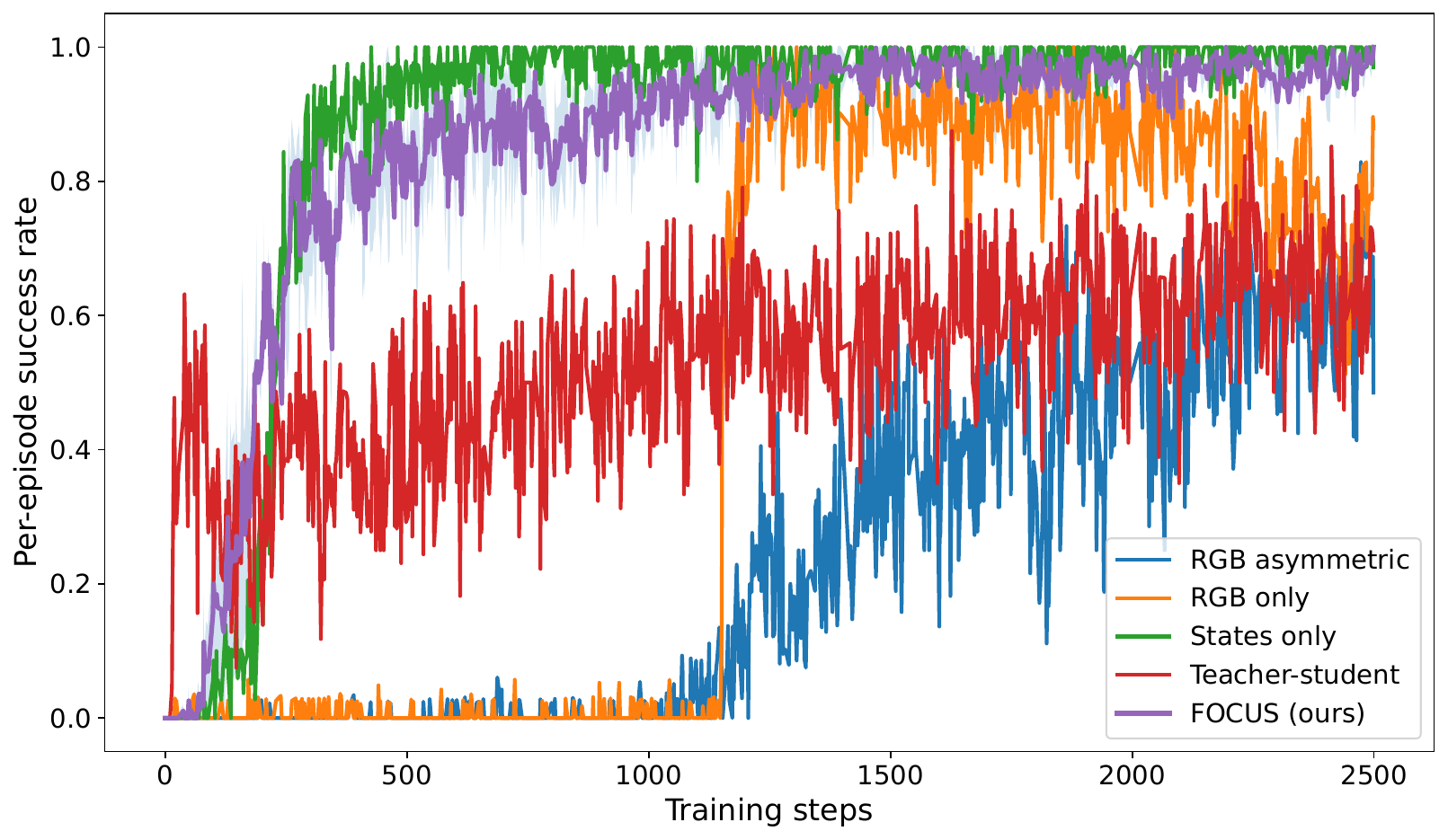}
        \caption{\textbf{Pick-and-Place.}}
        \label{fig:per-episode-pick-place-success}
    \end{subfigure}\hfill
    \begin{subfigure}[t]{0.49\columnwidth}\vspace{0pt}
        \centering
    \end{subfigure}

    \caption{\textbf{Per-episode training success rate} for all tasks. Earlier rises and higher plateaus indicate better sample efficiency.}
    \label{fig:per-episode-success-vs-steps}
\end{figure}

\clearpage

\section{Visualizations}
\label{app:visualizations}

This appendix collects qualitative probes of the learned visual and latent representations. These plots are intended as supporting diagnostics only. They are not part of the main quantitative evidence.

\subsection{Activation probe}
\label{app:activation-probe}

We first report an activation probe in
Figure~\ref{fig:activation-visualizations}, where we visualize intermediate
image-encoder activations as a lightweight probe of the learned visual
representation. We use forward activations rather than Grad-CAM, since they
are inexpensive to log and more compatible with our two-optimizer setup.

Qualitatively, the depth branch tends to emphasize geometric structure around the gripper, object, and drawer regions, while the RGB branch responds more broadly to appearance, contours, and scene boundaries. The activation patterns before and after image exposure remain broadly similar. We do not interpret these maps as a formal explanation of policy behavior. Instead, we use them as a simple check that the learned visual representation remains focused on task-relevant parts of the scene.

\begin{figure}[t]
\centering
\scriptsize

\begin{subfigure}[t]{0.49\columnwidth}
\centering
\setlength{\tabcolsep}{1pt}
\begin{tabular}{@{}c c c@{}}
 & \textbf{before} & \textbf{after} \\
\rotatebox{90}{\textbf{Lift}} &
\includegraphics[width=0.43\linewidth]{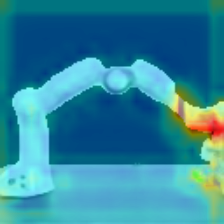} &
\includegraphics[width=0.43\linewidth]{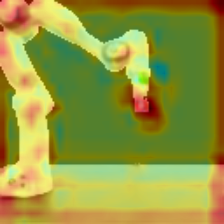} \\
\rotatebox{90}{\textbf{Push}} &
\includegraphics[width=0.43\linewidth]{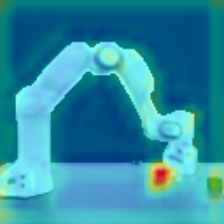} &
\includegraphics[width=0.43\linewidth]{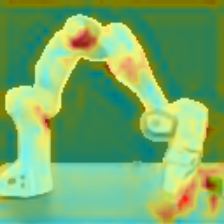} \\
\rotatebox{90}{\textbf{Open}} &
\includegraphics[width=0.43\linewidth]{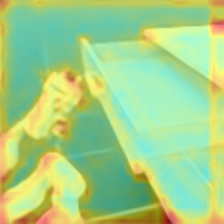} &
\includegraphics[width=0.43\linewidth]{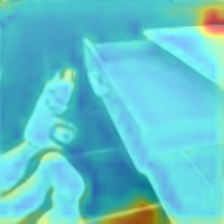} \\
\end{tabular}
\caption{RGB encoder branch}
\end{subfigure}\hfill
\begin{subfigure}[t]{0.49\columnwidth}
\centering
\setlength{\tabcolsep}{1pt}
\begin{tabular}{@{}c c c@{}}
 & \textbf{before} & \textbf{after} \\
\rotatebox{90}{\textbf{Lift}} &
\includegraphics[width=0.43\linewidth]{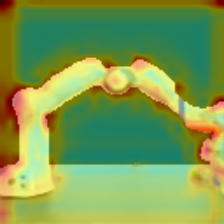} &
\includegraphics[width=0.43\linewidth]{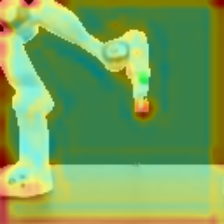} \\
\rotatebox{90}{\textbf{Push}} &
\includegraphics[width=0.43\linewidth]{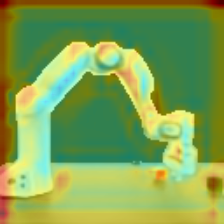} &
\includegraphics[width=0.43\linewidth]{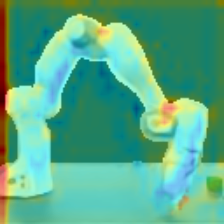} \\
\rotatebox{90}{\textbf{Open}} &
\includegraphics[width=0.43\linewidth]{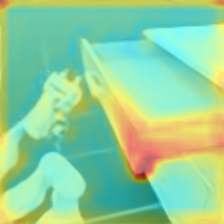} &
\includegraphics[width=0.43\linewidth]{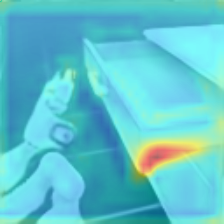} \\
\end{tabular}
\caption{Depth encoder branch}
\end{subfigure}

\caption{\textbf{Encoder activation visualizations} (channel-averaged intermediate feature magnitudes) before and after image exposure for Lift, Push, and Open-Drawer.}
\label{fig:activation-visualizations}
\end{figure}

\clearpage

\subsection{Additional activation visualizations}
\label{app:activation-probe-cheaper}

Figure~\ref{fig:activation-visualizations-cheaper-encoder} provides additional activation maps from the alternative visualization pipeline. As above, these plots are qualitative diagnostics of the learned visual representation. They are useful for checking whether the encoder continues to attend to task-relevant regions after the transition to RGB-D control.

\begin{figure}[H]
\centering
\scriptsize

\begin{subfigure}[t]{0.49\columnwidth}
\centering
\setlength{\tabcolsep}{1pt}
\begin{tabular}{@{}c c c@{}}
 & \textbf{before} & \textbf{after} \\
\rotatebox{90}{\textbf{Lift}} &
\includegraphics[width=0.43\linewidth]{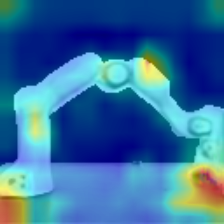} &
\includegraphics[width=0.43\linewidth]{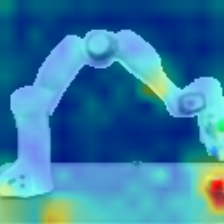} \\
\rotatebox{90}{\textbf{Pick-and-Place}} &
\includegraphics[width=0.43\linewidth]{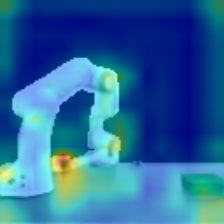} &
\includegraphics[width=0.43\linewidth]{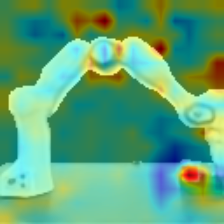} \\
\end{tabular}
\end{subfigure}\hfill
\begin{subfigure}[t]{0.49\columnwidth}
\centering
\setlength{\tabcolsep}{1pt}
\begin{tabular}{@{}c c c@{}}
 & \textbf{before} & \textbf{after} \\
\rotatebox{90}{\textbf{Push}} &
\includegraphics[width=0.43\linewidth]{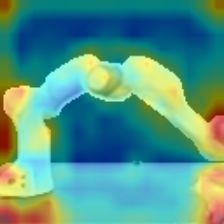} &
\includegraphics[width=0.43\linewidth]{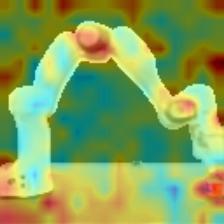} \\
\rotatebox{90}{\textbf{Open}} &
\includegraphics[width=0.43\linewidth]{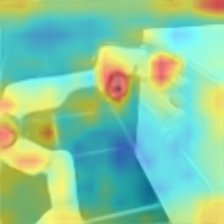} &
\includegraphics[width=0.43\linewidth]{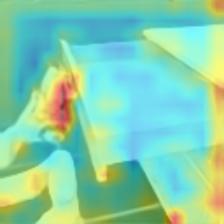} \\
\end{tabular}
\end{subfigure}

\caption{\textbf{RGB-D encoder activation probe} (channel-averaged intermediate feature magnitudes) before and after image exposure for Lift, Push, Pick-and-Place, and Open-Drawer. These plots are qualitative diagnostics of the learned visual representation.}
\label{fig:activation-visualizations-cheaper-encoder}
\end{figure}

\clearpage

\subsection{t-SNE latent visualizations}
\label{app:tsne-visualizations}

Figure~\ref{fig:tsne-tasks} shows t-SNE projections of the state and image latent representations before and after alignment for Lift, Push, Pick-and-Place, and Open-Drawer. These plots are also qualitative. Their purpose is to give a rough visual summary of how the two modalities relate in the latent space over training.

Across the harder tasks, the later projections often show tighter cross-modal grouping than the earlier ones, which is consistent with improved latent compatibility between state and RGB-D. At the same time, these plots should be interpreted cautiously, since t-SNE depends on the projection procedure and does not provide a quantitative measure of alignment. 
We therefore treat these plots as supporting diagnostics rather than primary
evidence. Our main quantitative evidence comes from the task-performance and
ablation results in the main paper.

\begin{figure}[H]
\centering
\scriptsize

\begin{subfigure}[t]{0.6\columnwidth}
\centering
\setlength{\tabcolsep}{1pt}
\begin{tabular}{@{}c c c@{}}
 & \textbf{before} & \textbf{after} \\
\rotatebox{90}{\textbf{Lift}} &
\includegraphics[width=0.6\linewidth]{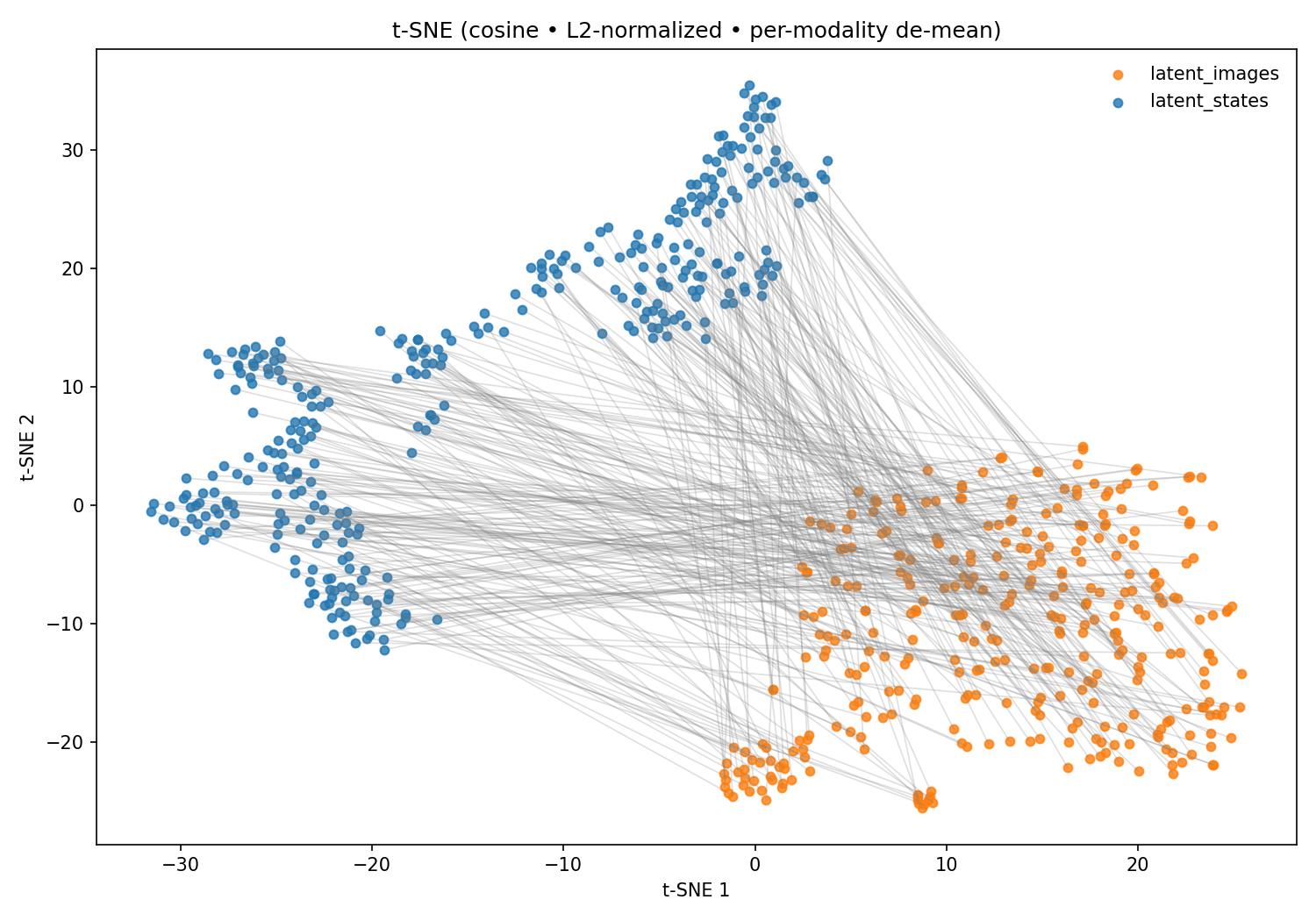} &
\includegraphics[width=0.6\linewidth]{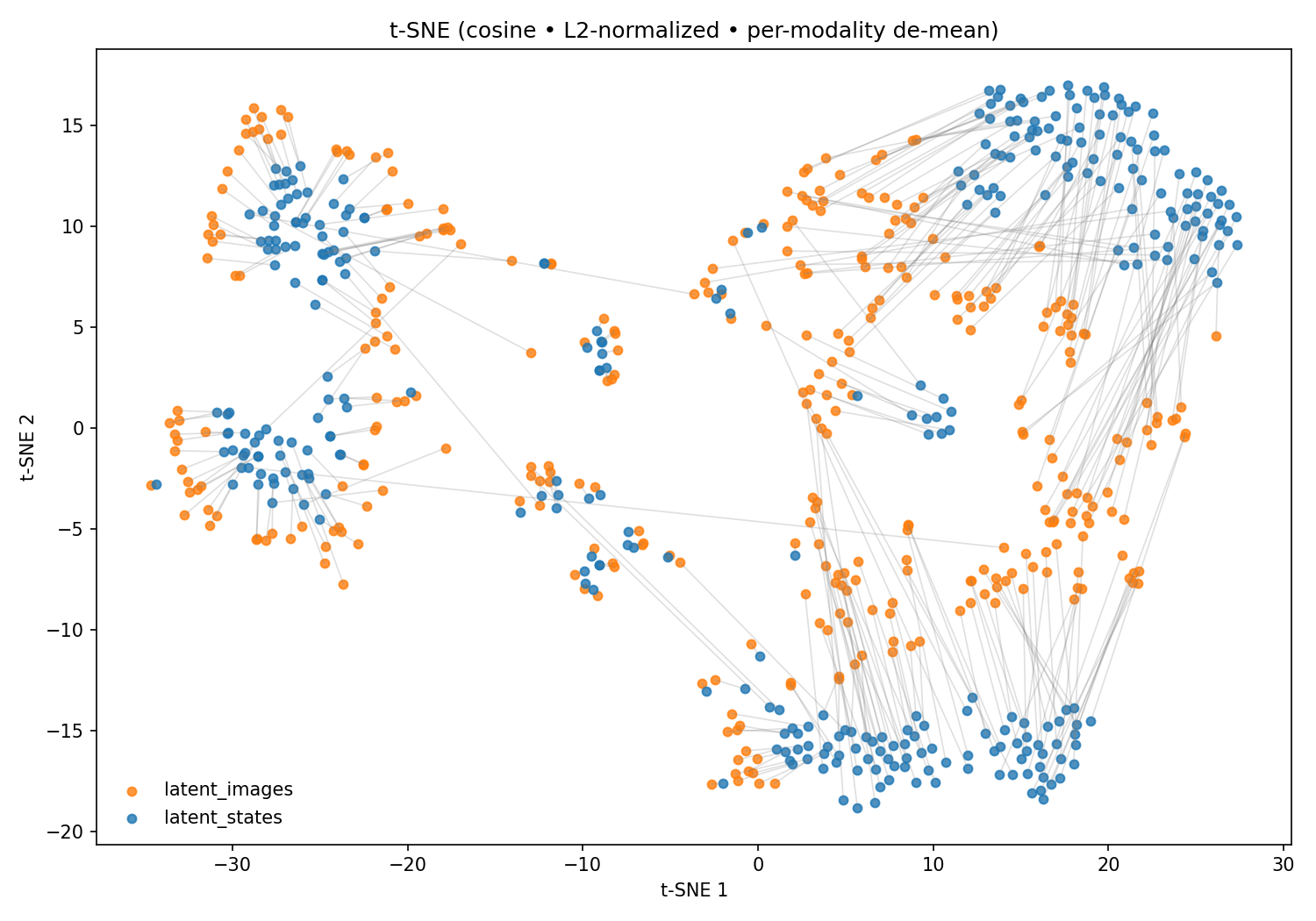} \\
\rotatebox{90}{\textbf{Pick-and-Place}} &
\includegraphics[width=0.6\linewidth]{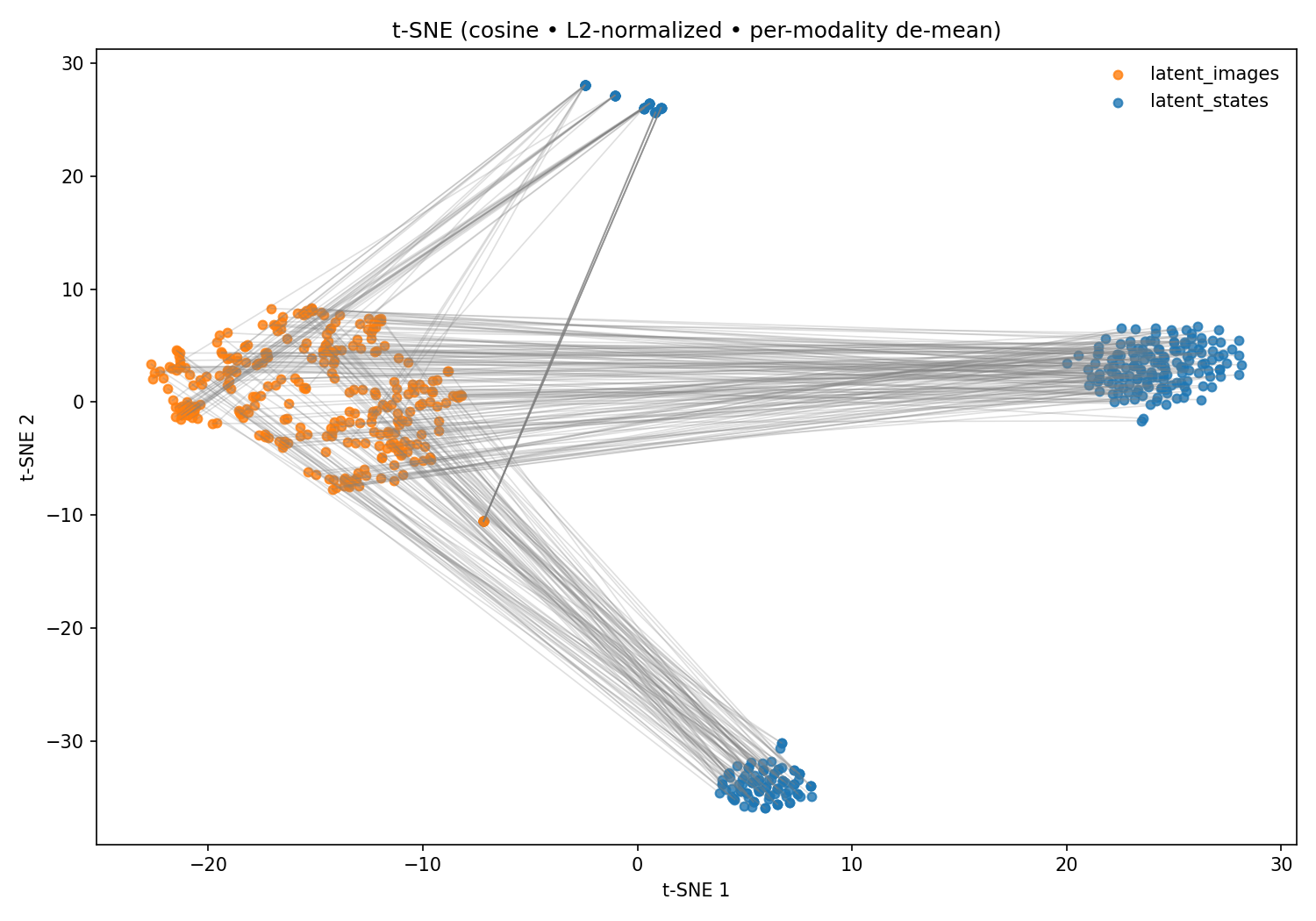} &
\includegraphics[width=0.6\linewidth]{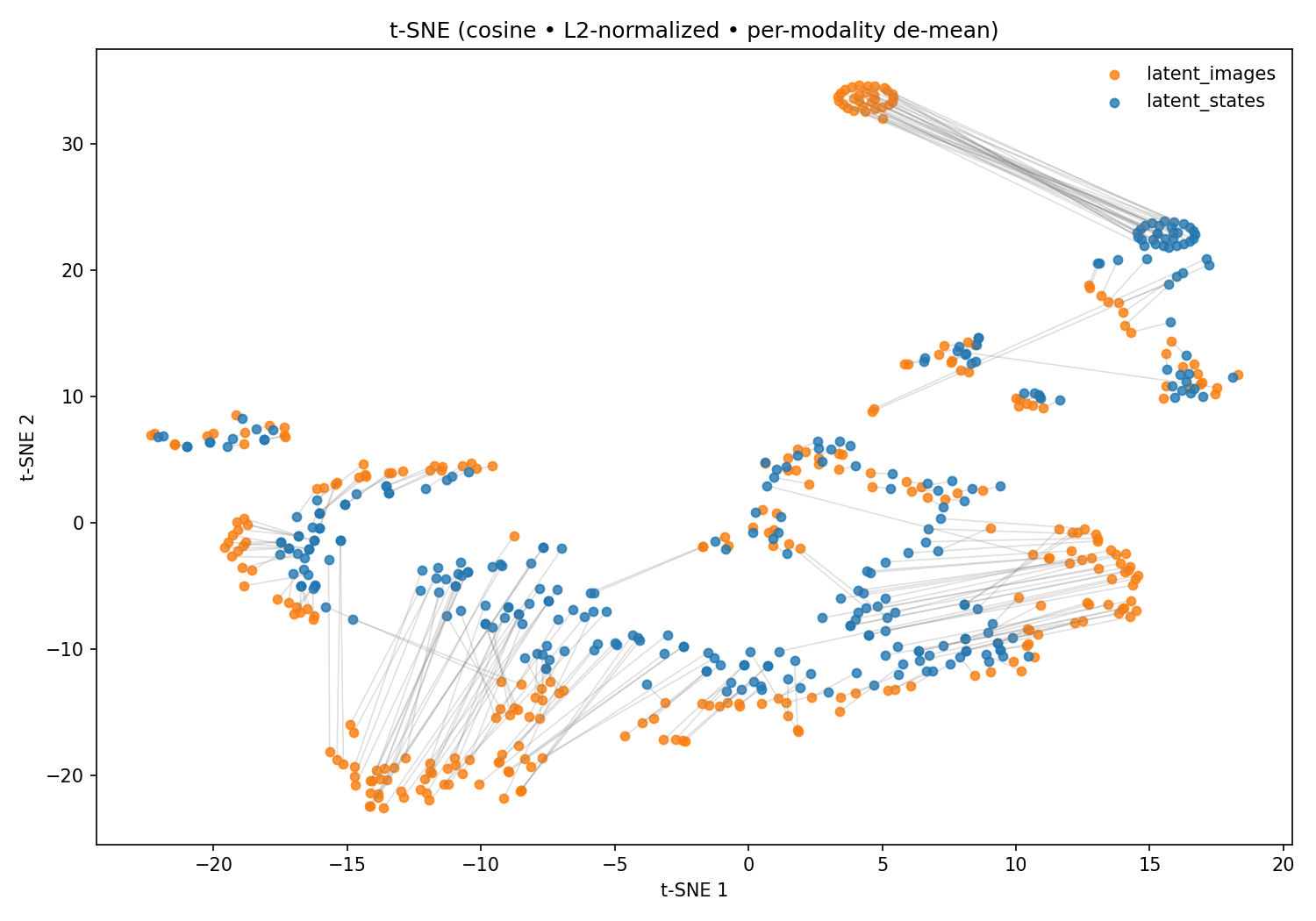} \\
\end{tabular}
\end{subfigure}\hfill
\begin{subfigure}[t]{0.6\columnwidth}
\centering
\setlength{\tabcolsep}{1pt}
\begin{tabular}{@{}c c c@{}}
 & \textbf{before} & \textbf{after} \\
\rotatebox{90}{\textbf{Push}} &
\includegraphics[width=0.6\linewidth]{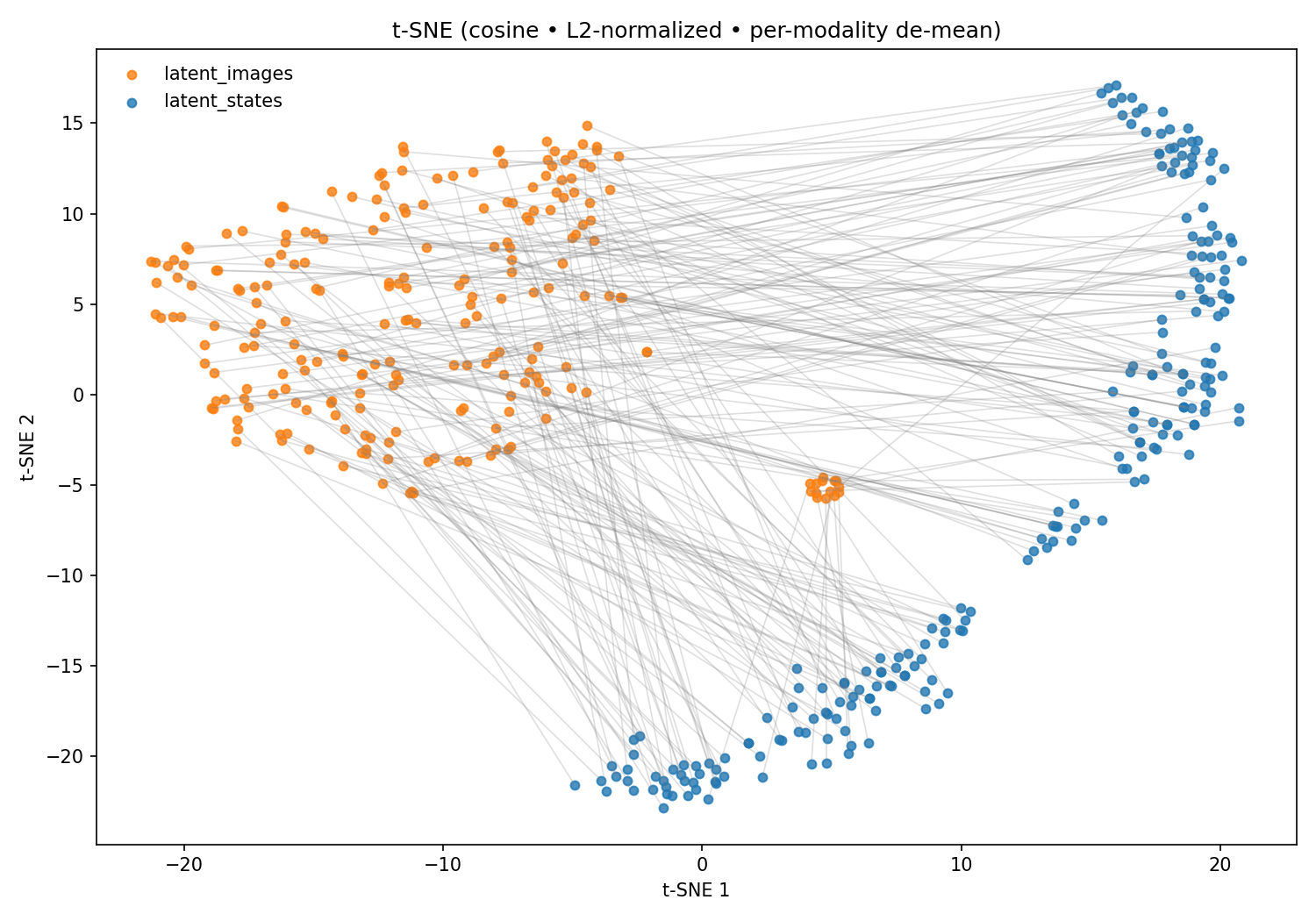} &
\includegraphics[width=0.6\linewidth]{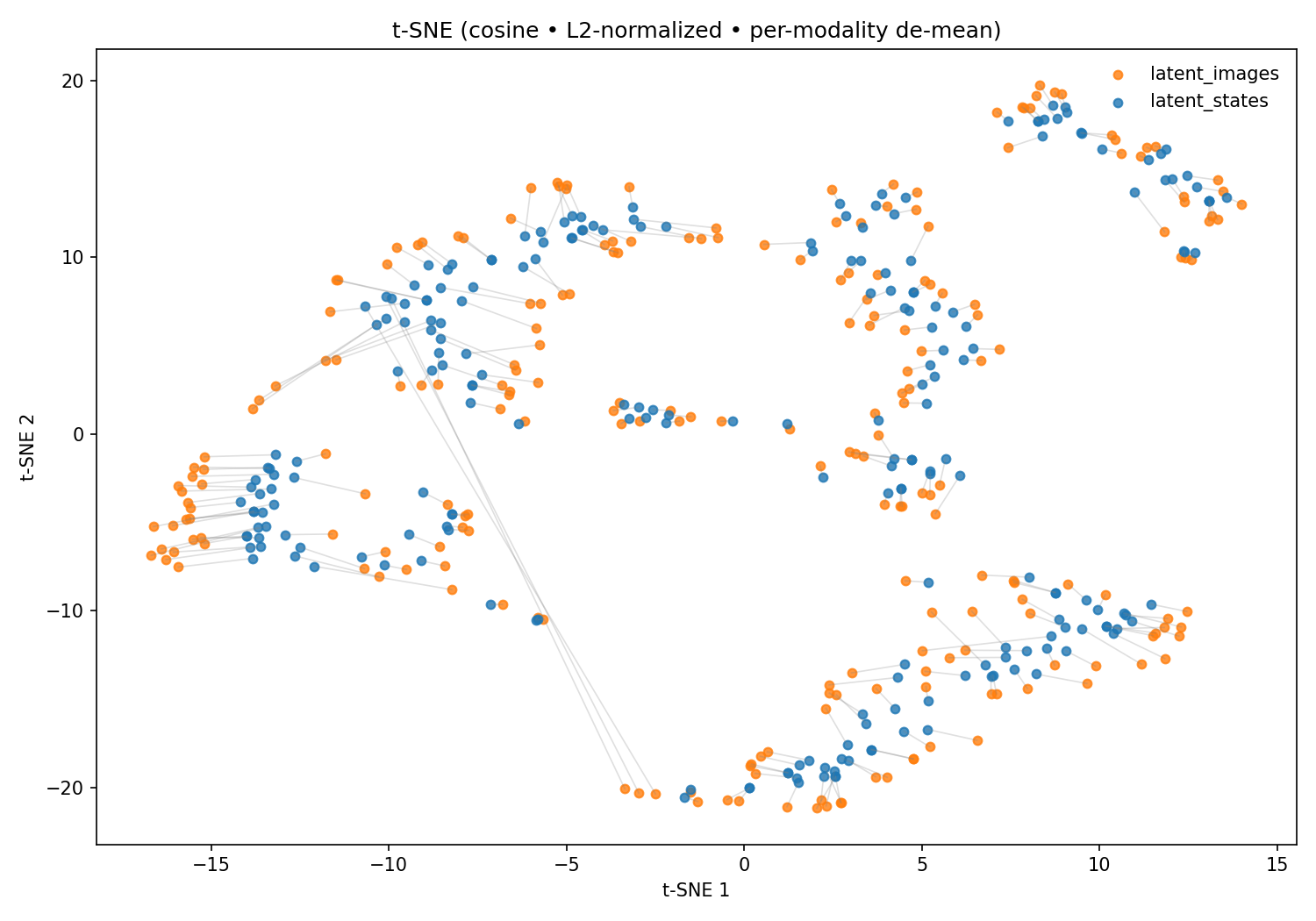} \\
\rotatebox{90}{\textbf{Open}} &
\includegraphics[width=0.6\linewidth]{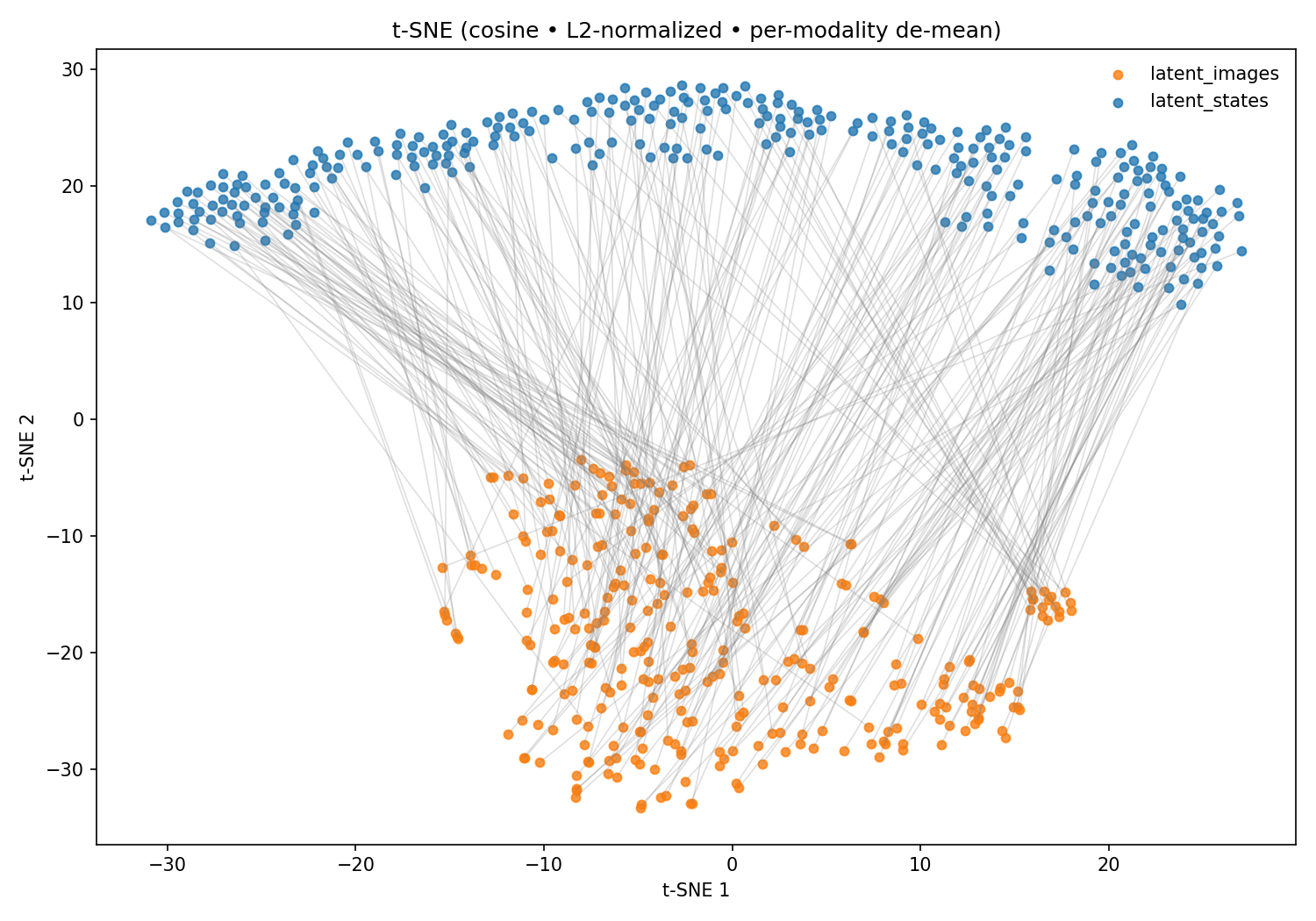} &
\includegraphics[width=0.6\linewidth]{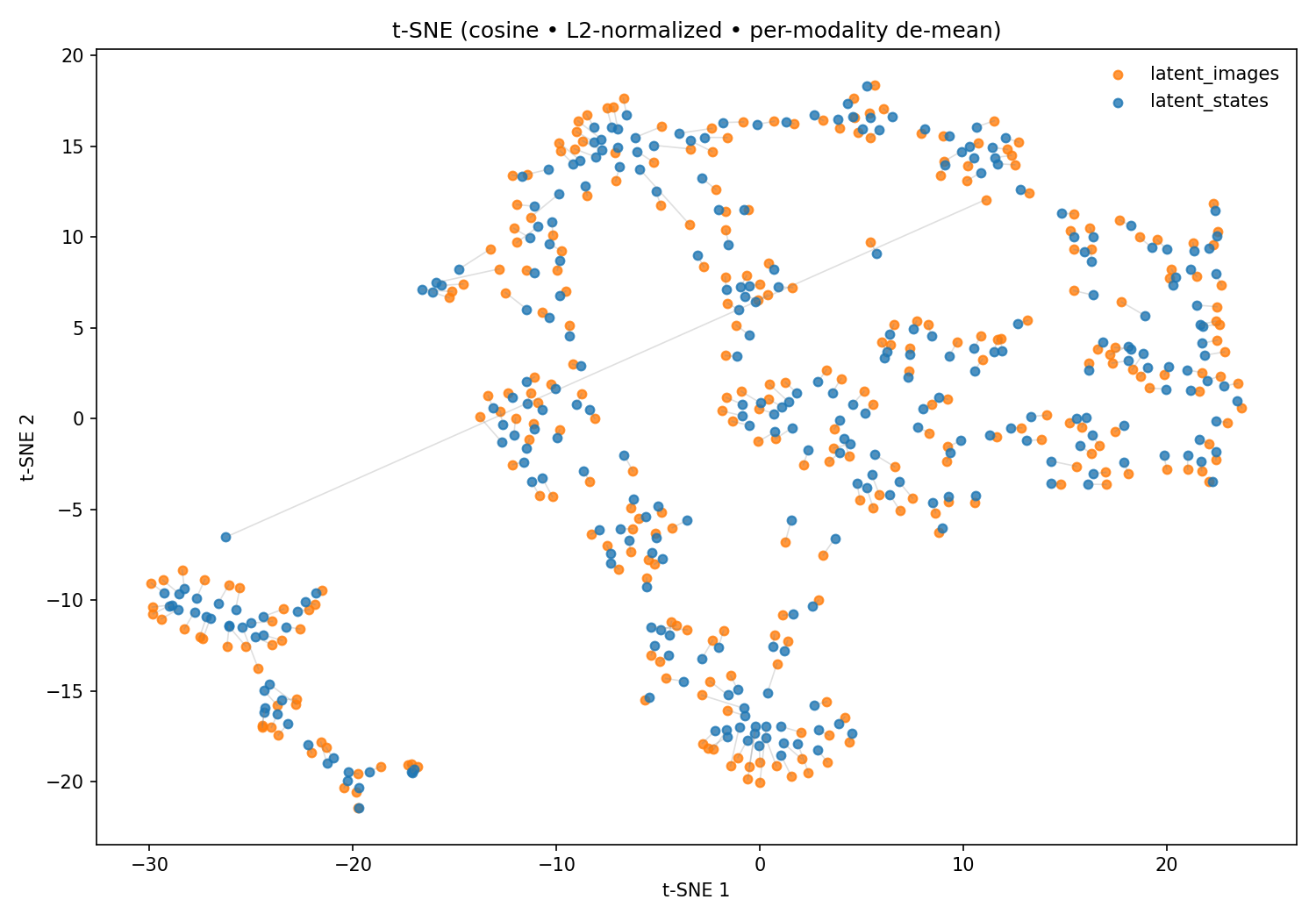} \\
\end{tabular}
\end{subfigure}

\caption{\textbf{t-SNE projections of state and image latent representations} before and after alignment for Lift, Push, Pick-and-Place, and Open-Drawer. These plots are qualitative and are included only as supporting diagnostics.}
\label{fig:tsne-tasks}
\end{figure}

\clearpage

\section{Protocol details}
\label{app:protocol_details}
We evaluate on five IsaacLab manipulation tasks: Reach-Cube, Lift-Cube, Push-Cube, Open-Drawer, and two-block Pick-and-Place. Simulation runs at 100Hz, control at 50Hz, and episodes last 5s unless stated otherwise. All methods use matched task definitions, observation preprocessing, action spaces, and evaluation protocols.

The tasks were selected to cover different levels of visual and control difficulty. Reach-Cube is a saturation and sanity-check task. Lift-Cube tests visually grounded grasping and lifting. Push-Cube tests contact-rich object-goal control. Open-Drawer evaluates articulated-object interaction under dense reward shaping. Pick-and-Place is the most demanding task, requiring grasping, transport, and placement.

\section{Implementation details}
\label{app:implementation_details}

RGB-D observations are rendered from tiled cameras, which are used in IsaacLab to accelerate parallel simulation, and are flattened before being passed to the policy implementation. Because tiled-camera rendering can introduce cross-environment visual artifacts, some task configurations apply a camera decontamination step before constructing the RGB-D observation. This step is used only to remove artifacts of vectorized rendering and is not a source of task information.

The deployable observation interface remains RGB-D, proprioception, and the previous action. The actor is not given object poses, privileged simulator state, contact state, semantic labels, segmentation IDs, or task variables at test time.

In the affected simulator configurations, the cleanup mask removes pixels that do not belong to the agent’s rendered environment before RGB-D construction. This is equivalent in purpose to standard deployable camera preprocessing such as invalid-depth removal, depth-range truncation, workspace cropping, or background clipping using camera calibration. The cleanup is applied identically to all vision-based methods on a given task, including RGB-D-only, RGB-D-asymmetric, teacher-student, and FOCUS.

When such simulator-specific cleanup is unavailable, the same observation contract can be implemented with depth truncation and workspace cropping alone. Thus, this preprocessing is a shared rendering-artifact correction rather than a method-specific component or a source of privileged task information.

\clearpage

\section{Additional implementation defaults}
\label{app:training_details}

\begin{table}[H]
\caption{\textbf{Main default hyperparameters} for \method{} unless otherwise noted.}
\label{tab:training-defaults-main}
\centering
\footnotesize
\setlength{\tabcolsep}{2pt}
\renewcommand{\arraystretch}{1.05}
\begin{tabularx}{\columnwidth}{@{}l X@{}}
\toprule
\multicolumn{2}{@{}l}{\textbf{Training and PPO defaults}} \\
\midrule
Latent dimension $d$                            & $256$ \\
PPO epochs                                      & $1$ \\
PPO mini-batches                                & $1$ \\
Clip parameter $\epsilon$                       & $0.2$ \\
Discount $\gamma$                               & $0.998$ \\
GAE parameter $\lambda$                         & $0.95$ \\
Max grad norm                                   & $1.0$ \\
\midrule
\multicolumn{2}{@{}l}{\textbf{Optimizer defaults}} \\
\midrule
RL optimizer                                    & Adam \\
Actor-critic learning rate                      & $1\times 10^{-3}$ \\
Fusion learning rate                            & $3\times 10^{-4}$ \\
State encoder learning rate                     & $3\times 10^{-4}$ \\
Representation optimizer                        & Adam \\
RGB-D encoder learning rate                     & $1\times 10^{-4}$ \\
Cross-attention aligner learning rate           & $2\times 10^{-4}$ \\
\midrule
\multicolumn{2}{@{}l}{\textbf{Modality switching defaults}} \\
\midrule
Max image exposure $\lambda_{\max}$             & $1.0$ \\
Switch warmup (start, duration)                 & $50$k, $30$k steps \\
Mean-KL budget $\varepsilon_{\text{mean}}$      & $0.03$ \\
Quantile-KL budget $\varepsilon_q$ ($q=0.9$)    & $0.05$ \\
Switch-controller EMA rate                      & $0.02$ \\
Policy-KL EMA rate                              & $0.10$ \\
Policy-KL tail EMA rate                         & $0.05$ \\
\midrule
\multicolumn{2}{@{}l}{\textbf{Main alignment / regularization weights}} \\
\midrule
RL policy-agreement $\lambda_{\pi,\mathrm{RL}}(t)$
                                                & cosine warmup to $0.02$ ($10$k / $50$k) \\
Value-agreement $\lambda_V(t)$                  & cosine warmup to $0.05$ ($5$k / $25$k) \\
Decoder reconstruction weight                   & $1.0$ \\
Latent conformity weight                        & $0.2$ \\
Distillation weight $\lambda_d$                 & $0.02$ \\
VICReg weight $\lambda_v$                       & $0.3$ \\
InfoNCE (raw) weight $\lambda_{\text{raw}}$     & $0.05$ \\
Image policy-agreement $\lambda_{\pi,\mathrm{img}}$
                                                & $0.01$ \\
\bottomrule
\end{tabularx}
\end{table}

\begin{table}[H]
\caption{\textbf{Additional implementation defaults} complementing Table~\ref{tab:training-defaults-main}.}
\label{tab:training-defaults-app}
\centering
\footnotesize
\setlength{\tabcolsep}{2pt}
\renewcommand{\arraystretch}{1.05}
\begin{tabularx}{\columnwidth}{@{}l X@{}}
\toprule
\multicolumn{2}{@{}l}{\textbf{Additional PPO / optimization details}} \\
\midrule
Value loss weight $\alpha_{\text{val}}$         & $1.0$ \\
Entropy weight $\beta_{\text{ent}}$             & $0.0$ \\
Clipped value loss                              & enabled \\
PPO KL target $\kappa_{\mathrm{PPO}}$           & $0.01$ when adaptive LR scheduling is enabled \\
RL optimizer details                            & Adam, betas $(0.9,0.999)$, weight decay $0$ \\
Representation optimizer details                & Adam, betas $(0.9,0.999)$, weight decay $10^{-5}$ \\
State decoder learning rate                     & $1\times 10^{-4}$ \\
\midrule
\multicolumn{2}{@{}l}{\textbf{Additional representation-loss details}} \\
\midrule
BYOL weight $\lambda_{\text{byol}}$             & $0.02$ \\
Raw-similarity weight $\lambda_{\text{raw-sim}}$& $0.01$ \\
BYOL EMA momentum $\tau(t)$                     & linear ramp $0.996 \rightarrow 0.999$ over $100$k steps \\
\bottomrule
\end{tabularx}
\end{table}

\begin{table}[H]
\caption{\textbf{Torch-level architecture} components used in the main experiments (latent dim $d{=}256$).}
\label{tab:arch-summary}
\centering
\captionsetup{skip=4pt}
\footnotesize
\setlength{\tabcolsep}{4pt}
\renewcommand{\arraystretch}{1.12}
\begin{tabularx}{\linewidth}{@{}p{0.22\linewidth} X@{}}
\toprule
\multicolumn{2}{@{}l}{\textbf{Architecture summary} (latent dim $d{=}256$)} \\
\midrule

\textbf{State encoder $E_s$} &
Linear$(\text{states}\!\rightarrow\!d)$
$\rightarrow$ Swish
$\rightarrow$ LayerNorm. \\

\addlinespace[0.6ex]
\textbf{RGB-D encoder $E_o$} &
\textbf{Shared-trunk ResNet-18 RGB-D encoder.}
Separate RGB/depth stems (RGB Conv1 from pretrained ResNet-18~\cite{he2016deep}; depth Conv1 initialized from averaged RGB weights, with optional validity-mask channel)
$\rightarrow$ summed stems
$\rightarrow$ shared ResNet-18 trunk (BN/ReLU/MaxPool, Layer1--Layer4)
$\rightarrow$ ECA channel attention~\cite{wang2020eca} (+ optional CBAM-style SpatialGate~\cite{woo2018cbam})
$\rightarrow$ AdaptiveAvgPool2d$(1,1)$
$\rightarrow$ Linear$(512\!\rightarrow\!d)$
$\rightarrow$ LayerNorm (projection: Identity). \\

\addlinespace[0.6ex]
\textbf{Fusion $g_\psi$} &
Residual fusion of a chosen latent $z\in\mathbb{R}^d$ with proprioception and previous action:
concat$([z;\mathbf{s}^{prop}_t;\mathbf{a}_{t-1}])$
$\rightarrow$ LayerNorm
$\rightarrow$ Linear$(m\!\rightarrow\!d)$ (zero-init)
$\rightarrow$ residual add $z{+}\Delta$.
(\emph{Same} fusion block used for actor and critic inputs; critic always uses $z_{t,s}$.) \\

\addlinespace[0.6ex]
\textbf{Aligner $f_\varphi$} &
Symmetric cross-attention aligner operating on raw latents $(\mathbf{z}_{t,s},\mathbf{z}_{t,o})$ (2 heads).
Per stream: CrossAttn$(q{=}z,\,kv{=}z_{\text{other}})$
$\rightarrow$ concat$([z;\mathrm{att}])$
$\rightarrow$ LayerNorm
$\rightarrow$ Linear$(2d\!\rightarrow\!d)$
$\rightarrow$ gated residual update. \\

\addlinespace[0.6ex]
\textbf{State decoder $D_s$} &
Residual MLP decoder from raw latent to privileged-state target:
$[d \!\rightarrow\! 512 \!\rightarrow\! 1024 \!\rightarrow\! 1024 \!\rightarrow\! 512 \!\rightarrow\! d_{\text{state}}]$
with Swish activations, Dropout$(0.1)$ after the first layer, final LayerNorm, and residual Linear$(d\!\rightarrow\!d_{\text{state}})$. \\

\addlinespace[0.6ex]
\textbf{Scratch CNN $E_o$ ablation} &
For the ``RGB-D encoder from scratch'' ablation, the ResNet-based $E_o$ is replaced by a lightweight two-branch RGB-D CNN.
The RGB and depth streams each use three stride-2 Conv--Swish blocks with channels $32,64,128$, followed by AdaptiveAvgPool2d$(4,4)$, flattening, Linear$(\cdot\!\rightarrow\!d)$, and LayerNorm.
The two $d$-dimensional branch features are concatenated and passed through a two-layer MLP to produce the final $d$-dimensional image latent.
The same RGB-D preprocessing contract is used as in the ResNet-based encoder. \\

\bottomrule
\end{tabularx}
\end{table}

\section{Logged metrics and diagnostics}
\label{app:diagnostics}
\begin{table}[H]
\caption{\textbf{Train/test metrics and additional diagnostics} logged by our protocol.}
\label{tab:metrics-diagnostics}
\centering
\tiny
\setlength{\tabcolsep}{2.5pt}
\renewcommand{\arraystretch}{1.08}
\begin{tabularx}{\linewidth}{@{}l c X@{}}
\toprule
\textbf{Signal (TB tag)} & \textbf{Phase} & \textbf{What it captures} \\
\midrule
\multicolumn{3}{@{}l}{\textit{Performance \& sample efficiency}}\\
\texttt{Train/success\_rate\_ema} & train & smoothed training success vs.\ env steps \\
\texttt{Train/steps\_to\_50pct} & train & Steps@50 (time-to-threshold) \\
\texttt{Test/Success\_rate} & test & evaluation success rate over episodes/seeds \\
\texttt{Test/rewards\_mean} & test & mean episode return (task-specific scale) \\
\midrule
\multicolumn{3}{@{}l}{\textit{PPO update stability}}\\
\texttt{Stability/ratio\_std\_mean} & train & mean ratio dispersion within PPO update \\
\texttt{Stability/ratio\_std\_p95} & train & upper-tail ratio volatility (95th percentile) \\
\texttt{Stability/ratio\_std\_time>0.3} & train & fraction of updates with ratio-std above threshold \\
\texttt{Stability/kl\_spike\_rate\_tail} & train & frequency of large KL spikes (tail events) \\
\texttt{Stability/neg\_jump\_rate\_tail} & train & tail rate of negative update “jumps” (debug) \\
\texttt{Stability/dReturn\_var\_tail} & train & tail variability proxy of returns/advantages (debug) \\
\texttt{StabilityLate/neg\_jump\_rate} & train & late-training negative jump rate (debug) \\
\texttt{StabilityLate/dReturn\_var} & train & late-training return-variance proxy (debug) \\
\midrule
\multicolumn{3}{@{}l}{\textit{Cross-modal policy mismatch}}\\
\texttt{Diag/policy\_KL\_image\_to\_state} & train & $\mathrm{KL}(\pi(\cdot\mid \tilde{\mathbf z}^{\,o})\,\|\,\pi(\cdot\mid \tilde{\mathbf z}^{\,s}))$ (behavioral gap) \\
\texttt{Diag/policy\_KL\_state\_to\_image} & train & $\mathrm{KL}(\pi(\cdot\mid \tilde{\mathbf z}^{\,s})\,\|\,\pi(\cdot\mid \tilde{\mathbf z}^{\,o}))$ \\
\texttt{Diag/KL\_i\_s\_p90} & train & upper-tail policy KL, 90th percentile (rare mismatch) \\
\texttt{Diag/KL\_i\_s\_p95} & train & upper-tail policy KL, 95th percentile (rarer mismatch) \\
\texttt{Diag/KL\_i\_s\_tail\_frac\_gt\_0.25} & train & tail exceedance rate: $\Pr[\mathrm{KL}_{o\to s} > 0.25]$ (example threshold) \\
\texttt{Diag/pi\_mean\_diff\_all} & train & policy mean disagreement summary across modalities (debug) \\
\texttt{Diag/mu\_diff\_norm} & train & $\|\Delta \mu\|$ between modality-conditioned Gaussians \\
\texttt{Diag/logstd\_diff\_norm} & train & $\|\Delta \log\sigma\|$ between modality-conditioned Gaussians \\
\texttt{Diag/std\_image\_mean} & train & mean action std under image-conditioned policy \\
\texttt{Diag/std\_state\_mean} & train & mean action std under state-conditioned policy \\
\midrule
\multicolumn{3}{@{}l}{\textit{Cross-modal value mismatch}}\\
\texttt{Diag/value\_gap\_mse\_half} & train & $\tfrac12\mathbb{E}\!\left[(V(\tilde{\mathbf z}^{\,s})-V(\tilde{\mathbf z}^{\,o}))^2\right]$ (critic disagreement) \\
\midrule
\multicolumn{3}{@{}l}{\textit{Latent / fusion mismatch (pre- and post-fusion)}}\\
\texttt{Diag/gap\_raw\_latents} & train & latent mismatch before alignment/fusion (raw encoder outputs) \\
\texttt{Diag/gap\_aligned\_latents} & train & latent mismatch after alignment module (if enabled) \\
\texttt{Diag/gap\_actor\_input} & train & mismatch at final actor input after fusion \\
\texttt{Diag/fusion\_gap\_raw} & train & raw fusion input gap (debug) \\
\texttt{Diag/fusion\_gap\_fused} & train & fused feature gap (debug) \\
\texttt{Diag/fusion\_gap\_fused\_over\_raw} & train & ratio of fused gap to raw gap (debug) \\
\texttt{Diag/fusion\_delta\_gap} & train & change in gap induced by fusion (debug) \\
\texttt{Diag/fusion\_delta\_norm\_i} & train & magnitude of fusion update on image pathway (debug) \\
\texttt{Diag/fusion\_delta\_norm\_s} & train & magnitude of fusion update on state pathway (debug) \\
\texttt{Diag/fusion\_fc\_weight\_norm} & train & fusion layer weight norm (debug / drift) \\
\texttt{Diag/zs\_norm} & train & $\|z_s\|$ (state latent norm tracker) \\
\texttt{Diag/zi\_norm} & train & $\|z_o\|$ (image latent norm tracker) \\
\midrule
\multicolumn{3}{@{}l}{\textit{Sensitivity diagnostics}}\\
\texttt{Diag/jacobian\_sensitivity} & train & local action sensitivity to latent mismatch \\
\texttt{Diag/actor\_L\_eff\_along\_dz} & train & effective gain $\|\Delta\mu\|/\|\Delta z\|$ along modality gap \\
\midrule
\multicolumn{3}{@{}l}{\textit{Representation losses (alignment branch)}}\\
\texttt{AlignLoss/alignment\_regularization} & train & total representation/alignment regularizer (summary) \\
\texttt{AlignLoss/distillation\_loss} & train & distillation/teacher matching term (if used) \\
\texttt{AlignLoss/feats\_mse\_loss} & train & feature-space MSE (debug / ablation) \\
\texttt{AlignLoss/similarity\_loss} & train & similarity-based alignment loss (debug) \\
\texttt{AlignLoss/sim\_score} & train & similarity score summary (higher is better) \\
\texttt{AlignLoss/raw\_latents\_similarity\_loss} & train & raw-latent similarity loss (pre-align) \\
\texttt{AlignLoss/raw\_latents\_sim\_score} & train & raw-latent similarity score (pre-align) \\
\texttt{AlignLoss/vicreg\_loss} & train & VICReg regularization term (if used) \\
\midrule
\multicolumn{3}{@{}l}{\textit{Alignment quality (retrieval / Procrustes summaries)}}\\
\texttt{PreAlign/procrustes\_rmse\_rl} & train & Procrustes RMSE between pre-align latents (RL-side) \\
\texttt{PreAlign/recall\_at\_1\_rl} & train & retrieval recall@1 for pre-align latents (debug) \\
\texttt{PreAlign/std\_image\_rl} & train & pre-align image latent std (debug) \\
\texttt{PreAlign/std\_state\_rl} & train & pre-align state latent std (debug) \\
\texttt{PostAlign/procrustes\_rmse} & train & Procrustes RMSE between post-align latents (debug) \\
\texttt{PostAlign/recall\_at\_1} & train & retrieval recall@1 for post-align latents (debug) \\
\midrule
\multicolumn{3}{@{}l}{\textit{RL-side agreement / weighting (if enabled)}}\\
\texttt{RL\_align/L\_align\_RL} & train & total RL-side alignment/consistency term (summary) \\
\texttt{RL\_align/policy\_consistency\_KL} & train & policy agreement regularizer (KL) \\
\texttt{RL\_align/value\_consistency\_MSE} & train & value agreement regularizer (MSE) \\
\texttt{RL\_align/lambda\_pi} & train & weight schedule/value for policy-consistency term \\
\texttt{RL\_align/lambda\_V} & train & weight schedule/value for value-consistency term \\
\bottomrule
\end{tabularx}
\end{table}

\section{Engineered rewards for each task}
\label{app:rewards}
\begin{lstlisting}[caption={Reward terms and corresponding implementations for the \textbf{Lift-Cube} task.}, label={lst:lift_cube_rewards}]
@configclass
class RewardsCfg:
    reaching_object = RewTerm(
        func=mdp.object_ee_distance_exp,
        params={"std": 0.1},
        weight=10.0,
    )

    grasp_reward = RewTerm(
        func=mdp.grasp_reward_from_pnp,
        params={
            "open_joint_pos": 0.04,
            "distance_threshold": 0.07,
            "near_close_reward": 1.0,
            "far_open_reward": 1.0,
        },
        weight=8.0,
    )

    lifting_object = RewTerm(
        func=mdp.object_is_lifted,
        params={"minimal_height": 0.04},
        weight=15.0,
    )

    dense_object_is_lifted = RewTerm(
        func=mdp.object_is_lifted_dense,
        params={
            "minimal_height": 0.02,
            "target_height": 0.10,
            "sigmoid_scale": 40.0,
            "object_cfg": SceneEntityCfg("object"),
        },
        weight=5.0,
    )

    object_goal_tracking = RewTerm(
        func=mdp.object_goal_distance_V2,
        params={
            "std": 0.3,
            "minimal_height": 0.04,
            "command_name": "object_pose",
        },
        weight=16.0,
    )

    object_goal_tracking_fine_grained = RewTerm(
        func=mdp.object_goal_distance_V2,
        params={
            "std": 0.05,
            "minimal_height": 0.04,
            "command_name": "object_pose",
        },
        weight=10.0,
    )

    success_rate = RewTerm(
        func=mdp.success_rate,
        params={
            "std": 0.1,
            "command_name": "object_pose",
        },
        weight=1e-16,
    )


def object_is_lifted(
    env: ManagerBasedRLEnv,
    minimal_height: float,
    object_cfg: SceneEntityCfg = SceneEntityCfg("object"),
) -> torch.Tensor:
    object: RigidObject = env.scene[object_cfg.name]
    return torch.where(object.data.root_link_pos_w[:, 2] > minimal_height, 1.0, 0.0)


def object_is_lifted_dense(
    env: ManagerBasedRLEnv,
    minimal_height: float,
    target_height: float = 0.2,
    object_cfg: SceneEntityCfg = SceneEntityCfg("object"),
    sigmoid_scale: float = 50.0,
    bonus_weight: float = 0.5,
) -> torch.Tensor:
    object = env.scene[object_cfg.name]
    z_pos = object.data.root_link_pos_w[:, 2]

    lift_progress = torch.sigmoid(sigmoid_scale * (z_pos - minimal_height))
    lifted_above_target = (z_pos > target_height).float()
    bonus = lifted_above_target * bonus_weight

    return lift_progress + bonus


def grasp_reward_from_pnp(
    env: ManagerBasedRLEnv,
    open_joint_pos: float = 0.04,
    distance_threshold: float = 0.02,
    near_close_reward: float = 1.0,
    far_open_reward: float = 1.0,
    asset_cfg: SceneEntityCfg = SceneEntityCfg("robot"),
    object_cfg: SceneEntityCfg = SceneEntityCfg("object"),
    ee_frame_cfg: SceneEntityCfg = SceneEntityCfg("ee_frame"),
) -> torch.Tensor:
    robot = env.scene[asset_cfg.name]
    object: RigidObject = env.scene[object_cfg.name]
    ee_frame: FrameTransformer = env.scene[ee_frame_cfg.name]

    object_pos = object.data.root_link_pos_w
    ee_tcp_pos = ee_frame.data.target_pos_w[..., 0, :]
    distance = torch.norm(ee_tcp_pos - object_pos, dim=-1)

    is_near = distance < distance_threshold
    is_far = ~is_near

    gripper_joint_pos = robot.data.joint_pos[:, asset_cfg.joint_ids]
    closure_amount = torch.sum(open_joint_pos - gripper_joint_pos, dim=-1) / open_joint_pos
    closure_amount = torch.clamp(closure_amount, 0.0, 1.0)

    reward = (
        is_near.float() * closure_amount * near_close_reward +
        is_far.float() * (1.0 - closure_amount) * 0.1 * far_open_reward
    )

    return reward


def object_ee_distance_exp(
    env: ManagerBasedRLEnv,
    std: float,
    object_cfg: SceneEntityCfg = SceneEntityCfg("object"),
    ee_frame_cfg: SceneEntityCfg = SceneEntityCfg("ee_frame"),
) -> torch.Tensor:
    object: RigidObject = env.scene[object_cfg.name]
    ee_frame: FrameTransformer = env.scene[ee_frame_cfg.name]

    cube_pos_w = object.data.root_link_pos_w
    ee_w = ee_frame.data.target_pos_w[..., 0, :]
    object_ee_distance = torch.norm(cube_pos_w - ee_w, dim=1)

    return torch.exp(-object_ee_distance / std)


def object_goal_distance_V2(
    env: ManagerBasedRLEnv,
    std: float,
    minimal_height: float,
    command_name: str,
    robot_cfg: SceneEntityCfg = SceneEntityCfg("robot"),
    object_cfg: SceneEntityCfg = SceneEntityCfg("object"),
) -> torch.Tensor:
    robot: RigidObject = env.scene[robot_cfg.name]
    object: RigidObject = env.scene[object_cfg.name]

    command = env.command_manager.get_command(command_name)

    des_pos_b = command[:, :3]
    des_pos_w, _ = combine_frame_transforms(
        robot.data.root_state_w[:, :3],
        robot.data.root_state_w[:, 3:7],
        des_pos_b,
    )

    object_pos = object.data.root_link_pos_w[:, :3]
    object_z = object_pos[:, 2]
    distance = torch.norm(des_pos_w - object_pos, dim=1)

    height_factor = torch.clamp((object_z - minimal_height) / 0.05, 0.0, 1.0)
    dist_reward = torch.exp(-distance / std)

    return height_factor * dist_reward
\end{lstlisting}

\begin{lstlisting}[caption={Reward configuration and reward functions for the \textbf{Pick-and-Place} task.}, label={lst:pick_place_rewards}]
@configclass
class RewardsCfg:
    pick_and_place_reward = RewTerm(
        func=mdp.rewards.pick_and_place_reward,
        params={
            "force_threshold": 0.3,
            "closure_thr": 0.5,
            "lift_height": 0.03,
            "place_horizontal_threshold": 0.07,
            "place_height_threshold": 0.06,
            "place_bonus": 3.0,
            "reach_scale": 5.0,
            "transport_scale": 5.0,
            "object_cfg": SceneEntityCfg("object"),
            "target_cfg": SceneEntityCfg("target_object"),
            "asset_cfg": SceneEntityCfg("robot", joint_names=["panda_finger_.*"]),
            "ee_frame_cfg": SceneEntityCfg("ee_frame"),
        },
        weight=10.0,
    )


def is_grasping(
    env,
    sensor_cfg: SceneEntityCfg = SceneEntityCfg("object_contact_sensor"),
    asset_cfg: SceneEntityCfg = SceneEntityCfg("robot"),
    open_joint_pos: float = 0.04,
    force_threshold: float = 2.0,
    closure_thr: float = 0.6,
) -> torch.Tensor:
    sensor: ContactSensor = env.scene.sensors[sensor_cfg.name]
    f = sensor.data.net_forces_w.norm(dim=-1)
    contact_detected = (f > force_threshold).any(dim=-1)

    robot = env.scene[asset_cfg.name]
    finger_q = robot.data.joint_pos[:, asset_cfg.joint_ids]
    open_amt = finger_q.mean(dim=-1).clamp(0.0, open_joint_pos) / open_joint_pos
    closed = open_amt < (1.0 - closure_thr)

    return contact_detected & closed


def pick_and_place_reward(
    env: ManagerBasedRLEnv,
    force_threshold: float = 0.1,
    place_horizontal_threshold: float = 0.07,
    place_height_threshold: float = 0.06,
    place_bonus: float = 3.0,
    reach_scale: float = 5.0,
    transport_scale: float = 5.0,
    lift_height: float = 0.03,
    lift_bonus: float = 2.0,
    closure_thr: float = 0.6,
    close_reward_scale: float = 2.0,
    close_distance_scale: float = 0.05,
    object_cfg: SceneEntityCfg = SceneEntityCfg("object"),
    target_cfg: SceneEntityCfg = SceneEntityCfg("target_object"),
    asset_cfg: SceneEntityCfg = SceneEntityCfg("robot"),
    sensor_cfg: SceneEntityCfg = SceneEntityCfg("object_contact_sensor"),
    ee_frame_cfg: SceneEntityCfg = SceneEntityCfg("ee_frame"),
) -> torch.Tensor:
    obj: RigidObject = env.scene[object_cfg.name]
    tgt: RigidObject = env.scene[target_cfg.name]
    ee_frame: FrameTransformer = env.scene[ee_frame_cfg.name]
    robot = env.scene[asset_cfg.name]

    obj_pos = obj.data.root_link_pos_w[:, :3]
    tgt_pos = tgt.data.root_link_pos_w[:, :3]
    ee_pos = ee_frame.data.target_pos_w[..., 0, :]

    reach_dist = torch.norm(ee_pos - obj_pos, dim=-1)
    reach_reward = 1.0 - torch.tanh(reach_scale * reach_dist)

    grasped = is_grasping(
        env,
        sensor_cfg=sensor_cfg,
        asset_cfg=asset_cfg,
        open_joint_pos=0.04,
        force_threshold=force_threshold,
        closure_thr=closure_thr,
    )

    finger_q = robot.data.joint_pos[:, asset_cfg.joint_ids]
    open_amt = finger_q.mean(dim=-1).clamp(0.0, 0.04) / 0.04

    near_object = torch.exp(-reach_dist / close_distance_scale)
    closing_reward = (~grasped).float() * near_object * (1.0 - open_amt) * close_reward_scale

    obj_rel_z = obj_pos[:, 2] - tgt_pos[:, 2]
    is_lifted = obj_rel_z > lift_height
    is_lifted_soft = torch.clamp(
        (obj_rel_z - (lift_height * 0.5)) / (lift_height * 0.5),
        0.0, 1.0,
    )

    lift_reward = grasped.float() * is_lifted_soft * lift_bonus

    carry_gate = torch.norm(ee_pos - obj_pos, dim=-1) < 0.08
    carrying = grasped & is_lifted & carry_gate

    transport_dist = torch.norm(obj_pos - tgt_pos, dim=-1)
    transport_reward = carrying.float() * (
        1.5 * (1.0 - torch.tanh(transport_scale * transport_dist))
        + 3.0 * torch.exp(-transport_dist / 0.05)
    )

    xy_dist = torch.norm(obj_pos[:, :2] - tgt_pos[:, :2], dim=-1)
    height_diff = (obj_pos[:, 2] - tgt_pos[:, 2]).abs()
    in_position = (xy_dist < place_horizontal_threshold) & (height_diff < place_height_threshold)

    hover_bonus = carrying.float() * in_position.float()
    release_approach_reward = carrying.float() * in_position.float() * open_amt * 4.0

    fully_open = open_amt > 0.40

    pick_latch = getattr(env.cfg, "grasp_hold_ctr", None)
    if pick_latch is None or pick_latch.shape != grasped.shape:
        pick_latch = torch.zeros_like(grasped, dtype=torch.bool, device=grasped.device)
    pick_latch = pick_latch | carrying
    setattr(env.cfg, "grasp_hold_ctr", pick_latch)

    release_reward = (pick_latch & in_position & fully_open).float() * place_bonus

    return (
        reach_reward
        + closing_reward
        + lift_reward
        + transport_reward
        + hover_bonus
        + release_approach_reward
        + release_reward
    )


def pick_and_place_reward_states(
    env: ManagerBasedRLEnv,
    force_threshold: float = 0.1,
    place_horizontal_threshold: float = 0.07,
    place_height_threshold: float = 0.06,
    place_bonus: float = 3.0,
    reach_scale: float = 5.0,
    transport_scale: float = 5.0,
    lift_height: float = 0.03,
    lift_bonus: float = 2.0,
    closure_thr: float = 0.6,
    object_cfg: SceneEntityCfg = SceneEntityCfg("object"),
    target_cfg: SceneEntityCfg = SceneEntityCfg("target_object"),
    asset_cfg: SceneEntityCfg = SceneEntityCfg("robot"),
    sensor_cfg: SceneEntityCfg = SceneEntityCfg("object_contact_sensor"),
    ee_frame_cfg: SceneEntityCfg = SceneEntityCfg("ee_frame"),
) -> torch.Tensor:
    obj: RigidObject = env.scene[object_cfg.name]
    tgt: RigidObject = env.scene[target_cfg.name]
    ee_frame: FrameTransformer = env.scene[ee_frame_cfg.name]
    robot = env.scene[asset_cfg.name]

    obj_pos = obj.data.root_link_pos_w[:, :3]
    tgt_pos = tgt.data.root_link_pos_w[:, :3]
    ee_pos = ee_frame.data.target_pos_w[..., 0, :]

    finger_q = robot.data.joint_pos[:, asset_cfg.joint_ids]
    open_amt = finger_q.mean(dim=-1).clamp(0.0, 0.04) / 0.04

    reach_dist = torch.norm(ee_pos - obj_pos, dim=-1)
    reach_reward = 1.0 - torch.tanh(reach_scale * reach_dist)

    near_object = torch.exp(-reach_dist / 0.05)
    closing_reward = near_object * (1.0 - open_amt) * 2.0

    grasped = is_grasping(
        env,
        sensor_cfg=sensor_cfg,
        asset_cfg=asset_cfg,
        open_joint_pos=0.04,
        force_threshold=force_threshold,
        closure_thr=closure_thr,
    )

    obj_rel_z = obj_pos[:, 2] - tgt_pos[:, 2]
    is_lifted = obj_rel_z > lift_height
    is_lifted_soft = torch.clamp(
        (obj_rel_z - (lift_height * 0.5)) / (lift_height * 0.5),
        0.0, 1.0,
    )

    lift_reward = grasped.float() * is_lifted_soft * lift_bonus

    transport_dist = torch.norm(obj_pos - tgt_pos, dim=-1)
    transport_reward = (grasped & is_lifted).float() * (
        (1.0 - torch.tanh(transport_scale * transport_dist))
        + 2.0 * torch.exp(-transport_dist / 0.05)
    )

    xy_dist = torch.norm(obj_pos[:, :2] - tgt_pos[:, :2], dim=-1)
    height_diff = (obj_pos[:, 2] - tgt_pos[:, 2]).abs()
    in_position = (xy_dist < place_horizontal_threshold) & (height_diff < place_height_threshold)

    hover_bonus = (grasped & is_lifted).float() * in_position.float()
    release_approach_reward = (grasped & is_lifted).float() * in_position.float() * open_amt * 2.0

    fully_open = open_amt > 0.70
    release_reward = (in_position & fully_open).float() * place_bonus

    return (
        reach_reward
        + closing_reward
        + lift_reward
        + transport_reward
        + hover_bonus
        + release_approach_reward
        + release_reward
    )
\end{lstlisting}

\begin{lstlisting}[caption={Reward terms and corresponding implementations for the \textbf{Push-Cube} task.}, label={lst:push_rewards}]
@configclass
class RewardsCfg:
    reaching_object = RewTerm(
        func=mdp.object_ee_distance,
        params={"std": 0.1},
        weight=20.0,
    )

    object_to_target_distance = RewTerm(
        func=mdp.distance_obj_to_target_reward,
        params={
            "sigma": 0.2,
            "object_cfg": SceneEntityCfg("object"),
            "target_object_cfg": SceneEntityCfg("target_object"),
        },
        weight=1.0,
    )

    gripper_fingers_reward = RewTerm(
        func=mdp.gripper_fingers_close_reward,
        params={
            "ee_frame_cfg": SceneEntityCfg("ee_frame"),
            "max_distance": 0.02,
        },
        weight=1.0,
    )

    ee_to_object_distance = RewTerm(
        func=mdp.ee_to_object_distance_penalty,
        params={"distance_threshold": 0.03},
        weight=-1e-3,
    )


def object_ee_distance(
    env: ManagerBasedRLEnv,
    std: float,
    object_cfg: SceneEntityCfg = SceneEntityCfg("object"),
    ee_frame_cfg: SceneEntityCfg = SceneEntityCfg("ee_frame"),
) -> torch.Tensor:
    object: RigidObject = env.scene[object_cfg.name]
    ee_frame: FrameTransformer = env.scene[ee_frame_cfg.name]

    cube_pos_w = object.data.root_pos_w
    ee_w = ee_frame.data.target_pos_w[..., 0, :]
    object_ee_distance = torch.norm(cube_pos_w - ee_w, dim=1)

    return 1 - torch.tanh(object_ee_distance / std)


def distance_obj_to_target_reward(
    env: ManagerBasedRLEnv,
    sigma: float = 0.2,
    object_cfg: SceneEntityCfg = SceneEntityCfg("object"),
    target_object_cfg: SceneEntityCfg = SceneEntityCfg("target_object"),
) -> torch.Tensor:
    object: RigidObject = env.scene[object_cfg.name]
    target_object: RigidObject = env.scene[target_object_cfg.name]

    object_pos = object.data.root_pos_w
    target_pos = target_object.data.root_pos_w

    distance = torch.norm(object_pos - target_pos, dim=1)
    shaped_distance = torch.tanh(torch.clamp(distance, min=0.02) / sigma)
    reward = 1.0 - shaped_distance

    return reward


def gripper_fingers_close_reward(
    env: ManagerBasedRLEnv,
    ee_frame_cfg: SceneEntityCfg = SceneEntityCfg("ee_frame"),
    max_distance: float = 0.02,
) -> torch.Tensor:
    ee_fingertips_w = env.scene[ee_frame_cfg.name].data.target_pos_w[..., 1:, :]
    lfinger_pos = ee_fingertips_w[..., 0, :]
    rfinger_pos = ee_fingertips_w[..., 1, :]

    finger_distance = torch.norm(lfinger_pos - rfinger_pos, dim=1)
    reward = torch.clamp(1.0 - finger_distance / max_distance, 0.0, 1.0)

    return reward


def ee_to_object_distance_penalty(
    env: ManagerBasedRLEnv,
    distance_threshold: float = 0.03,
    object_cfg: SceneEntityCfg = SceneEntityCfg("object"),
    ee_frame_cfg: SceneEntityCfg = SceneEntityCfg("ee_frame"),
    target_object_cfg: SceneEntityCfg = SceneEntityCfg("target_object"),
) -> torch.Tensor:
    object: RigidObject = env.scene[object_cfg.name]
    ee_frame: FrameTransformer = env.scene[ee_frame_cfg.name]
    target_object: RigidObject = env.scene[target_object_cfg.name]

    object_pos = object.data.root_pos_w
    ee_tcp_pos = ee_frame.data.target_pos_w[..., 0, :]
    target_object_pos = target_object.data.root_pos_w

    ee_distance = torch.norm(ee_tcp_pos - object_pos, dim=1)
    distance_to_target = torch.norm(target_object_pos[:, 0:3] - object_pos, dim=1)

    penalty = (ee_distance > distance_threshold) * 0.02 / (
        torch.clamp(distance_to_target, min=0.02)
    )

    return penalty
\end{lstlisting}

\begin{lstlisting}[caption={Reward configuration and reward functions for the \textbf{Reach-Cube} task.}, label={lst:reach_rewards}]
@configclass
class RewardsCfg:
    reaching_object = RewTerm(
        func=mdp.object_ee_distance,
        params={"std": 0.1},
        weight=9.0,
    )

    success_rate = RewTerm(
        func=mdp.object_ee_distance_success_rate,
        params={"std": 0.1},
        weight=1e-16,
    )


def object_ee_distance(
    env: ManagerBasedRLEnv,
    std: float,
    object_cfg: SceneEntityCfg = SceneEntityCfg("object"),
    ee_frame_cfg: SceneEntityCfg = SceneEntityCfg("ee_frame"),
) -> torch.Tensor:
    object: RigidObject = env.scene[object_cfg.name]
    ee_frame: FrameTransformer = env.scene[ee_frame_cfg.name]

    cube_pos_w = object.data.root_pos_w
    ee_w = ee_frame.data.target_pos_w[..., 0, :]
    object_ee_distance = torch.norm(cube_pos_w - ee_w, dim=1)

    return 1 - torch.tanh(object_ee_distance / std)


def object_ee_distance_success_rate(
    env: ManagerBasedRLEnv,
    std: float,
    object_cfg: SceneEntityCfg = SceneEntityCfg("object"),
    ee_frame_cfg: SceneEntityCfg = SceneEntityCfg("ee_frame"),
) -> torch.Tensor:
    object: RigidObject = env.scene[object_cfg.name]
    ee_frame: FrameTransformer = env.scene[ee_frame_cfg.name]

    cube_pos_w = object.data.root_pos_w
    ee_w = ee_frame.data.target_pos_w[..., 0, :]
    object_ee_distance = torch.norm(cube_pos_w - ee_w, dim=1)
    success_f = (object_ee_distance < 0.3).float()

    if env.cfg.success_rate is None:
        env.cfg.success_rate = success_f.clone()
    else:
        env.cfg.success_rate += success_f
    env.cfg.success_rate_count += 1

    env.cfg.last_step_success = success_f

    return torch.zeros_like(1 - torch.tanh(object_ee_distance / std), device=env.device)
\end{lstlisting}

\begin{lstlisting}[caption={Reward configuration and reward functions for the \textbf{Open-Drawer} task.}, label={lst:open_drawer_rewards}]
@configclass
class RewardsCfg:
    approach_ee_handle = RewTerm(
        func=mdp.approach_ee_handle,
        weight=2.0,
        params={"threshold": 0.2},
    )

    align_ee_handle = RewTerm(
        func=mdp.align_ee_handle,
        weight=0.5,
    )

    approach_gripper_handle = RewTerm(
        func=mdp.approach_gripper_handle,
        weight=5.0,
        params={"offset": MISSING},
    )

    align_grasp_around_handle = RewTerm(
        func=mdp.align_grasp_around_handle,
        weight=0.125,
    )

    grasp_handle = RewTerm(
        func=mdp.grasp_handle,
        weight=0.5,
        params={
            "threshold": 0.03,
            "open_joint_pos": MISSING,
            "asset_cfg": SceneEntityCfg("robot", joint_names=MISSING),
        },
    )

    open_drawer_bonus = RewTerm(
        func=mdp.open_drawer_bonus,
        weight=7.5,
        params={"asset_cfg": SceneEntityCfg("cabinet", joint_names=["drawer_top_joint"])},
    )

    multi_stage_open_drawer = RewTerm(
        func=mdp.multi_stage_open_drawer,
        weight=1.0,
        params={"asset_cfg": SceneEntityCfg("cabinet", joint_names=["drawer_top_joint"])},
    )

    success_rate = RewTerm(
        func=mdp.open_drawer_success_rate,
        weight=1e-16,
        params={
            "asset_cfg": SceneEntityCfg("cabinet", joint_names=["drawer_top_joint"]),
        },
    )


def approach_ee_handle(env: ManagerBasedRLEnv, threshold: float) -> torch.Tensor:
    ee_tcp_pos = env.scene["ee_frame"].data.target_pos_w[..., 0, :]
    handle_pos = env.scene["cabinet_frame"].data.target_pos_w[..., 0, :]

    distance = torch.norm(handle_pos - ee_tcp_pos, dim=-1, p=2)

    reward = 1.0 / (1.0 + distance**2)
    reward = torch.pow(reward, 2)
    return torch.where(distance <= threshold, 2 * reward, reward)


def align_ee_handle(env: ManagerBasedRLEnv) -> torch.Tensor:
    ee_tcp_quat = env.scene["ee_frame"].data.target_quat_w[..., 0, :]
    handle_quat = env.scene["cabinet_frame"].data.target_quat_w[..., 0, :]

    ee_tcp_rot_mat = matrix_from_quat(ee_tcp_quat)
    handle_mat = matrix_from_quat(handle_quat)

    handle_x, handle_y = handle_mat[..., 0], handle_mat[..., 1]
    ee_tcp_x, ee_tcp_z = ee_tcp_rot_mat[..., 0], ee_tcp_rot_mat[..., 2]

    align_z = torch.bmm(ee_tcp_z.unsqueeze(1), -handle_x.unsqueeze(-1)).squeeze(-1).squeeze(-1)
    align_x = torch.bmm(ee_tcp_x.unsqueeze(1), -handle_y.unsqueeze(-1)).squeeze(-1).squeeze(-1)
    return 0.5 * (torch.sign(align_z) * align_z**2 + torch.sign(align_x) * align_x**2)


def align_grasp_around_handle(env: ManagerBasedRLEnv) -> torch.Tensor:
    handle_pos = env.scene["cabinet_frame"].data.target_pos_w[..., 0, :]
    ee_fingertips_w = env.scene["ee_frame"].data.target_pos_w[..., 1:, :]
    lfinger_pos = ee_fingertips_w[..., 0, :]
    rfinger_pos = ee_fingertips_w[..., 1, :]

    is_graspable = (rfinger_pos[:, 2] < handle_pos[:, 2]) & (lfinger_pos[:, 2] > handle_pos[:, 2])
    return is_graspable


def approach_gripper_handle(env: ManagerBasedRLEnv, offset: float = 0.04) -> torch.Tensor:
    handle_pos = env.scene["cabinet_frame"].data.target_pos_w[..., 0, :]
    ee_fingertips_w = env.scene["ee_frame"].data.target_pos_w[..., 1:, :]
    lfinger_pos = ee_fingertips_w[..., 0, :]
    rfinger_pos = ee_fingertips_w[..., 1, :]

    lfinger_dist = torch.abs(lfinger_pos[:, 2] - handle_pos[:, 2])
    rfinger_dist = torch.abs(rfinger_pos[:, 2] - handle_pos[:, 2])

    is_graspable = (rfinger_pos[:, 2] < handle_pos[:, 2]) & (lfinger_pos[:, 2] > handle_pos[:, 2])

    return is_graspable * ((offset - lfinger_dist) + (offset - rfinger_dist))


def grasp_handle(
    env: ManagerBasedRLEnv,
    threshold: float,
    open_joint_pos: float,
    asset_cfg: SceneEntityCfg,
) -> torch.Tensor:
    ee_tcp_pos = env.scene["ee_frame"].data.target_pos_w[..., 0, :]
    handle_pos = env.scene["cabinet_frame"].data.target_pos_w[..., 0, :]
    gripper_joint_pos = env.scene[asset_cfg.name].data.joint_pos[:, asset_cfg.joint_ids]

    distance = torch.norm(handle_pos - ee_tcp_pos, dim=-1, p=2)
    is_close = distance <= threshold

    return is_close * torch.sum(open_joint_pos - gripper_joint_pos, dim=-1)


def open_drawer_bonus(env: ManagerBasedRLEnv, asset_cfg: SceneEntityCfg) -> torch.Tensor:
    drawer_pos = env.scene[asset_cfg.name].data.joint_pos[:, asset_cfg.joint_ids[0]]
    is_graspable = align_grasp_around_handle(env).float()

    return (is_graspable + 1.0) * drawer_pos


def multi_stage_open_drawer(env: ManagerBasedRLEnv, asset_cfg: SceneEntityCfg) -> torch.Tensor:
    drawer_pos = env.scene[asset_cfg.name].data.joint_pos[:, asset_cfg.joint_ids[0]]
    is_graspable = align_grasp_around_handle(env).float()

    open_easy = (drawer_pos > 0.01) * 0.5
    open_medium = (drawer_pos > 0.2) * is_graspable
    open_hard = (drawer_pos > 0.3) * is_graspable

    return open_easy + open_medium + open_hard


def open_drawer_success_rate(
    env: ManagerBasedRLEnv,
    asset_cfg: SceneEntityCfg,
    threshold: float = 0.01,
) -> torch.Tensor:
    drawer_pos = env.scene[asset_cfg.name].data.joint_pos[:, asset_cfg.joint_ids[0]]
    is_graspable = align_grasp_around_handle(env).float()

    success_f = (drawer_pos > threshold).float()

    if env.cfg.success_rate is None:
        env.cfg.success_rate = success_f.clone()
    else:
        env.cfg.success_rate += success_f
    env.cfg.success_rate_count += 1

    env.cfg.last_step_success = success_f

    return torch.zeros_like((is_graspable + 1.0) * drawer_pos, device=env.device)
\end{lstlisting}


\end{document}